\documentclass[nohyperref]{article}

\usepackage{scrlfile}
\PreventPackageFromLoading[]{algorithmic}

\usepackage{microtype}
\usepackage{graphicx}
\usepackage{subfigure}
\usepackage{booktabs} %

\usepackage{hyperref}

\PassOptionsToPackage{dvipsnames}{xcolor}

\usepackage[accepted]{icml2022}

\usepackage{amsmath}
\usepackage{amssymb}
\usepackage{mathtools}
\usepackage{amsthm}
\usepackage{xspace}
\usepackage{hyperref}
\usepackage[utf8]{inputenc}
\usepackage{pgfplots}
\usepackage{tikz}
\usepackage[frozencache,cachedir=.]{minted}
\usepackage{placeins}
\usepackage{natbib}
\usepackage{algorithm}
\usepackage[noend]{algpseudocode}

\algrenewcommand\algorithmicindent{0.6em}%
\usepackage{hyperref}
\usepackage{enumitem}
\usepackage{adjustbox}
\usepackage{multirow}

\DeclareUnicodeCharacter{2212}{−}
\usepgfplotslibrary{groupplots,dateplot}
\usetikzlibrary{patterns,shapes.arrows}
\pgfplotsset{compat=newest}
\newcommand{\bind}{%
  \mathrel{%
\begin{tikzpicture}[x=0.75pt,y=0.75pt,yscale=-1,xscale=1]

\draw   (0,5) .. controls (0,2.24) and (2.24,0) .. (5,0) .. controls (7.76,0) and (10,2.24) .. (10,5) .. controls (10,7.76) and (7.76,10) .. (5,10) .. controls (2.24,10) and (0,7.76) .. (0,5) -- cycle ;
\draw    (5,0) -- (1.57,8.57) ;
\draw    (8.47,1.43) -- (5,10) ;
\draw    (0.33,3.33) -- (9.67,3.33) ;
\draw    (0.33,6.67) -- (9.67,6.67) ;
\end{tikzpicture}
  }%
}

\usepackage[capitalize,noabbrev]{cleveref}

\theoremstyle{plain}
\newtheorem{theorem}{Theorem}[section]

\theoremstyle{definition}

\theoremstyle{remark}

\usepackage[textsize=tiny]{todonotes}

\icmltitlerunning{Connectionist Symbolic Pseudo Secrets}

\newcommand{\longName}{\textsc{Connectionist Symbolic Pseudo Secrets}\xspace}
\newcommand{\shortNameTxt}{\textsc{CSPS}\xspace}

\begin{document}
\twocolumn[
\icmltitle{Deploying Convolutional Networks on Untrusted Platforms\\Using 2D Holographic Reduced Representations}

\begin{icmlauthorlist}
\icmlauthor{Mohammad Mahmudul Alam}{umbc}
\icmlauthor{Edward Raff}{umbc,lps,bah}
\icmlauthor{Tim Oates}{umbc}
\icmlauthor{James Holt}{lps}
\end{icmlauthorlist}

\icmlaffiliation{umbc}{Department of Computer Science and Electrical Engineering, University of Maryland, Baltimore County, Baltimore, MD, USA}
\icmlaffiliation{lps}{Laboratory for Physical Sciences, College Park, MD, USA}
\icmlaffiliation{bah}{Booz Allen Hamilton, McLean, VA, USA}

\icmlcorrespondingauthor{Edward Raff}{Raff\_Edward@bah.com}
\icmlcorrespondingauthor{Tim Oates}{oates@cs.umbc.edu}

\icmlkeywords{Holographic Reduced Representations}

\vskip 0.3in
]

\printAffiliationsAndNotice{} %

\begin{abstract}
Due to the computational cost of running inference for a neural network, the need to deploy the inferential steps on a third party's compute environment or hardware is common. If the third party is not fully trusted, it is desirable to obfuscate the nature of the inputs and outputs, so that the third party can not easily determine what specific task is being performed. Provably secure protocols for leveraging an untrusted party exist but are too computational demanding to run in practice. We instead explore a different strategy of fast, heuristic security that we call \textit{Connectionist Symbolic Pseudo Secrets}. By leveraging Holographic Reduced Representations (HRR), we create a neural network with a pseudo-encryption style defense that empirically shows robustness to attack, even under threat models that unrealistically favor the adversary.
\end{abstract}

\section{Introduction}
As convolutional neural networks (CNN) have become more popular, so to have the concerns around their deployment. Many tricks like low-precision floats, pruning of weights, and classic software engineering and performance tuning have been employed to reduce these computation costs. Still, it is often necessary to deploy a model on third-party compute hardware or cloud environments for a variety of reasons (e.g., lower latency to customers, lack of computing resources, and elasticity of computing demand). In these situations, there are cases where the owner of the model does not fully trust the third party and desires to obfuscate information about the model running on this untrusted platform.
\par 
The current solutions to this situation naturally come from the encryption community, and the tools of Secure Multi-party Computation (SMC) \citep{Kerschbaum2006, Du:2001:SMC:508171.508174} and Homomorphic Encryption (HE)~\citep{pmlr-v48-gilad-bachrach16} provide methods for running programs on untrusted hardware that guarantee the privacy of the results. These are valuable tools, but computationally demanding and limiting. They often require restrictions on even basic CNN functionality like avoiding softmax activation and sigmoid/tanh non-linearity, limits on the size of the computation itself, and can dwarf the compute time saved by offloading to the third party. Especially when providers charge by compute-hours, this makes SMC and HE tools impractical when computing and latency constraints are a factor, or when neural networks are very large.
\par 
Current approaches to untrusted inference are all slower than simply running the computation locally, making them impractical. For this reason, our work scarifies provable security for empirical security, by developing an approach to insert ``secrets'' into a network's input that can be later extracted, yet obfuscate the input/output to the untrusted party in an encryption-like manner. We emphatically stress this is not strong encryption, but empirically we observe a realistic adversary's attacks are at random-guessing performance, and an unrealistically powerful adversary fairs little better. We term our approach \longName (\shortNameTxt)\footnote{ Our code can be found at \url{https://github.com/NeuromorphicComputationResearchProgram/Connectionist-Symbolic-Pseudo-Secrets}}, and compared to the fastest alternative \citep{244032}.  \shortNameTxt is $5000\times$ faster and transfers $18,000\times$ less data, making it practically deployable.
\par 

To summarize, we leverage
inspirations from encryption and neuro-symbolic methods
to symbolically represent a one-time pad strategy from the
encryption literature within a neural network. 
The rest of our paper is organized as follows. First, we will review work related to our own in \autoref{sec:related work}. Our approach uses a Vector Symbolic Architecture (VSA) known as the Holographic Reduced Representations (HRR) from more classical symbolic AI work that may not be familiar to all readers, so we will review them briefly in \autoref{sec:background}. This will allow us to discuss our method \shortNameTxt and how we develop a mechanism for inserting a secret ``one-time pad'' into a network input and extracting it from the output in \autoref{sec:method}. This produces a 5000$\times$ speedup and $18,000\times$ reduction in data transfer compared to the fastest alternatives, providing the first speedup for untrusted computation, as shown in  \autoref{sec:exp}. In addition, we show an overly powerful adversary is empirically only slightly better than random guessing, providing practical security for many applications, and extensive ablation studies over six alternative design choices that validate our approach. Finally, we conclude in \autoref{sec:conclusion} with a discussion of the limitations of our approach. Most notably that \textit{we are not implementing true strong encryption}, and so must temper expectations where privacy is a critical requirement.

\section{Related Work} \label{sec:related work}
The desire to hide the details of a program's inputs and outputs from third-party  performing the computation has been studied for many decades by the security and cryptography communities. These methods have been naturally adapted to deep learning tasks, providing provable privacy guarantees. Unfortunately, the high costs of these methods prevent them from being useful when there is any compute or runtime constraint, often requiring multiple order-of-magnitude slowdowns. We review the primary approaches.
\par 
The first approach that has been used is (Fully) Homomorphic Encryption (FHE), which allows recasting any program into a new version that takes encrypted inputs, and produces encrypted outputs, providing strong privacy. However, this conversion process can result in extreme computational cost, often requiring arbitrary-precision integer arithmetic\footnote{Also called ``bignum'' or ``big-integer''.}. To make FHE ``practical'', restrictions on the size, depth, and activation functions have been necessary to minimize these compute overheads. For example, \citep{pmlr-v48-gilad-bachrach16} required squared activations ($\sigma(x) = x^2$) to perform MNIST in an hour per datum. Current FHE methods, through a mix of network and FHE optimizations, can scale to CIFAR 10 \citep{pmlr-v97-brutzkus19a}, but result in networks slower and less accurate than our \shortNameTxt. We are not aware of any works that have scaled past CIFAR-10 for FHE-based inference \citep{Chou2018, QaisarAhmadAlBadawi2020, VanElsloo2019, Nandakumar2019, esperanca17a}.
\par 
The second broad class of approaches is protocol-based, requiring multiple rounds of communication where data is sent back-and-forth between the host that is requesting computation, and the third party server performing the bulk of computation. Methods like Secure Multi-party Computation (SMC) \citep{Kerschbaum2006,Du:2001:SMC:508171.508174} and other ``protocols'' are developed on top of ``Oblivious Transfer'' (OT), a primitive by which a \textit{sender} and \textit{receiver} exchange messages \citep{Rabin1981}.
Many OT protocols\footnote{We note that there are many different classes of protocols involved in these works, and our related work is  oversimplifying them to be \textit{just} ``OT'', but a full description of the different nuances would not aid the reader in understanding our approach, and all prior work share the same fundamental limitations.} have been customized for deep learning applications \citep{10.1145/3196494.3196522, 10.1145/3195970.3196023, Chandran2019, 10.5555/3361338.3361442, 10.1145/3133956.3134056, Mohassel2017}, but suffer similar limitations to FHE. They require minutes of computation per data point prediction, require multiple rounds of computation (a problem for deployment with any limited bandwidth or high latency network), and must send large ``messages'' on the order of hundreds of megabytes per prediction. We note that our approach requires only one round of communication, the messages are the same size as the original data points and can perform predictions in milliseconds.
\par 
Hybrid approaches combining OT and FHE have been developed \citep{217515,244032} and are faster than only OT or FHE, but they have not yet overcome the compute, multiple rounds of communication, and scaling limitations that prevent practical use.

The most similar approach to our own work is InstaHide \cite{pmlr-v119-huang20i}, which randomly combines training instances with a second population of images. These mixed images (and labels) are sent to a third party for \textit{training}, in an attempt to hide the true training task from the third party. \citet{Carlini2021a} showed how to break InstaHide and proved learning bounds indicating the impossibility of the approach. The key failure of InstaHide being a dual problem that: 1) the random additions are highly structured natural images, creating an attack avenue and 2) the goal is third party training, which requires InstaHide to provide the mixed image, leaving only 4 parameters a ``secret'' per image. Our focus on  HRRs allows unstructured secrets making attack harder, and the focus on inference of a trained model allows us to hide a large secret from the third party. We perform extensive customized attacks against \shortNameTxt to show we do not suffer the same failing, and provide learning bounds in the linear case that show we do not suffer the same conditions identified by \cite{Carlini2021a}.

We note that to the best of our knowledge, our work is the only approach seeking an approximate solution to the problem. This means our method should not be used when privacy is of extreme importance to be ``mission critical''. Still, we do obtain empirically good privacy in our results, and we show that our method is the only approach practically deployable when runtime or latency is a requirement. In particular, for all prior work cited \textit{the time it takes to run the FHE or OT protocols are orders of magnitude greater than the time to compute the result locally}. Our work is the first that we are aware of to present a method that enters the positive direction on the runtime trade-off.

\section{Technical Background} \label{sec:background}
Holographic Reduced Representations (HRR) is a method of representing compositional structure using circular convolution in distributed representations~\citep{b1}. Vectors can be composed together using circular convolution which is referred to as a \emph{binding} operation. Using the original notation for binding $\bind$ of Plate’s (1995) paper, the binding operation is expressed in eq.~\ref{binding}, where $\mathbf{x}_i, \mathbf{y}_i \in \mathbb{R}^d$ are arbitrary vector values, $\mathcal{F}(\cdot)$ and $\mathcal{F}^{-1}(\cdot)$ are the Fast Fourier Transform and its inverse, respectively.

\begin{equation}
\mathcal{B} = \mathbf{x}_i \bind \mathbf{y}_i =  \mathcal{F}^{-1}(\mathcal{F}(\mathbf{x}_i) \odot \mathcal{F}(\mathbf{y}_i))
\label{binding}
\end{equation}

$\mathcal{B} \in \mathbb{R}^d$ is the bound term comprised of $\mathbf{x}_i$ and $\mathbf{y}_i$. Two things makes HRR intriguing and valuable: the use of circular convolution, which is commutative, and its ability to retrieve bound components. The retrieval of bound components is referred to as \emph{unbinding}. A vector can be retrieved by defining an inverse function $\dagger: \mathbb{R}^d \to \mathbb{R}^d$ and identity function $\mathcal {F}(\mathbf{y}_{i}^{\dagger}) \cdot \mathcal{F}(\mathbf{y}_i) = \vec{1}$ which gives 
$\mathbf{y}^{\dagger}_i = \mathcal{F}^{-1} \left( \frac{1}{\mathcal{F}(\mathbf{y}_i)} \right)$.
Using the inverse of vector $\mathbf{y}_i$, another component of the bound term can be approximately retrieved by  $\mathbf{x}_{i} \approx \mathcal{B} \bind \mathbf{y}_{i}^{\dagger}$
\par 
These properties are interesting because they  hold in expectation even if $\mathcal{B}$ is defined with multiple terms, i.e., $\mathcal{B} = \sum_{i=1}^k \mathbf{x}_i\bind \mathbf{y}_i $, or when composed with hierarchical structure. This allows composing complex symbolic relationships by assigning meaning to arbitrary vectors,  staying in a fixed $d$-dimensional space. As the number of terms bound or added together increases, the noise of the reconstruction $\mathbf{x}'_i$ will also increase. To make these properties work we will use  initialization conditions proposed by \citep{hrrxml}, where $\mathbf{x}_i, \mathbf{y}_i \sim \pi(\mathcal{N}(0, 1/d))$, where $\pi(\cdot)$ is a projection onto the ball of complex unit magnitude $\pi(\mathbf{y}_i) = \mathcal{F}^{-1} \left(\: \frac{\mathcal{F}(\mathbf{y}_i)}{|\mathcal{F}(\mathbf{y}_i)|} \:\right)$, and $\mathcal{N}(\mu, \sigma^2)$ is the Normal distribution.

\begin{figure*}[!h]
\centerline{\includegraphics[scale=.855]{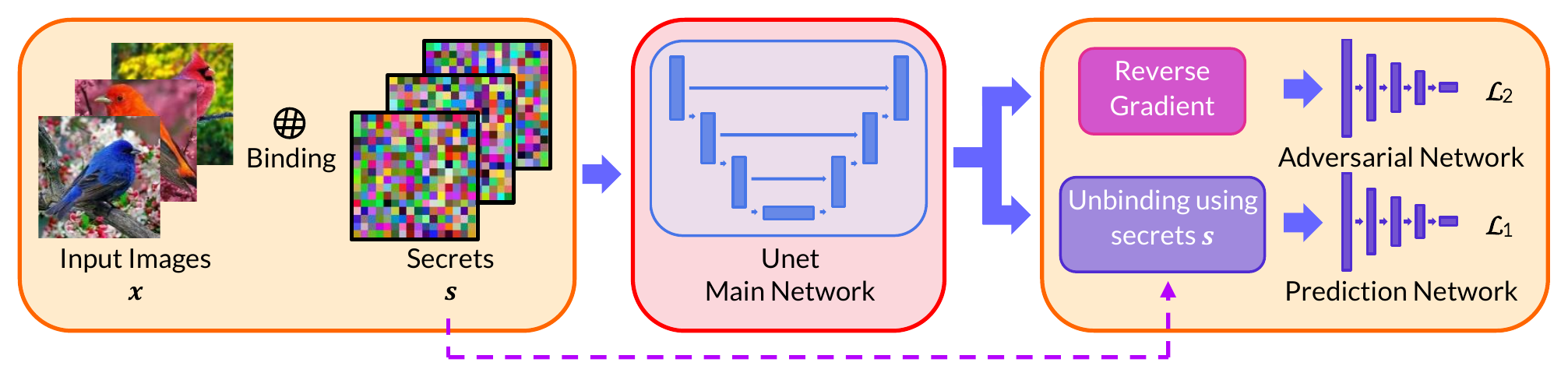}}
\caption{Block diagram of encryption process of the CNN using improved 2D HRR with three stages. Both of the orange regions are on the user-end. The secrets to unbind the images and outputs of the main network are only shared in these regions (dashed line). The red region indicates the untrusted third party who will run the main network after it has been trained.} 
\label{fig:train_approach}
\label{block}
\end{figure*}

\section{\longName} \label{sec:method}
Our approach to make \longName requires two steps. First, we introduce a simple modification of the HRR from 1-D to 2-D, exploiting a property of its construction so that we can embed the secret into the inputs in such a manner that they are likely to be preserved by the network. Then we design a training approach to use the symbolic behavior of HRR to bind the input, and then unbind the output, such that the majority of work can be done by a remote 3rd party.

\subsection{2D HRR As Pseudo One-Time Pad}
Our first insight comes from the fact that in \autoref{binding}, the result $\mathcal{B} = \boldsymbol{x} \bind \boldsymbol{s}$ is a simple linear operation that at infinite precision is invertible, giving $\boldsymbol{s} = \boldsymbol{x}^\dagger \mathcal{B}$. Thus if we have $\boldsymbol{x}$ represent the image (network input) we wish to obscure, and we have a random secret $\boldsymbol{s}$ to apply, then the resulting $\mathcal{B}$ object will appear random in nature. And for any bound output $\mathcal{B}$, there are infinite possible image/secret pairs that will produce the exact same output\footnote{The output is not uniformly random, a key difference from a true one-time pad.}. This allows for a ``one-time pad'' kind of approach to obscuring the input to the network from the untrusted party. If we can preserve the secret $\boldsymbol{s} \in \mathcal{B}$ as $\mathcal{B}$ is processed by a neural network, we can attempt to extract it using the unbinding operation at the end. Phrased mathematically, if $f(\cdot)$ is a normal CNN, we desire a secure function $\tilde{f}(\cdot)$ such that $\tilde{f}(\boldsymbol{x} \bind \boldsymbol{s}) \bind \boldsymbol{s}^\dagger \approx f(\boldsymbol{x})$, yet $\tilde{f}(\boldsymbol{x} \bind \boldsymbol{s})$ appears random. A critical part of this is to maintain the information within $\boldsymbol{s}$. 
\par
We can achieve this by recognizing that the HRR operation is equivalent to a 1D convolution over a sequence, but the 1D convolution is not an important property. By simply switching to 2D Fourier Transforms, we instead perform 2D convolutions to bind our secret, which aligns the resulting $\boldsymbol{x} \bind \boldsymbol{s}$ with the 2D CNN that will process it. By construction, this retains all the symbolic properties of HRR, as experimentally shown in \autoref{fig:hrr_2d}, but allows the inputs to behave in a manner consistent with a CNN.
\par 
This is critical as it means the binding operation $\bind$ is equivalent to another convolutional layer of the network, where the secret $\boldsymbol{s}$ is a user-chosen weight matrix rather than a learned one. This also means subsequent layers of convolution and pooling may learn to retain the structure of $\boldsymbol{s}$ to a sufficient degree that it can be extracted later, and effectively obfuscates the nature of the input as shown in \autoref{fig:hrr_2d_image}.

\begin{figure}[!h]
\input{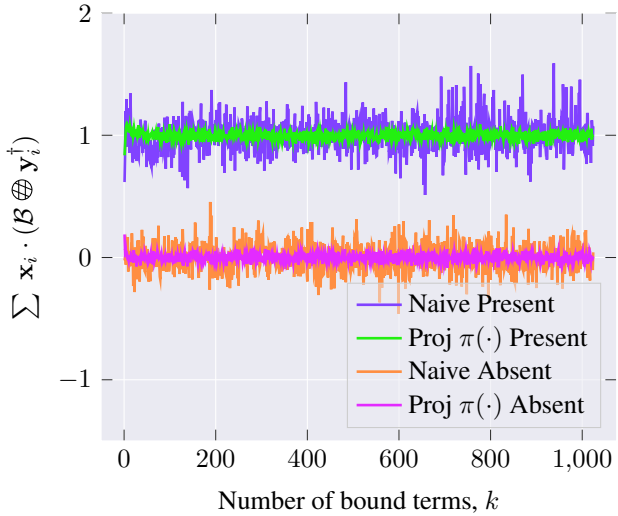}
\caption{Binding and unbinding terms using improved 2D HRR where for the terms that are present, $\mathbf{x}_i \bind \mathbf{y}_i \in \mathcal{B}$, the output along the y-axis is close to $1$ and for the absent terms, $\mathbf{x}_i \bind \mathbf{y}_i \notin \mathcal{B}$, output is close to $0$. Retrieval without using projection to the input is referred to as \textit{Naive present} and \textit{absent} shown in violet and orange color. Present and absent terms output with projection is shown in green and pink, respectively.}
\label{fig:hrr_2d}
\end{figure}

\begin{figure}[!h]
\centering 
\subfigure[Original Image]
{\includegraphics[scale=.22]{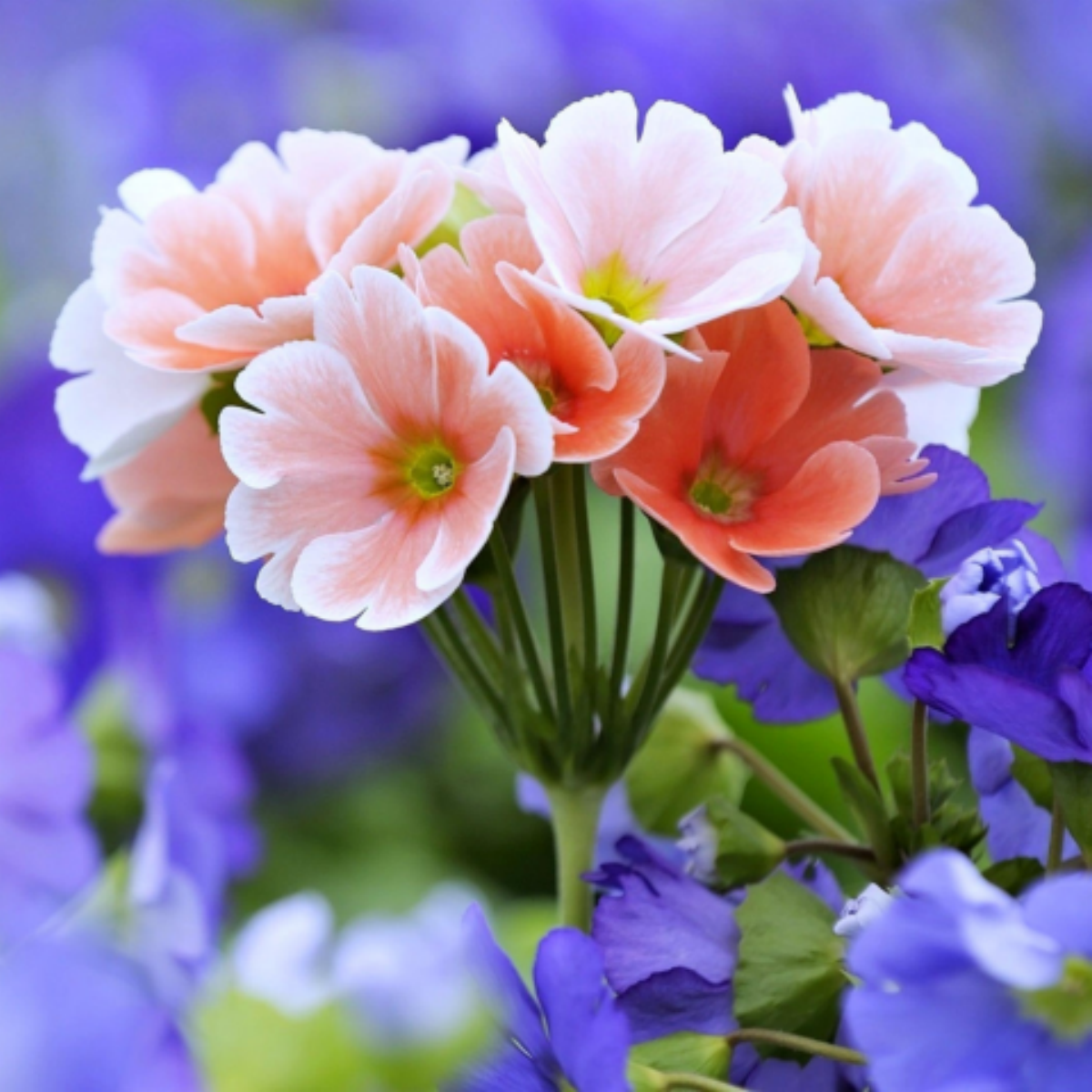}}
\subfigure[Bound Image]
{\includegraphics[scale=.22]{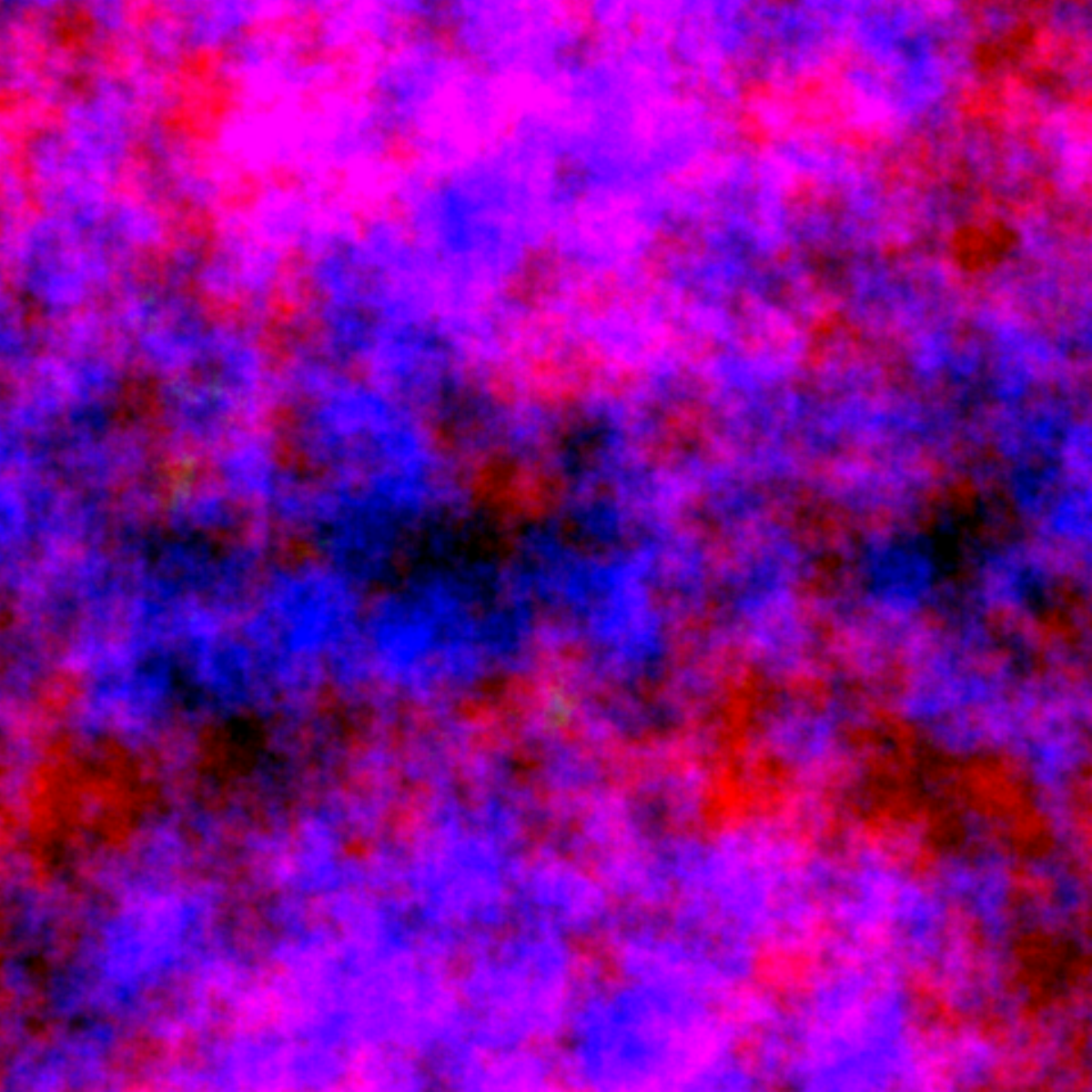}}
\subfigure[Retrieved Image]
{\includegraphics[scale=.22]{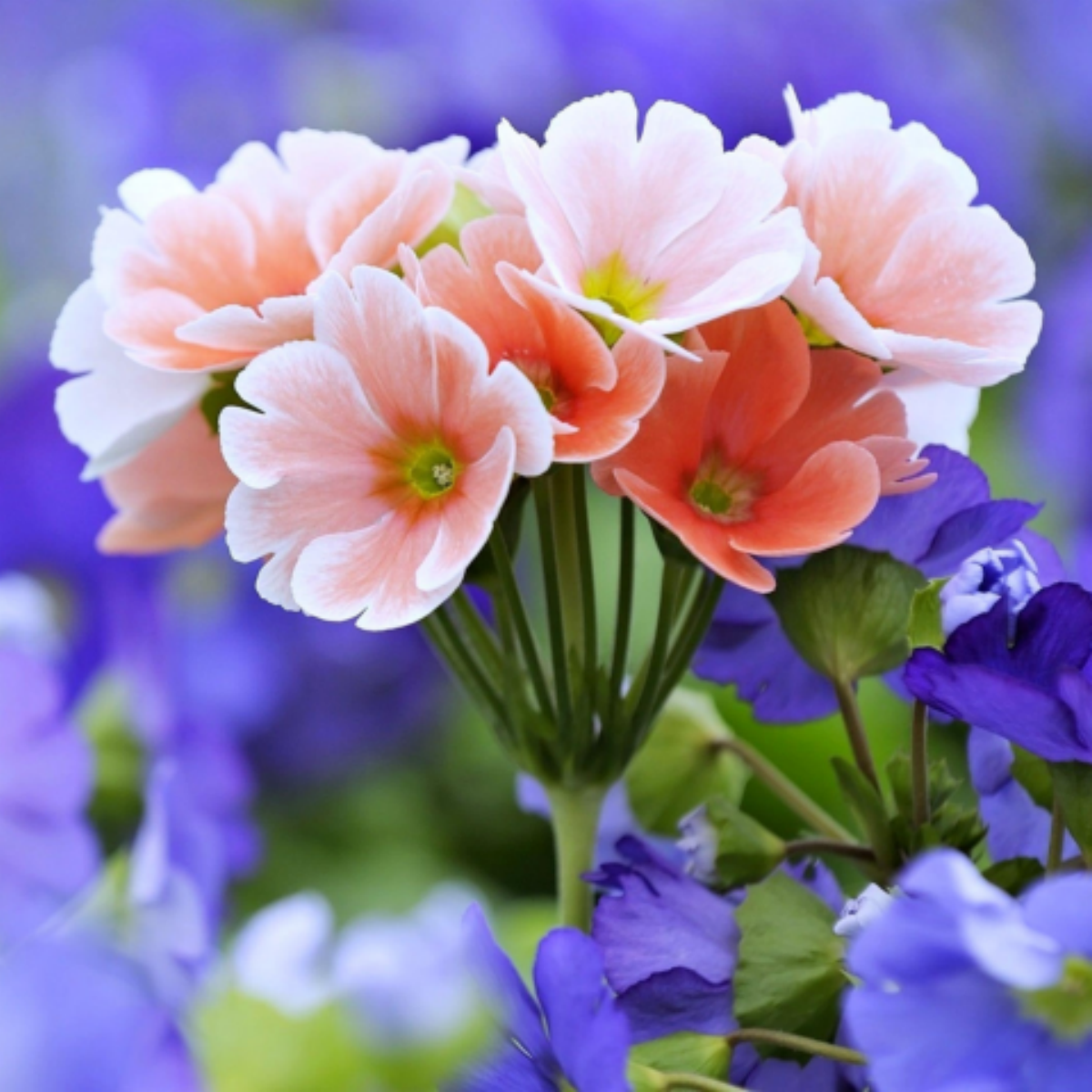}}
\caption{A sampled image $\boldsymbol{x}$ in (a) bound with a secret $\boldsymbol{s}$ in (b) using improved 2D HRR. The original image is retrieved using $\boldsymbol{s}^\dagger (\boldsymbol{x} \bind \boldsymbol{s}) \approx \boldsymbol{x}$ is shown in (c).}
\label{fig:hrr_2d_image}
\end{figure}

\subsection{Network Design}
We now specify our novel approach to a network architecture that leverages 2D HRR to hide the output, while offloading $\approx75\%$ of computation onto a remote third party. The proposed method has three networks, one larger backbone network $f_W(\cdot)$ that performs the ``work'' of feature extraction, and two identical smaller networks. As for the main network $f_W(\cdot)$, U-net CNN architecture~\citep{b2} has been employed, and at deployment would be run by the untrusted party. There are multiple benefits of using U-net architecture. It is a deep CNN with identical input and output shapes so that our secret vector $\boldsymbol{s}$ can be used on the input to and output to $f_W(\cdot)$. This requirement is because the secret $\boldsymbol{s}$ needs to be matched with the shape of the output of the network. The client computes $\boldsymbol{\hat{x}} = \boldsymbol{x} \bind \boldsymbol{s}$, sending $\boldsymbol{\hat{x}}$ to the third party while keeping the randomly chosen $\boldsymbol{s} \sim \pi\left(\mathcal{N}(0, 1/\sqrt{W\times H \times D})\right)$ a secret. The provider sends back the result $\boldsymbol{r} = f_W(\boldsymbol{\hat{x}})$.

Afterward, two identical classification networks are designed, as shown in the third stage of  Figure~\ref{block}. One is the \textbf{\mbox{Prediction~Network}} $f_P(\cdot)$ which is classifying after unbinding the main network outputs, giving the prediction $\hat{y}_P = f_P(\boldsymbol{r} \bind \boldsymbol{s}^\dagger)$. The other network is the  \textbf{\mbox{Adversarial~Network}}  $f_A(\cdot)$ which attempts to perform the prediction task without access to the secret $\boldsymbol{s}$, computing it's own prediction $\hat{y}_A = f_A(\boldsymbol{r})$. Both networks are used during training, but we reverse the gradient sign (i.e., multiply by $-1$) from $f_A(\cdot)$ back to the backbone network $f_W(\cdot)$.  \citet{Ganin:2016:DTN:2946645.2946704} introduced this approach as a form of domain adaption, but we instead use it to enforce that the secret $\boldsymbol{s}$ be necessary to extract meaning from the base network $f_W$. This gradient reversal will ensure that $f_A(\cdot)$ is attempting to minimize the same predictive loss, but the change in gradient sign means $f_W(\cdot)$ receives a learning signal sending it's optimization in the opposite direction - discouraging it from retaining any information that would be useful to $f_A(\cdot)$, but still receiving the correct gradient from $f_P(\cdot)$.

Combined, the binding $\boldsymbol{\hat{x}} = \boldsymbol{x} \bind \boldsymbol{s}$ ensures that the input to $f_W(\cdot)$ is random in appearance and is not discernible on its own, and the gradient reversal on $f_A(\cdot)$ ensures the output $\boldsymbol{r} = f_W(\boldsymbol{\hat{x}})$ is also not informative on its own. A new secret $\boldsymbol{s}$ is sampled for every prediction so that the third party cannot collect multiple samples to try and discover any ``single key''. This design heuristically provides the components of a secure protocol, but uses only standard operations built into all modern deep learning frameworks, giving it minimal overhead compared to prior approaches outlined in \autoref{sec:related work}. The overall training procedure is given in Algorithm~\ref{algo:train}.

\begin{algorithm}
    \caption{\longName Training using a dataset with images of size $W\times H \times D$ using a loss function $\ell(\cdot, \cdot)$.}
    \label{algo:train}
\begin{algorithmic}
\For{$\boldsymbol{x}_i, y_i \in $ dataset} \Comment{\small Training loop }
    \State $\boldsymbol{s}   \sim \pi\left(\mathcal{N}(0, 1/\sqrt{W\times H \times D})\right)$ \Comment{\small New secret} %
    \State $\boldsymbol{\hat{x}} \gets \boldsymbol{x}_i \bind \boldsymbol{s} $ \Comment{Obfuscated input}
    \State $\boldsymbol{r} \gets f_W(\boldsymbol{\hat{x}})$  \Comment{Run by 3rd party after training}
    \State $\hat{y}_P \gets f_P\left(\boldsymbol{r} \bind \boldsymbol{s}^\dagger\right)$ \Comment{Used locally after training}
    \State $\hat{y}_A \gets \Call{ReverseGrad}{f_A(\boldsymbol{r})}$ \Comment{Discarded after training}
    \State $\mathcal{L} \gets \ell(y, \hat{y}_P) + \ell(y, \hat{y}_A)$ \Comment{\small Incur training loss }
    \State Back-propagate on the loss $\mathcal{L}$
    \State Run optimization step
\EndFor
\end{algorithmic}
\end{algorithm}

The main network $f_W(\cdot)$ has four U-Net rounds in every experiment and doubles from 64 filters after each round, reversing for the decode. The $f_A(\cdot)$ and $f_P(\cdot)$ are always identical, with 3 rounds of aggressive convolution followed by pooling to minimize compute costs and shrink the representation, followed by two fully connected hidden layers. Mini-ImageNet receives a fourth round of pooling due to its larger resolution. All network details and code can be found in \autoref{sec:train_details}. We further perform extensive ablation studies in \autoref{sec:ablation} looking at different binding operations (HRR without projection, 1D HRR, 1D HRR with Hilbert Curves, and the  vector-derived transformation binding (VTB)) and network designs (Residual style) that show our design of U-Net with 2D HRRs is critical to obtaining high predictive accuracy. 

While our results are heuristic for the deep neural networks, we provide theoretical evidence for our approach by analyzing the linear case. Given an adversary who has all $n$ bounded inputs $\hat{\mathbf{x}}_i$ with the true labels $y_i$, the problem is likely not linearly learnable due to an $\mathcal{O}(n)$ Rademacher complexity, as we show in \autoref{thm:linear_csps}. 

\begin{theorem} \label{thm:linear_csps}
Learning $\boldsymbol{w}^\top \boldsymbol{x}_i \bind \boldsymbol{s}_i$ without the secrets $\boldsymbol{s}_i$ has a non-trivial Rademacher complexity of $\mathcal{O}(n)$, 
\end{theorem}
\begin{proof}
The \shortNameTxt Rademacher model gives $
\frac{1}{n} \underset{\boldsymbol{\sigma}}{\mathbb{E}}\left[\sup _{\mathbf{w} \in \mathbb{R}^d | \|\mathbf{w}\|_2 \leq 1} \sum_{i=1}^{n} \sigma_{i} \mathbf{w}^\top \boldsymbol{x}_{i}\right] 
$. The binding operation with a vector $\boldsymbol{s}_i$ is equivalent to the matrix-vector product with a corresponding circulant matrix $S^C_i$. Each $\mathbf{w}^\top S^C_i$ can be written as an independent random rotation leading to $n$ independent $\hat{\mathbf{w}}_i$  terms, allowing the supremum to move into the summation to give $\frac{1}{n} \underset{\boldsymbol{\sigma}}{\mathbb{E}}\left[ \sum_{i=1}^n \sup _{\tilde{\mathbf{w}}_i \in \mathbb{R}^d | \|\mathbf{w}_i\|_2 \leq 1} \sigma_i \tilde{\mathbf{w}}_i^\top \boldsymbol{x}_i \right]$. Applying the result for $n$ independent linear models \cite{shalev-shwartz_ben-david_2021} trained on one point gives the final complexity $\sum_{i=1}^n \|\boldsymbol{x}_i\|_2 \leq n \max_{i}\left\|\boldsymbol{x}_{i}\right\|_{2}  $.
\end{proof}

\section{Experiments \& Results} \label{sec:exp} \label{sec:results}

We do not argue that \shortNameTxt is any true form of encryption, only that it is empirically effective at hiding the nature of inputs and outputs sent to an untrusted party. We will demonstrate this through a series of experiments to show that: 1) Compared to a network with the same design but without the HRR binding/unbinding of the secret $\boldsymbol{s}$, that our approach has some loss of accuracy but is more accurate than prior approaches. 2) The loss of accuracy can be largely mitigated by averaging the results of $\leq 10$ queries. 3) Our approach is robust to adversaries using unsupervised learning that try to infer class information. 4) Our approach is still robust to unrealistically strong adversaries that know the training data and classes, obtaining $1.5-4.7\times$ random guessing accuracy. 5) Our approach is up to 290$\times$ faster than existing provable methods. We also perform an extensive ablation study of alternative designs that show our approach performs considerably better than alternatives. Before our results, we briefly review the datasets and training details.
\par 
As the proposed method is doing image classification, various well-known image classification datasets are used for the experiments. These datasets are diverse in shape, color, channels, contents, and the number of classes. In total $5$ image classification datasets are utilized, namely, MNIST, SVHN, CIFAR-10, CIFAR-100, and Mini-ImageNet. Dataset details, along with training time, and data augmentation can be found in \autoref{sec:data_train_details}.

\subsection{Accuracy Results} 
Our results focus on Top-1 classification accuracy for all datasets, and Top-5 accuracy for datasets with 100 classes. We start by demonstrating the accuracy of our approach in \autoref{tab:accuracy}, where ``Base'' indicates a network with the same total architecture (including U-Net backbone), but without any of the binding/unbinding of secrets or gradient reversal of the adversarial network. This shows that 1) our method can scale to Mini-ImageNet, a result not previously possible \citep{244032}, and 2) that there is some cost that our secret binding incurs on the accuracy of the result.

\begin{table}[!t]
\centering
\caption{Accuracies of the Base model, and model with secret binding and unbinding using improved 2D HRR.}
\label{tab:accuracy}
\renewcommand{\arraystretch}{1.2}
\adjustbox{max width=\columnwidth}{%
\begin{tabular}{@{}clcc@{}}
\toprule
Dataset & Model & Top-1 & Top-5 \\ \midrule
\multirow{2}{*}{\shortstack{MNIST\\$28\times28$}} 
            &  Base     &  98.80  &  --          \\
            &  \shortNameTxt  &  98.51  &  --       \\ 
\multirow{2}{*}{\shortstack{SVHN\\$32\times32$}}
            &  Base     &  93.76  &  --          \\
            &  \shortNameTxt  &  88.44  &  --       \\ 
\multirow{2}{*}{\shortstack{CIFAR-10\\$32\times32$}} 
            &  Base     &  83.57  &  --          \\
            &  \shortNameTxt  &  78.21  &  --       \\ 
\multirow{2}{*}{\shortstack{CIFAR-100\\$32\times32$}} 
            &  Base     &  62.59  &  86.99       \\
            &  \shortNameTxt  &  48.84  &  75.82    \\ 
\multirow{2}{*}{\shortstack{Mini-ImageNet\\$84\times84$}} 
            &  Base     &  55.73  &  80.55       \\
            &  \shortNameTxt  &  40.99  &  66.99    \\ \bottomrule
\end{tabular}
}
\end{table}

The results in \autoref{tab:accuracy} are for a single attempt at the prediction process, and noise is introduced by the randomly selected secret $\boldsymbol{s}$. We can average out this noise by sending $k$ inputs $\boldsymbol{x} \bind \boldsymbol{s}_1, \boldsymbol{x} \bind \boldsymbol{s}_2, \ldots, \boldsymbol{x} \bind \boldsymbol{s}_k$ to be classified, and averaging the resulting $k$ predictions. This provides the result given in \autoref{fig:avg_prediction}, showing that $k \leq 10$ is sufficient to almost completely eliminate the accuracy drop. Additional discussion of this result is in \autoref{sec:avg_pred_extra}. We note the lower accuracy numbers compared to more modern networks on these problems comes from the difficulty of learning with random $\boldsymbol{s}$ vectors bound to the input. For example, our CIFAR-10 training accuracy is 86.17\%, which is close to the test accuracy. For this reason overfitting does not appear to be a culprit in the results. The difficulty of the trained network to work with random $\boldsymbol{s}$ vectors also provides intuition as to its success as a defense: the attacker with less access should have more difficulty handling the HRR vectors, and thus inhibits their success.

\begin{figure}[!htbp]
\centerline{\includegraphics[keepaspectratio, width=0.95\columnwidth]{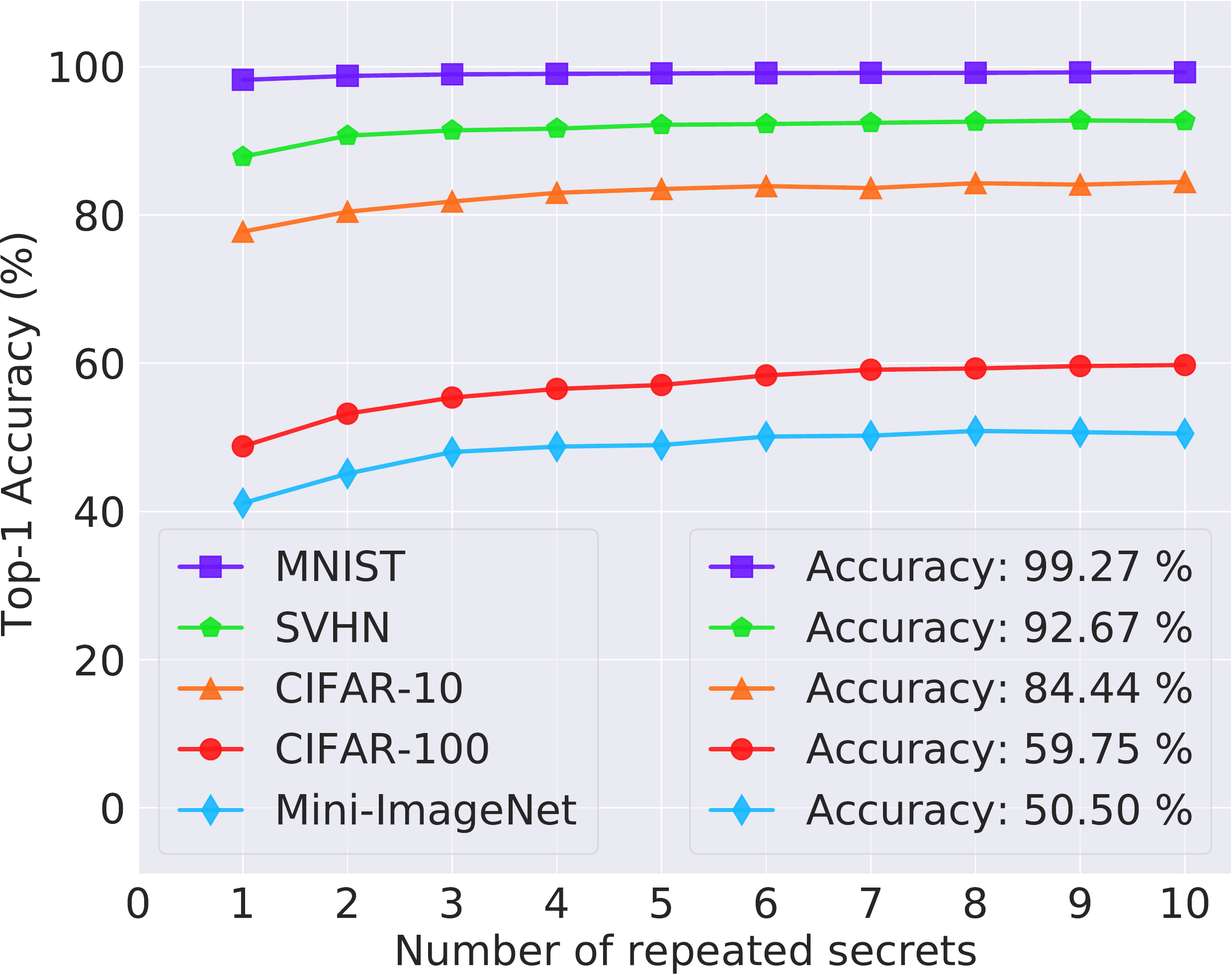}} 
\caption{Accuracy (y-axis) of \shortNameTxt after averaging $k$ predictions (x-axis), which almost fully restores the accuracy lost due to the secret binding/unbinding.}
\label{fig:avg_prediction}
\end{figure}

\subsection{Run-time Results}
To show the speed advantages of \shortNameTxt, we perform a comparison with HE that is \textit{unrealistically favorable to HE}. HE has significant design constraints for a neural network, and thus we could not replicate our ``Base'' architecture with HE libraries like ~\citep{b4}. Instead, we compare against a HE network with a single convolutional layer, followed by aggressive pooling and a fully connected layer, making the model extremely small, unable to learn with any predictive accuracy, and unusable with performance that is close to random guessing (except on MNIST). The  results are in \autoref{tab:timing}, where Mini-ImageNet failed to make a single prediction in under 24 hours. These settings are overly idealized for HE, and is still $290\times$ slower than our \shortNameTxt.

\begin{table}[!t]
\centering
\caption{Time to perform prediction on each dataset in its entirety, where the Homomorphic Encryption alternative is a single CNN layer and unrealistically small to minimize its runtime at the cost of all predictive accuracy.}
\label{tab:timing}
\renewcommand{\arraystretch}{1.2}
\adjustbox{max width=\columnwidth}{%
\begin{tabular}{@{}ccc@{}}
\toprule
Dataset       & Our \shortNameTxt & HE Est.                \\ \midrule
MNIST         & $4.56$ Seconds   & $2$ Hours $46$ Minutes  \\
SVHN          & $12.44$ Seconds  & $55$ Hours $32$ Minutes \\
CIFAR-10      & $7.58$ Seconds   & $21$ Hours $20$ Minutes \\
CIFAR-100     & $9.07$ Seconds   & $43$ Hours $53$ Minutes \\
Mini-ImageNet & $28.37$ Seconds  & Timeout                 \\ \bottomrule
\end{tabular}
}
\end{table}

\begin{table}[!t]
\centering
\caption{Amount of computation performed by the local user and the remote third party}
\label{tab:computation}
\renewcommand{\arraystretch}{1.2}
\adjustbox{max width=\columnwidth}{%
\begin{tabular}{@{}lcc@{}}
\toprule
Dataset       & Remote \% & Local \% \\ \midrule
MNIST         & 74.24    & 25.76   \\
SVHN          & 65.06    & 34.94   \\
CIFAR-10      & 66.08    & 33.92   \\
CIFAR-100     & 66.78    & 33.22   \\
Mini-ImageNet & 74.42    & 25.58   \\ \bottomrule
\end{tabular}
}
\end{table}

Because \shortNameTxt uses standard deep learning code, there is no extraneous compute overhead for arbitrary precision math or multiple rounds of network communication. Thus we can look at the amount of compute saved by offloading to a remote party in \autoref{tab:computation}. We see that at least $65\%$ of compute can be offloaded, netting a $2.9-3.5\times$ reduction in cost. This cost savings can be important for low-power, battery, or  compute constrained devices. The fastest prior work by \cite{244032} requires 60 MB of extra communication \textit{per prediction} on CIFAR-10 (that is $18,000\times$ larger than the image being worked on, the entire corpus is only 200 MB), and is reported to be $5019\times$ slower than our approach. To the best of our knowledge, \shortNameTxt is thus the only method that can realize a real-world resource reduction.

\subsection{Realistic Adversary}
Following \citet{Biggio2014} we specify a realistic adversary that seeks to infer the nature of our model's outputs and class distribution. Because the output shape $\boldsymbol{r} \in \mathbb{R}^{W\times H\times D}$ has no relationship with the number of classes, they must attempt to use some unsupervised clustering to identify patterns within the predictions. We have applied several diverse clustering algorithms such as Kmeans~\citep{kmeans,Raff2021}, Spectral~\citep{spectral}, Gaussian Mixture Model (GMM), Birch~\citep{birch}, and HDBSCAN~\citep{hdbscan} cluster to the outputs $\boldsymbol{r} = f_W(\cdot)$ that the adversary has access to. If they are able to perform clustering with greater than random chance, they may be able to extract information about how our model works. We pessimistically assume the adversary knows the exact number of clusters $k$ that they should be looking for. Thus we use the Adjusted Rand-Index (ARI) to score how well the clusters perform with respect to the true class labels \citep{Vinh:2010:ITM:1756006.1953024}. A near 0\% ARI indicates there is no information to be extracted
, and thus our clustering has performed well. 
We also consider the case where the adversary clusters on the bounded inputs they receive, $\hat{x}$, for completeness. We note this is a poor attack avenue in realistic settings because multiple classes may exist per given set of input images, and that information is only leaked by the network $f_W(\cdot)$ and not the inputs.

The results are  in \autoref{tab:clustering_label_all}, where the adversary performs best on MNIST with $\leq 1.5\%$ ARI scores. On all other datasets, the score is $\leq 0.2\%$, indicating there is almost no label information to be extracted from the clusters using existing methods. This shows our approach is effective in hiding the nature of the output and predictions from the untrusted party. Note that on HDBSCAN zero scores are obtained because of degenerate results. HDBSCAN has the concept of ``outliers'' that do not belong to any cluster and assigns almost the entire test set to the ``outlier'' class, resulting in worst-case scores. 
We note $\hat{x}$ is larger than $\boldsymbol{r}$, resulting in Spectral clustering timing out after several days of running. Overall the results clearly demonstrate that minimal amount of label information is leaked by the model.

\begin{table}[!h]
\caption{Clustering results of the adversary attempting to discover class information directly from the main network output $\boldsymbol{r}$ (top) and the bound image inputs $\boldsymbol{\hat{x}}$ (bottom). 
All numbers are percentages, and the Adjusted Rand-Index is $\leq 1.5\%$ for all cases. Since ARI accounts for random-chance clustering, the unrealistic adversary is not able to meaningful discern class information.
}
\label{tab:clustering_label_all}
\adjustbox{max width=\columnwidth}{%
\begin{tabular}{@{}lrrrrr@{}}
\toprule
        & \multicolumn{1}{c}{MNIST} & \multicolumn{1}{c}{SVHN} & \multicolumn{1}{c}{\begin{tabular}[c]{@{}c@{}}CIFAR\\ 10\end{tabular}} & \multicolumn{1}{c}{\begin{tabular}[c]{@{}c@{}}CIFAR\\ 100\end{tabular}} & \multicolumn{1}{c}{\begin{tabular}[c]{@{}c@{}}Mini\\ ImgNet\end{tabular}} \\ \midrule
K-Means  & 1.28                      & 0.06                     & 0.21                         & 0.03                          & 0.08                              \\
Spectral & 0.01                      & 0.01                     & 0.00                         & 0.00                          & 0.02                              \\
GMM      & 1.28                      & 0.06                     & 0.17                         & 0.04                          & 0.09                              \\
Birch    & 1.51                      & 0.03                     & 0.13                         & 0.05                          & 0.07                              \\
HDBSCAN  & 0.00                      & 0.00                     & 0.00                         & 0.00                          & 0.00                              \\ 
\cmidrule(l){2-6}
K-Means & -0.02                     & -0.01                    & 0.18                                                                   & 0.54                                                                    & 0.42                                                                      \\
GMM     & 0.01                      & 0.00                     & 0.09                                                                   & 0.61                                                                    & 0.44                                                                      \\
Birch   & 0.20                      & 0.00                     & 0.14                                                                   & 0.45                                                                    & 0.35                                                                      \\
HDBSCAN & 0.00                      & -0.24                    & 1.23                                                                   & 0.01                                                                    & 0.02                                                                      \\ \bottomrule
\end{tabular}
}
\end{table}

\begin{figure}[!t]
\centering 
\subfigure[True Class Labels]
{\includegraphics[keepaspectratio, width=0.45\columnwidth]{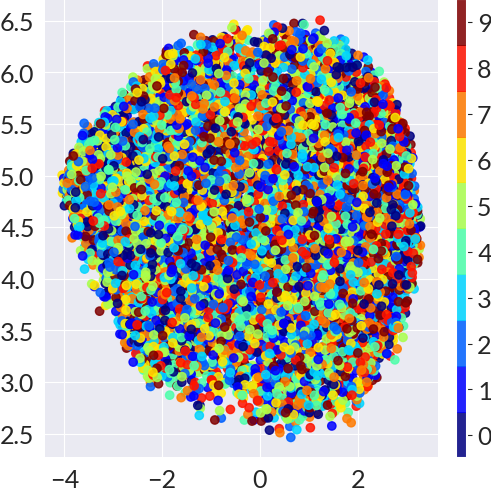}}
\hspace{2mm}
\subfigure[Birch Cluster Labels]
{\includegraphics[keepaspectratio, width=0.45\columnwidth]{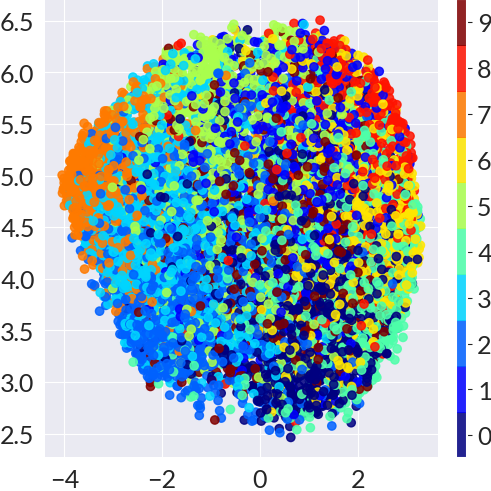}}
\caption{UMAP embeddings of output $r$ of $f_W(\cdot)$ (left) and the ``best'' clustering (right). Without the secret $\boldsymbol{s}$ the class labels appear random (left), and clustering detects spurious density patterns with no correlation to the true labels.}
\label{clustering}
\end{figure}

\autoref{clustering} shows visually how \shortNameTxt is able to achieve this result. The result vectors $\boldsymbol{r}$ that the untrusted party has access to are plotted using UMAP~\citep{mcinnes2018umap}. Because $f_A(\cdot)$ is trained with gradient reversal, $f_W(\cdot)$ has learned a representation $\boldsymbol{r}$ that requires the secret $\boldsymbol{s}$ to extract the meaning from its representation. Thus the result is an embedding space that points from the same class are randomly dispersed and intermixed. When it is clustered the clusters do not correspond to the true class distributions, as shown on the right. Additional visualizations showing how the cluster labels do not correlate with class labels in \autoref{sec:cluster_visual_results}.

\subsection{Overly Strong Adversary}
An unrealistically powerful adversary would have access to the entire training set, the class labels, know the procedure of binding/unbinding secrets, and be able to train their own model that predicts the class label from the intermediate result $\boldsymbol{r}$ (i.e., knows everything but the secrets $\mathbf{s}_i$). This is in fact what the adversarial network $f_A(\cdot)$ performs, and represents the worst possible case scenario for the adversary's strength: that they can train their own model to ignore the secret $\boldsymbol{s}$ and extract the true labels. We thus perform this test by training a new model on the ground-truth pairs between $\boldsymbol{r}_i$ and the class label $\boldsymbol{y}_i$, with the results shown in \autoref{tab:adversarial}. Note an adversary of this strength already knows what the predictive task and type of data is, which is what \shortNameTxt is designed to hide. Success of \shortNameTxt at this strength shows efficaciousness from task level to individual level protection beyond our intended goal, but also provides even greater evidence of task level protection.

\begin{table}[!h]
\centering
\caption{Accuracies of the \textbf{Adversarial~Network} (lower is better) where the secret to unbind the output of the U-Net is unknown.}
\label{tab:adversarial}
\renewcommand{\arraystretch}{1.2}
\adjustbox{max width=\columnwidth}{%
\begin{tabular}{@{}ccc@{}}
\toprule
Dataset & Top-1 Accuracy (\%) & Top-5 Accuracy (\%) \\ \midrule
MNIST         & 19.72 & --    \\ 
SVHN          & 21.13 & --    \\ 
CIFAR-10      & 12.91 & --    \\ 
CIFAR-100     &  2.66 & 10.33 \\ 
Mini-ImageNet &  4.68 & 15.01 \\ \bottomrule
\end{tabular}%
}
\end{table}

In this worst case situation, the adversary's predictions are little better than random-guessing performance. For MNIST, SVHN, and CIFAR-10 that would be 10\%, with SVHN having the best results at just 2.1$\times$ better than random-guessing. For CIFAR-100 and Mini-ImageNet random-guessing is 1\%, and we see high ratios at $2.6$ and $4.7\times$, but the total accuracy is still far below the 60\% and 51\% obtainable by knowing the secret $\boldsymbol{s}$. This shows that our approach is highly effective at obscuring the information from the untrusted party, forcing them near random-guessing performance even in unrealistically powerful settings. This also reinforces that our approach is not real encryption, and should not be used when provable security is a requirement. Our results do indicate strong empirical security though, and the only method that is practical from a runtime perspective. While multiple classifications of the same image with different secrets improves accuracy for the user, the same can not be said for the adversary. \autoref{fig:unrealistic_avg_prediction} shows that attack success improves by a minor amount only on MNIST for a network trained to take in $k$ pairs of an image bound with different secrets.

\begin{figure}[!h]
\centerline{\includegraphics[keepaspectratio, width=0.9\columnwidth]{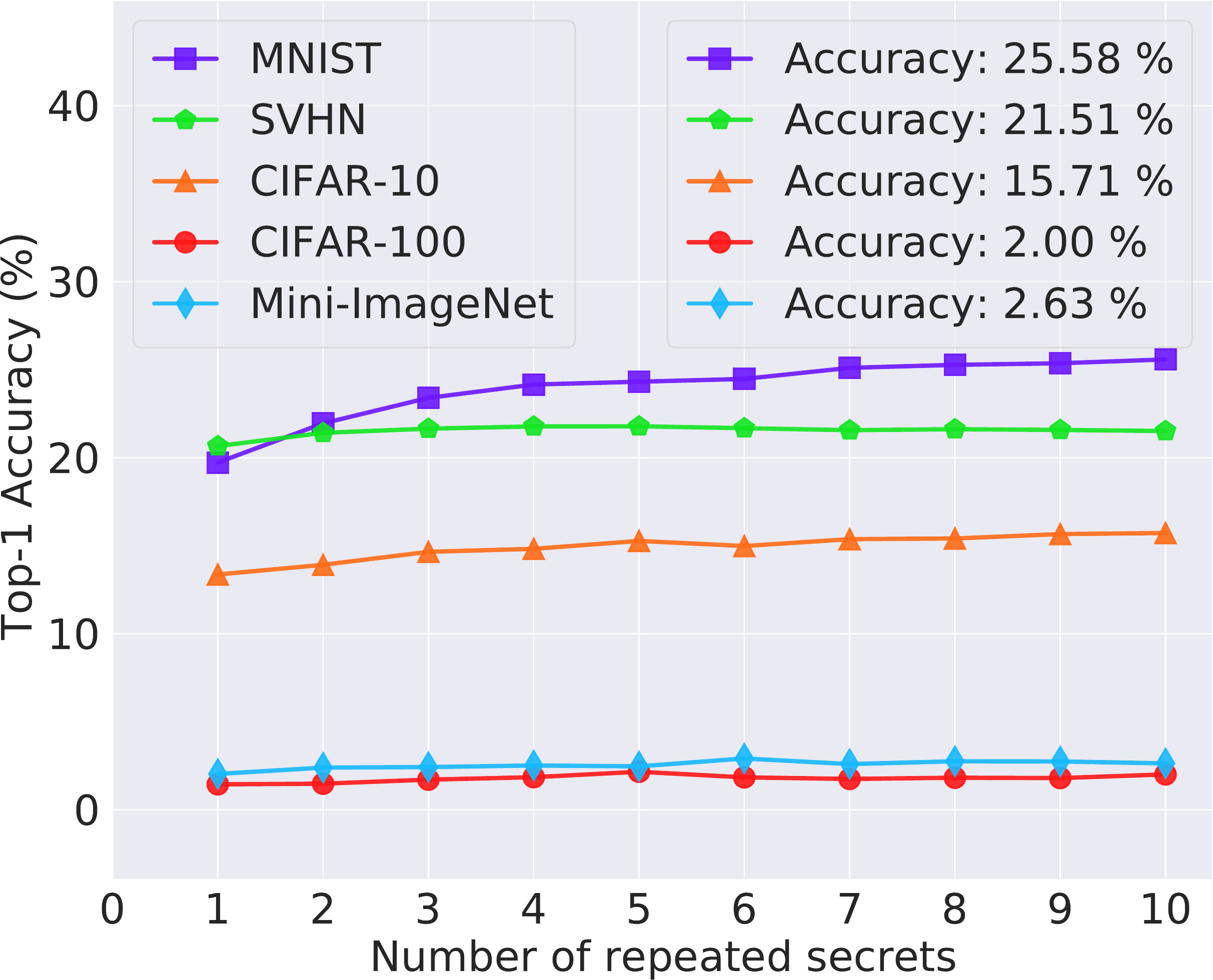}} 
\caption{Accuracy (y-axis) of \textbf{Unrealistic Adversary} after averaging $k$ predictions (x-axis) of the secret binding/unbinding.}
\label{fig:unrealistic_avg_prediction}
\end{figure}

\begin{table}[!h]
\centering
\caption{Accuracy predicting the inputs comparing defender \shortNameTxt and unrealistic adversary attacking the inputs}
\label{tab:unrealistic_adversary_acc}
\adjustbox{max width=\columnwidth}{%
\begin{tabular}{@{}ccc@{}}
\toprule
Dataset       & Our \shortNameTxt  &  Unrealistic Adversary   \\ \midrule
MNIST         & 98.51 \%        &  81.23 \% \\
SVHN          & 88.44 \%        &  39.61 \% \\
CIFAR-10      & 78.21 \%        &  43.40 \% \\
CIFAR-100     & 48.84 \%        &  16.58 \% \\
Mini-ImageNet & 40.99 \%        &  16.00 \% \\ \bottomrule
\end{tabular}
}
\end{table}

In the most extreme scenario, the adversary would have access to the bound images along with the class labels and it may train a network to learn class labels directly from the bound images. Results of this experiment are reported in \autoref{tab:unrealistic_adversary_acc}. Even though this unrealistic adversary seemingly performed better with the access of correct class labels (note linear models can get 92\% MNIST accuracy), still without the correct secret to unbind the image, it falls behind our \shortNameTxt method, providing a strong level of individual protection.

\subsection{Model Inversion Adversary}

As our final attack, we use the Frechet Inception Distance (FID)\cite{10.5555/3295222.3295408} to design an inversion attack. Given the bound input $\hat{\boldsymbol{x}} = \boldsymbol{x} \bind \boldsymbol{s}$, the adversary has their own copy of the training data to compute the FID score of $\hat{\boldsymbol{x}} \bind \hat{\boldsymbol{s}}^\dagger$. Then the adversary can optimize their copy of $\hat{\boldsymbol{s}}$ to try and find the secret that will result in a realistic looking image. We remind the reader that without \shortNameTxt, the adversary intrinsically receives the true input $\boldsymbol{x}$ whenever a prediction is made, and so no comparison to a baseline is possible. Our goal is purely to see if an inversion strategy yields cracks in the effectiveness of \shortNameTxt.

\newenvironment{figrow}%
{%
\centering\addtocounter{figure}{1}%
\begin{enumerate}[%
itemsep=2pt, parsep=0em,
label={(\alph*)},
ref={\thefigure.(\alph*)}
]}%
{\end{enumerate}\addtocounter{figure}{-1}}

\begin{figure}[!h]
\begin{figrow}

\item \raisebox{-0.5\height}{
\includegraphics[keepaspectratio, width=0.080\columnwidth]{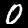} 
\includegraphics[keepaspectratio, width=0.080\columnwidth]{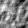}
\includegraphics[keepaspectratio, width=0.080\columnwidth]{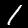}
\includegraphics[keepaspectratio, width=0.080\columnwidth]{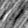}
\includegraphics[keepaspectratio, width=0.080\columnwidth]{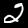}
\includegraphics[keepaspectratio, width=0.080\columnwidth]{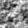}
\includegraphics[keepaspectratio, width=0.080\columnwidth]{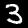}
\includegraphics[keepaspectratio, width=0.080\columnwidth]{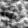}
\includegraphics[keepaspectratio, width=0.080\columnwidth]{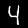}
\includegraphics[keepaspectratio, width=0.080\columnwidth]{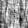}
}

\item \raisebox{-0.5\height}{
\includegraphics[keepaspectratio, width=0.080\columnwidth]{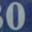} 
\includegraphics[keepaspectratio, width=0.080\columnwidth]{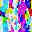}
\includegraphics[keepaspectratio, width=0.080\columnwidth]{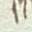}
\includegraphics[keepaspectratio, width=0.080\columnwidth]{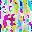}
\includegraphics[keepaspectratio, width=0.080\columnwidth]{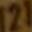}
\includegraphics[keepaspectratio, width=0.080\columnwidth]{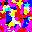}
\includegraphics[keepaspectratio, width=0.080\columnwidth]{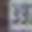}
\includegraphics[keepaspectratio, width=0.080\columnwidth]{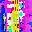}
\includegraphics[keepaspectratio, width=0.080\columnwidth]{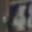}
\includegraphics[keepaspectratio, width=0.080\columnwidth]{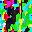} 
}

\item \raisebox{-0.5\height}{
\includegraphics[keepaspectratio, width=0.080\columnwidth]{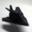} 
\includegraphics[keepaspectratio, width=0.080\columnwidth]{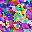}
\includegraphics[keepaspectratio, width=0.080\columnwidth]{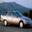}
\includegraphics[keepaspectratio, width=0.080\columnwidth]{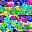}
\includegraphics[keepaspectratio, width=0.080\columnwidth]{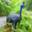}
\includegraphics[keepaspectratio, width=0.080\columnwidth]{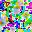}
\includegraphics[keepaspectratio, width=0.080\columnwidth]{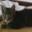}
\includegraphics[keepaspectratio, width=0.080\columnwidth]{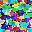}
\includegraphics[keepaspectratio, width=0.080\columnwidth]{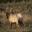}
\includegraphics[keepaspectratio, width=0.080\columnwidth]{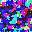} 
}

\item \raisebox{-0.5\height}{
\includegraphics[keepaspectratio, width=0.080\columnwidth]{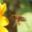} 
\includegraphics[keepaspectratio, width=0.080\columnwidth]{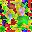}
\includegraphics[keepaspectratio, width=0.080\columnwidth]{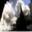}
\includegraphics[keepaspectratio, width=0.080\columnwidth]{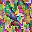}
\includegraphics[keepaspectratio, width=0.080\columnwidth]{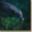}
\includegraphics[keepaspectratio, width=0.080\columnwidth]{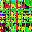}
\includegraphics[keepaspectratio, width=0.080\columnwidth]{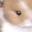}
\includegraphics[keepaspectratio, width=0.080\columnwidth]{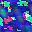}
\includegraphics[keepaspectratio, width=0.080\columnwidth]{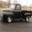}
\includegraphics[keepaspectratio, width=0.080\columnwidth]{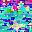} 
}

\item \raisebox{-0.5\height}{
\includegraphics[keepaspectratio, width=0.080\columnwidth]{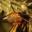} 
\includegraphics[keepaspectratio, width=0.080\columnwidth]{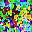}
\includegraphics[keepaspectratio, width=0.080\columnwidth]{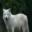}
\includegraphics[keepaspectratio, width=0.080\columnwidth]{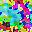}
\includegraphics[keepaspectratio, width=0.080\columnwidth]{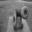}
\includegraphics[keepaspectratio, width=0.080\columnwidth]{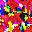}
\includegraphics[keepaspectratio, width=0.080\columnwidth]{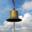}
\includegraphics[keepaspectratio, width=0.080\columnwidth]{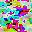}
\includegraphics[keepaspectratio, width=0.080\columnwidth]{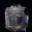}
\includegraphics[keepaspectratio, width=0.080\columnwidth]{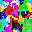} 
}

\end{figrow}
\caption{Model inversion attack by the adversary using projected gradient descent to unbind the bound images given sample of the original images. Images are shown in pairs where the original image is shown in left and generated image is shown in right. Results for MNIST (a), SVHN (b), CIFAR-10 (c), CIFAR-100 (d), and Mini-ImageNet (e) all confirm the adversary can not extract the nature of the bound inputs even when optimizing for visual realism to extract the secret.} \label{fig:inverse_attack}
\end{figure}

Examples of the attack are presented in \autoref{fig:inverse_attack}, showing that inverting the original images is highly challenging. This attack involves the adversary having the true training data, knowing the procedure to extract the input, and using gradient descent to attempt to find a secret that maximizes the apparent realism of the result via FID scores.

A second inversion strategy assumes the adversary has examples of secrets $\boldsymbol{s}$ and encoded images $\boldsymbol{r}$, and attempts to learn an auto-encoder that minimizes $\|\boldsymbol{r}-\boldsymbol{s}\|_2^2$, to directly predict the secret from a bound image. We find that this strategy also fails to provide meaningful results, as demonstrated in \autoref{fig:auto_encoder}. 

\begin{figure}[!h]
\begin{figrow}
\centering
\item \raisebox{-0.5\height}{
\includegraphics[keepaspectratio, width=0.080\columnwidth]{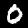} 
\includegraphics[keepaspectratio, width=0.080\columnwidth]{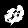}
\includegraphics[keepaspectratio, width=0.080\columnwidth]{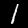}
\includegraphics[keepaspectratio, width=0.080\columnwidth]{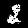}
\includegraphics[keepaspectratio, width=0.080\columnwidth]{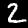}
\includegraphics[keepaspectratio, width=0.080\columnwidth]{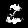}
\includegraphics[keepaspectratio, width=0.080\columnwidth]{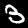}
\includegraphics[keepaspectratio, width=0.080\columnwidth]{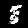}
}

\item \raisebox{-0.5\height}{
\includegraphics[keepaspectratio, width=0.080\columnwidth]{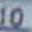} 
\includegraphics[keepaspectratio, width=0.080\columnwidth]{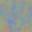}
\includegraphics[keepaspectratio, width=0.080\columnwidth]{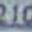}
\includegraphics[keepaspectratio, width=0.080\columnwidth]{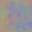}
\includegraphics[keepaspectratio, width=0.080\columnwidth]{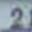}
\includegraphics[keepaspectratio, width=0.080\columnwidth]{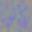}
\includegraphics[keepaspectratio, width=0.080\columnwidth]{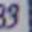}
\includegraphics[keepaspectratio, width=0.080\columnwidth]{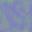}
}

\item \raisebox{-0.5\height}{
\includegraphics[keepaspectratio, width=0.080\columnwidth]{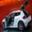}
\includegraphics[keepaspectratio, width=0.080\columnwidth]{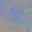}
\includegraphics[keepaspectratio, width=0.080\columnwidth]{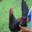}
\includegraphics[keepaspectratio, width=0.080\columnwidth]{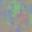}
\includegraphics[keepaspectratio, width=0.080\columnwidth]{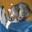}
\includegraphics[keepaspectratio, width=0.080\columnwidth]{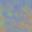}
\includegraphics[keepaspectratio, width=0.080\columnwidth]{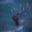}
\includegraphics[keepaspectratio, width=0.080\columnwidth]{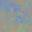}
}

\item \raisebox{-0.5\height}{
\includegraphics[keepaspectratio, width=0.080\columnwidth]{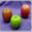} 
\includegraphics[keepaspectratio, width=0.080\columnwidth]{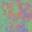}
\includegraphics[keepaspectratio, width=0.080\columnwidth]{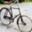}
\includegraphics[keepaspectratio, width=0.080\columnwidth]{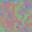}
\includegraphics[keepaspectratio, width=0.080\columnwidth]{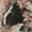}
\includegraphics[keepaspectratio, width=0.080\columnwidth]{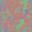}
\includegraphics[keepaspectratio, width=0.080\columnwidth]{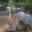}
\includegraphics[keepaspectratio, width=0.080\columnwidth]{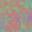}
}

\item \raisebox{-0.5\height}{
\includegraphics[keepaspectratio, width=0.080\columnwidth]{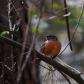} 
\includegraphics[keepaspectratio, width=0.080\columnwidth]{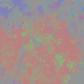}
\includegraphics[keepaspectratio, width=0.080\columnwidth]{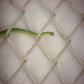}
\includegraphics[keepaspectratio, width=0.080\columnwidth]{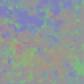}
\includegraphics[keepaspectratio, width=0.080\columnwidth]{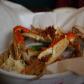}
\includegraphics[keepaspectratio, width=0.080\columnwidth]{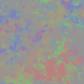}
\includegraphics[keepaspectratio, width=0.080\columnwidth]{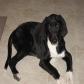}
\includegraphics[keepaspectratio, width=0.080\columnwidth]{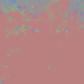}
}

\end{figrow}
\caption{Inversion attack using a trained auto-encoder to directly predict the secret $\boldsymbol{s}$. Results for MNIST (a), SVHN (b), CIFAR-10 (c), CIFAR-100 (d), and Mini-ImageNet (e) show that the adversary can not meaningfully predict the secret from new outputs $\boldsymbol{r}$. }
\label{fig:auto_encoder}
\end{figure}

\section{Conclusion and Limitations} \label{sec:conclusion}
We have shown the first approach that allows deploying the majority of computational effort onto an untrusted third party, that does not impose compute or communication overheads larger than performing the computation locally. This makes it the only current option for deployment scenarios that are runtime or power-constrained, such as on mobile devices. While the method is not provably secure, limiting use to highly sensitive data, it shows strong empirical robustness to unrealistically powerful adversaries. Compared to state-of-the-art alternatives our \shortNameTxt is $\approx5000\times$ faster and sends $\approx18,000\times$ less data per query.

Our approach is validated by attacking it at a threat model stronger than what we intend it to be used for. We stress this intended treat model of \shortNameTxt as a limitation because we can not prove it's effectiveness in the more general context. Our hope is that the numerous attacks at a higher threat model will mean that any future successful attack against \shortNameTxt will thus require some scientific advancement, but also means we must acknowledge the risk of unidentified attacks.  It should thus be used with high caution for any application where privacy is a hard requirement. The scalability of \shortNameTxt is also limited, we found it unable to reach full ImageNet scale in its current form - a limitation shared by other approaches, but still a constraint toward many applications. 

\section*{Acknowledgments}

We thank Frank Ferraro for his helpful review of this paper, and the anonymous reviewers for their additional experiments that have further strengthened the paper's empirical conclusions. 

\bibliography{refs}
\bibliographystyle{icml2022}

\newpage
\appendix
\onecolumn

\FloatBarrier

\section{Network Model Details} \label{sec:train_details}
To train the proposed model, first both training and augmented data are bound with a randomly generated secret sampled from a normal distribution with complex unit magnitude projection using improved 2D HRR. The bound image is then transferred to the main network, i.e., U-net. Next, the output of the U-net is received to the user-end and exploited by both the prediction and the adversarial network. In the prediction network, the output of the U-net is unbound using the secret, and class labels are predicted. On the other hand, the same output of the U-net is fed to the adversarial network by reversing the gradient which also predicts the class labels. As no secret unbinding is performed in the adversarial network, it will make the classification harder if someone tries to classify labels using the output of the U-net without doing secret unbinding. 
\par 
The input and output dimension of each of the network solely depends on the dataset. The size of the prediction and adversarial network also depends on the dataset and the number of parameters increases as the dimension of the images increases. For instance, CIFAR-10 dataset has RGB images of dimension $(32\times32\times3)$ in $10$ different classes. Therefore, the input and output dimension of the U-net is $(32\times32\times3)$. As for the prediction and adversarial network, three blocks of Convolutional and Maxpooling layer are used back-to-back with the number of filters $32$, $64$, $128$ in each block. The output dimension of the final block is $(4\times4\times128)$ which is flattened to a vector of length $2048$. Next, three FC layers are used back-to-back, and the dimension of the vector is reduced from $2048$ to $1024$, $512$, and finally, $10$ which is the number of classes in the CIFAR-10 dataset.
\par 
The network is optimized using Adam optimizer with a learning rate of $10^{-3}$ which is subsequently reduced by $10$ fold after 200 epochs for better convergence. The batch size is chosen is to be $64$ and each of the networks is trained until convergence up to 400 epochs.
\par 
Below we show code snippets of each network's definition for a complete specification of its design, layers, activation, and neurons. Complete code is provided in the supplemental material.

\begin{figure}[!h]
    \centering
\begin{minted}[mathescape=true]{python}
def forward(self, x, key):
    x_main = self.F_main(x) #$\boldsymbol{r} = f_W(\boldsymbol{x}\bind \boldsymbol{s})$
    y_attack = self.F_attact(x_main) #$\hat{y}_A = \operatorname{ReverseGrad}\left(f_A(\boldsymbol{x})\right)$
    x_unbind = unbinding_2d(x_main, key) #$\mathit{tmp} = \boldsymbol{r} \bind \boldsymbol{s}^\dagger$
    y_pred = self.F_pred(x_unbind) #$\hat{y}_P = f_P(\mathit{tmp}) = f_P(\boldsymbol{r} \bind \boldsymbol{s}^\dagger)$
    return y_pred, y_attack, x_main
\end{minted}
    \caption{Forward function signature used by all of our networks. }
\end{figure}

\begin{figure}[!h]
    \centering
\begin{minted}[fontsize=\footnotesize]{python}
class NetworkMiniImageNet(nn.Module):
    def __init__(self, activation=nn.LeakyReLU(0.1)):
        super().__init__()

        self.F_main = nn.Sequential(
            unet.UNet2D(3, 3)
        )

        self.F_attact = nn.Sequential(
            RevGrad(),
            nn.Conv2d(3, 64, (3, 3), padding=(1, 1)), nn.BatchNorm2d(64), activation,
            nn.MaxPool2d((2, 2)), nn.Dropout(0.25),
            nn.Conv2d(64, 128, (3, 3), padding=(1, 1)), nn.BatchNorm2d(128), activation,
            nn.MaxPool2d((2, 2)), nn.Dropout(0.25),
            nn.Conv2d(128, 256, (2, 2), padding=(1, 1)), nn.BatchNorm2d(256), activation,
            nn.MaxPool2d((2, 2)), nn.Dropout(0.25),
            nn.Conv2d(256, 256, (3, 3), padding=(1, 1)), nn.BatchNorm2d(256), activation,
            nn.MaxPool2d((2, 2)), nn.Dropout(0.25),
            nn.Flatten(),
            nn.Linear(6400, 2048), activation, nn.BatchNorm1d(2048),
            nn.Dropout(p=0.5),
            nn.Linear(2048, 1024), activation, nn.BatchNorm1d(1024),
            nn.Dropout(p=0.5),
            nn.Linear(1024, 100), nn.Softmax(dim=-1)
        )

        self.F_pred = nn.Sequential(
            nn.Conv2d(3, 64, (3, 3), padding=(1, 1)), nn.BatchNorm2d(64), activation,
            nn.MaxPool2d((2, 2)), nn.Dropout(0.25),
            nn.Conv2d(64, 128, (3, 3), padding=(1, 1)), nn.BatchNorm2d(128), activation,
            nn.MaxPool2d((2, 2)), nn.Dropout(0.25),
            nn.Conv2d(128, 256, (2, 2), padding=(1, 1)), nn.BatchNorm2d(256), activation,
            nn.MaxPool2d((2, 2)), nn.Dropout(0.25),
            nn.Conv2d(256, 256, (3, 3), padding=(1, 1)), nn.BatchNorm2d(256), activation,
            nn.MaxPool2d((2, 2)), nn.Dropout(0.25),
            nn.Flatten(),
            nn.Linear(6400, 2048), activation, nn.BatchNorm1d(2048),
            nn.Dropout(p=0.5),
            nn.Linear(2048, 1024), activation, nn.BatchNorm1d(1024),
            nn.Dropout(p=0.5),
            nn.Linear(1024, 100), nn.Softmax(dim=-1)
        )
\end{minted}
    \caption{Network for Mini-ImageNet results.}
\end{figure}

\begin{figure}[!h]
    \centering
\begin{minted}[fontsize=\footnotesize]{python}
class NetworkCIFAR100(nn.Module):
    def __init__(self, activation=nn.LeakyReLU(0.1)):
        super().__init__()

        self.F_main = nn.Sequential(
            unet.UNet2D(3, 3)
        )

        self.F_attact = nn.Sequential(
            RevGrad(),
            nn.Conv2d(3, 64, (3, 3), padding=(1, 1)), nn.BatchNorm2d(64), activation,
            nn.MaxPool2d((2, 2)), nn.Dropout(0.25),
            nn.Conv2d(64, 128, (3, 3), padding=(1, 1)), nn.BatchNorm2d(128), activation,
            nn.MaxPool2d((2, 2)), nn.Dropout(0.25),
            nn.Conv2d(128, 256, (3, 3), padding=(1, 1)), nn.BatchNorm2d(256), activation,
            nn.MaxPool2d((2, 2)), nn.Dropout(0.25),
            nn.Flatten(),
            nn.Linear(4096, 2048), activation, nn.BatchNorm1d(2048),
            nn.Dropout(p=0.5),
            nn.Linear(2048, 1024), activation, nn.BatchNorm1d(1024),
            nn.Dropout(p=0.5),
            nn.Linear(1024, 100), nn.Softmax(dim=-1)
        )

        self.F_pred = nn.Sequential(
            nn.Conv2d(3, 64, (3, 3), padding=(1, 1)), nn.BatchNorm2d(64), activation,
            nn.MaxPool2d((2, 2)), nn.Dropout(0.25),
            nn.Conv2d(64, 128, (3, 3), padding=(1, 1)), nn.BatchNorm2d(128), activation,
            nn.MaxPool2d((2, 2)), nn.Dropout(0.25),
            nn.Conv2d(128, 256, (3, 3), padding=(1, 1)), nn.BatchNorm2d(256), activation,
            nn.MaxPool2d((2, 2)), nn.Dropout(0.25),
            nn.Flatten(),
            nn.Linear(4096, 2048), activation, nn.BatchNorm1d(2048),
            nn.Dropout(p=0.5),
            nn.Linear(2048, 1024), activation, nn.BatchNorm1d(1024),
            nn.Dropout(p=0.5),
            nn.Linear(1024, 100), nn.Softmax(dim=-1)
        )
\end{minted}
    \caption{Network for CIFAR-100 results.}
\end{figure}

\begin{figure}[!h]
    \centering
\begin{minted}[fontsize=\footnotesize]{python}
class NetworkCIFAR10(nn.Module):
    def __init__(self, activation=nn.LeakyReLU(0.1)):
        super().__init__()

        self.F_main = nn.Sequential(
            unet.UNet2D(3, 3)
        )

        self.F_attact = nn.Sequential(
            RevGrad(),
            nn.Conv2d(3, 32, (3, 3), padding=(1, 1)), nn.BatchNorm2d(32), activation,
            nn.MaxPool2d((2, 2)), nn.Dropout(0.25),
            nn.Conv2d(32, 64, (3, 3), padding=(1, 1)), nn.BatchNorm2d(64), activation,
            nn.MaxPool2d((2, 2)), nn.Dropout(0.25),
            nn.Conv2d(64, 128, (3, 3), padding=(1, 1)), nn.BatchNorm2d(128), activation,
            nn.MaxPool2d((2, 2)), nn.Dropout(0.25),
            nn.Flatten(),
            nn.Linear(2048, 1024), activation, nn.BatchNorm1d(1024),
            nn.Dropout(p=0.5),
            nn.Linear(1024, 512), activation, nn.BatchNorm1d(512),
            nn.Dropout(p=0.5),
            nn.Linear(512, 10), nn.Softmax(dim=-1)
        )

        self.F_pred = nn.Sequential(
            nn.Conv2d(3, 32, (3, 3), padding=(1, 1)), nn.BatchNorm2d(32), activation,
            nn.MaxPool2d((2, 2)), nn.Dropout(0.25),
            nn.Conv2d(32, 64, (3, 3), padding=(1, 1)), nn.BatchNorm2d(64), activation,
            nn.MaxPool2d((2, 2)), nn.Dropout(0.25),
            nn.Conv2d(64, 128, (3, 3), padding=(1, 1)), nn.BatchNorm2d(128), activation,
            nn.MaxPool2d((2, 2)), nn.Dropout(0.25),
            nn.Flatten(),
            nn.Linear(2048, 1024), activation, nn.BatchNorm1d(1024),
            nn.Dropout(p=0.5),
            nn.Linear(1024, 512), activation, nn.BatchNorm1d(512),
            nn.Dropout(p=0.5),
            nn.Linear(512, 10), nn.Softmax(dim=-1)
        )
\end{minted}
    \caption{Network for CIFAR-10 results.}
\end{figure}

\begin{figure}[!h]
    \centering
\begin{minted}[fontsize=\footnotesize]{python}
class NetworkSVHN(nn.Module):
    def __init__(self, activation=nn.LeakyReLU(0.1)):
        super().__init__()

        self.F_main = nn.Sequential(
            unet.UNet2D(3, 3)
        )

        self.F_attact = nn.Sequential(
            RevGrad(),
            nn.Conv2d(3, 32, (3, 3), padding=(1, 1)), nn.BatchNorm2d(32), activation,
            nn.MaxPool2d((2, 2)), nn.Dropout(0.25),
            nn.Conv2d(32, 64, (3, 3), padding=(1, 1)), nn.BatchNorm2d(64), activation,
            nn.MaxPool2d((2, 2)), nn.Dropout(0.25),
            nn.Conv2d(64, 128, (3, 3), padding=(1, 1)), nn.BatchNorm2d(128), activation,
            nn.MaxPool2d((2, 2)), nn.Dropout(0.25),
            nn.Flatten(),
            nn.Linear(2048, 1024), activation, nn.BatchNorm1d(1024),
            nn.Dropout(p=0.5),
            nn.Linear(1024, 512), activation, nn.BatchNorm1d(512),
            nn.Dropout(p=0.5),
            nn.Linear(512, 10), nn.Softmax(dim=-1)
        )

        self.F_pred = nn.Sequential(
            nn.Conv2d(3, 32, (3, 3), padding=(1, 1)), nn.BatchNorm2d(32), activation,
            nn.MaxPool2d((2, 2)), nn.Dropout(0.25),
            nn.Conv2d(32, 64, (3, 3), padding=(1, 1)), nn.BatchNorm2d(64), activation,
            nn.MaxPool2d((2, 2)), nn.Dropout(0.25),
            nn.Conv2d(64, 128, (3, 3), padding=(1, 1)), nn.BatchNorm2d(128), activation,
            nn.MaxPool2d((2, 2)), nn.Dropout(0.25),
            nn.Flatten(),
            nn.Linear(2048, 1024), activation, nn.BatchNorm1d(1024),
            nn.Dropout(p=0.5),
            nn.Linear(1024, 512), activation, nn.BatchNorm1d(512),
            nn.Dropout(p=0.5),
            nn.Linear(512, 10), nn.Softmax(dim=-1)
        )
\end{minted}
    \caption{Network for SVHN results.}
\end{figure}

\begin{figure}[!h]
    \centering
\begin{minted}[fontsize=\footnotesize]{python}
class NetworkMNIST(nn.Module):
    def __init__(self, activation=nn.LeakyReLU(0.1)):
        super().__init__()

        self.F_main = nn.Sequential(
            nn.Conv2d(1, 3, (1, 1)),
            unet.UNet2D(3, 3),
            nn.Conv2d(3, 1, (1, 1)),
        )

        self.F_attact = nn.Sequential(
            RevGrad(),
            nn.Conv2d(1, 32, (3, 3), padding=(1, 1)), nn.BatchNorm2d(32), activation,
            nn.MaxPool2d((2, 2)), nn.Dropout(0.25),
            nn.Conv2d(32, 64, (3, 3), padding=(1, 1)), nn.BatchNorm2d(64), activation,
            nn.MaxPool2d((2, 2)), nn.Dropout(0.25),
            nn.Conv2d(64, 128, (3, 3), padding=(1, 1)), nn.BatchNorm2d(128), activation,
            nn.MaxPool2d((2, 2)), nn.Dropout(0.25),
            nn.Flatten(),
            nn.Linear(1152, 1024), activation, nn.BatchNorm1d(1024),
            nn.Dropout(p=0.5),
            nn.Linear(1024, 512), activation, nn.BatchNorm1d(512),
            nn.Dropout(p=0.5),
            nn.Linear(512, 10), nn.Softmax(dim=-1)
        )

        self.F_pred = nn.Sequential(
            nn.Conv2d(1, 32, (3, 3), padding=(1, 1)), nn.BatchNorm2d(32), activation,
            nn.MaxPool2d((2, 2)), nn.Dropout(0.25),
            nn.Conv2d(32, 64, (3, 3), padding=(1, 1)), nn.BatchNorm2d(64), activation,
            nn.MaxPool2d((2, 2)), nn.Dropout(0.25),
            nn.Conv2d(64, 128, (3, 3), padding=(1, 1)), nn.BatchNorm2d(128), activation,
            nn.MaxPool2d((2, 2)), nn.Dropout(0.25),
            nn.Flatten(),
            nn.Linear(1152, 1024), activation, nn.BatchNorm1d(1024),
            nn.Dropout(p=0.5),
            nn.Linear(1024, 512), activation, nn.BatchNorm1d(512),
            nn.Dropout(p=0.5),
            nn.Linear(512, 10), nn.Softmax(dim=-1)
        )
\end{minted}
    \caption{Network for MNIST results.}
\end{figure}

\FloatBarrier
\section{Dataset and Training Details} \label{sec:data_train_details}

Additionally, during training, data is augmented by applying random horizontal flip $(50\%)$, random brightness $(0.7, 1.3)$, random contrast $(0.7, 1.3)$, random crop $(0.7, 1.3)$, random rotation $(-60, 60)$, random translation $(0.0, 0.25)$, random scaling $(0.9, 1.1)$, and Gaussian blur with standard deviation of $(0.5, 1.5)$. Therefore, the network will learn from a large dataset and be more likely to be regularized.

\begin{table}[!htbp]
\centering
\caption{Datasets Description}
\label{tab:dataset}
\renewcommand{\arraystretch}{1.2}
\resizebox{8.4cm}{!}{%
\begin{tabular}{@{}ccccc@{}}
\toprule
Dataset       & Dimension           & Classes & Training  & Testing  \\ \midrule
MNIST         & $28\times28\times1$ & $10$               & $50000$           & $10000$          \\
SVHN          & $32\times32\times3$ & $10$               & $73257$           & $26032$          \\
CIFAR-10      & $32\times32\times3$ & $10$               & $50000$           & $10000$          \\
CIFAR-100     & $32\times32\times3$ & $100$              & $50000$           & $10000$          \\
Mini-ImageNet & $84\times84\times3$ & $100$              & $50000$           & $10000$          \\ \bottomrule
\end{tabular}
}
\end{table}

\FloatBarrier
\section{Further Exposition on Averaged Predictions} \label{sec:avg_pred_extra}

Our results from \autoref{fig:avg_prediction} show improved accuracy and have a subtle benefit. Under existing protocols discussed in \autoref{sec:related work}, performing $k$ predictions would be $k$ times as expensive. We would expect this to be $< k$ times as expensive in a real-world deployment, as it allows amortizing the communication overhead/latency once for $k$ predictions. It is also well known that batched computations are more compute efficiently, so sending $k$ predictions to perform instead of 1 would be less than $k$ times as much compute on the party running the backbone network $f_W(\cdot)$.
\par 
The reader may rightfully wonder if this decreases the security of our method, as some amount of information is revealed if it is known that the $k$ items all represent the same input with different secrets. We do not expect this to be the case to any meaningful degree, given our results in \autoref{sec:results} showing clustering being ineffective at the extraction of information and even still being robust to omnipotent adversaries that are aware of the training approach, dataset, and class labels. 
\par 
Further, if this was a concern there are strategies that could be employed to further complicate life for the untrusted party. This includes sending additional random/fake prediction tasks, interleaving the $k$ replicates with other prediction tasks and other possible strategies that hide this information. This gets into a game-theoretic exercise beyond the scope of our work and is of limited importance given that we explicitly do not aim for provable security.

\FloatBarrier
\section{Clustering Visual Results} \label{sec:cluster_visual_results}

Below are plots of the true class distribution, followed by the cluster's identified classes, for all datasets. This is most legible on MNIST (\autoref{fig:clustering_all_mnist}), SVHN (\autoref{fig:clustering_all_svhn}), and CIFAR-10 (\autoref{fig:clustering_all_cifar10}) due to the lower number of classes. They show clearly that the true class distribution appears random in the output space when the secret $\boldsymbol{s}$ is not available, and the clusters identified are latching onto a false manifold produced by the model. The results for CIFAR-100 (\autoref{fig:clustering_all_cifar100}) and Mini-ImageNet (\autoref{fig:clustering_all_miniimagenet}) are qualitatively the same, but hard to reader due to the issues in plotting 100 classes makes many different classes ``look'' similar due to the same coloring. 
\par 
In \autoref{clustering-unbound} we show all datasets UMAP plots of the true class labels after unbinding the secret $\boldsymbol{s}$, which shows how dramatically the output space is changed. This is clearest for MNIST, but we find an unusually consistent grouping for all other datasets. Each group has a mix of classes in it that can be separated with reasonable accuracy (as evidenced by our results), but the totality of the behavior is not yet fully understood. The most important result from this though is to demonstrate the significance of obfuscating impact that $\boldsymbol{s}$ has on the manifold of the data.

\begin{figure}[!h]
\centering 
\subfigure[True Class]
{\includegraphics[keepaspectratio, width=0.32\columnwidth]{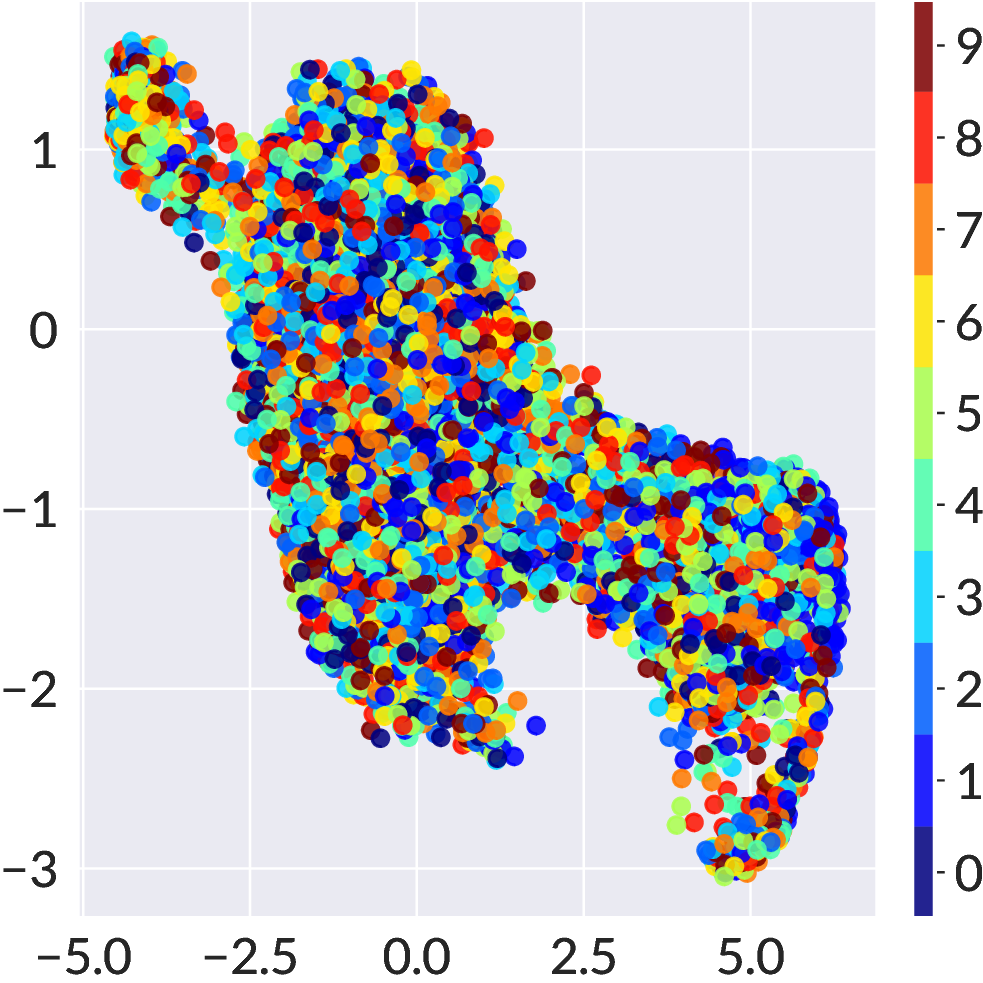}}
\subfigure[Kmeans]
{\includegraphics[keepaspectratio, width=0.32\columnwidth]{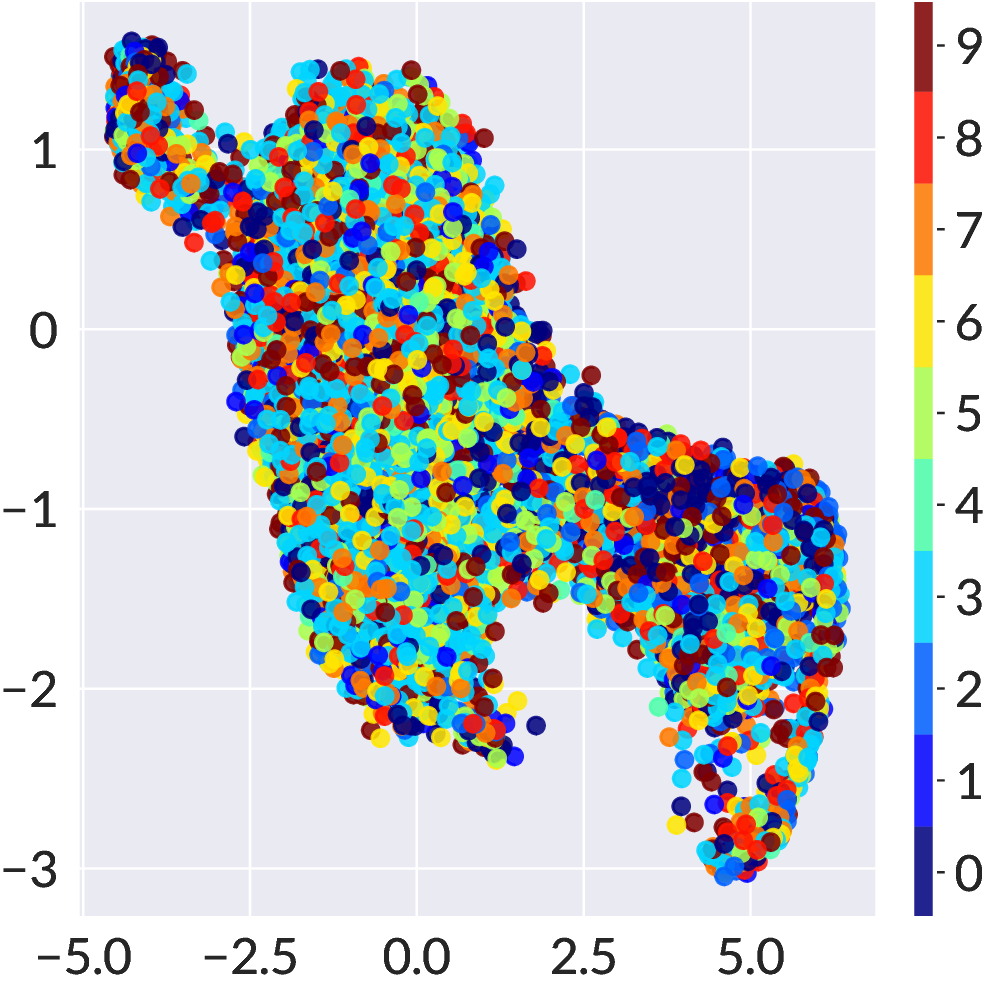}}
\subfigure[Spectral]
{\includegraphics[keepaspectratio, width=0.32\columnwidth]{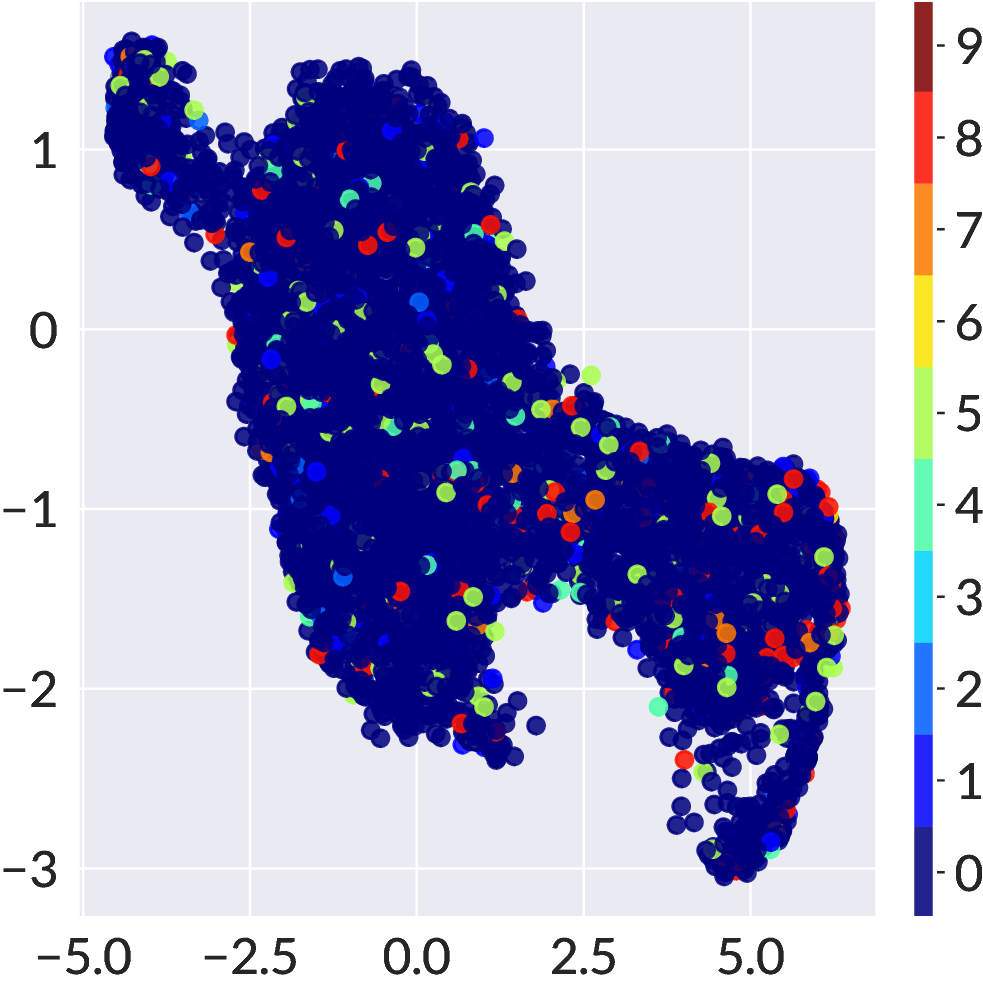}}
\subfigure[GMM]
{\includegraphics[keepaspectratio, width=0.32\columnwidth]{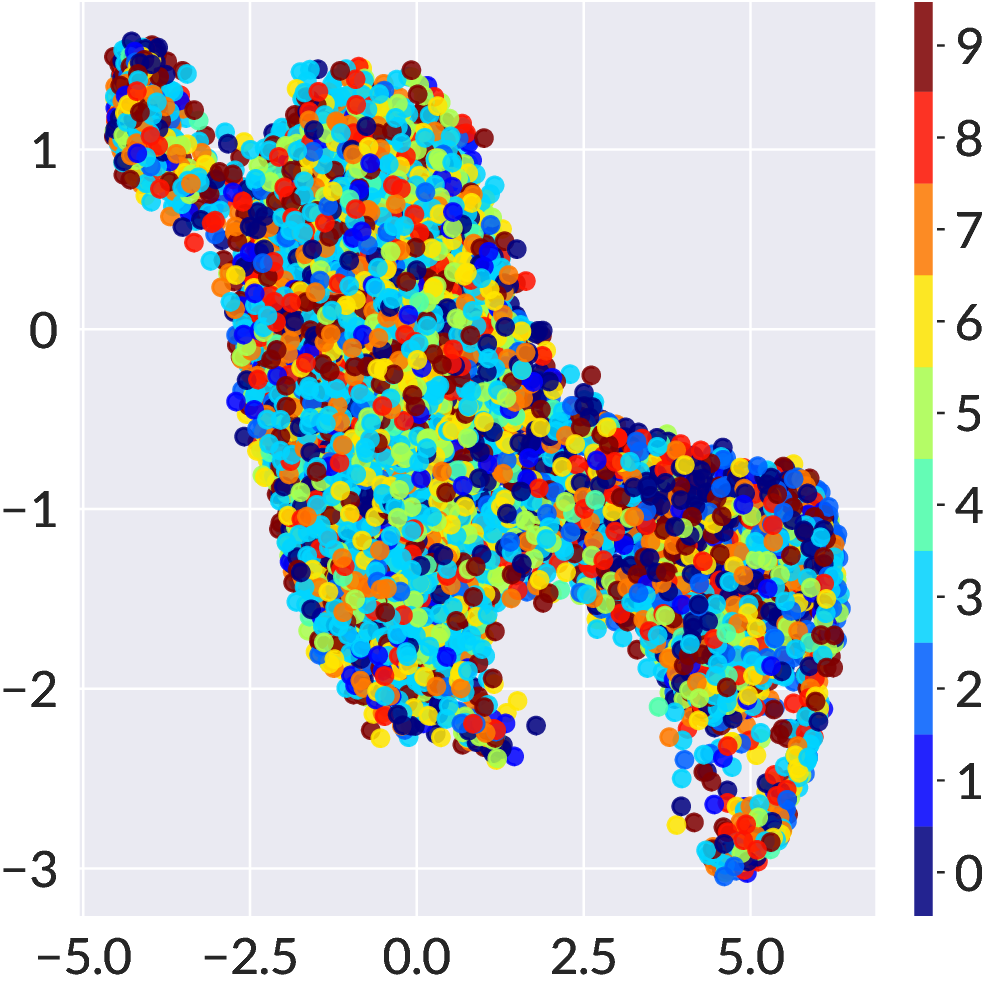}}
\subfigure[Birch]
{\includegraphics[keepaspectratio, width=0.32\columnwidth]{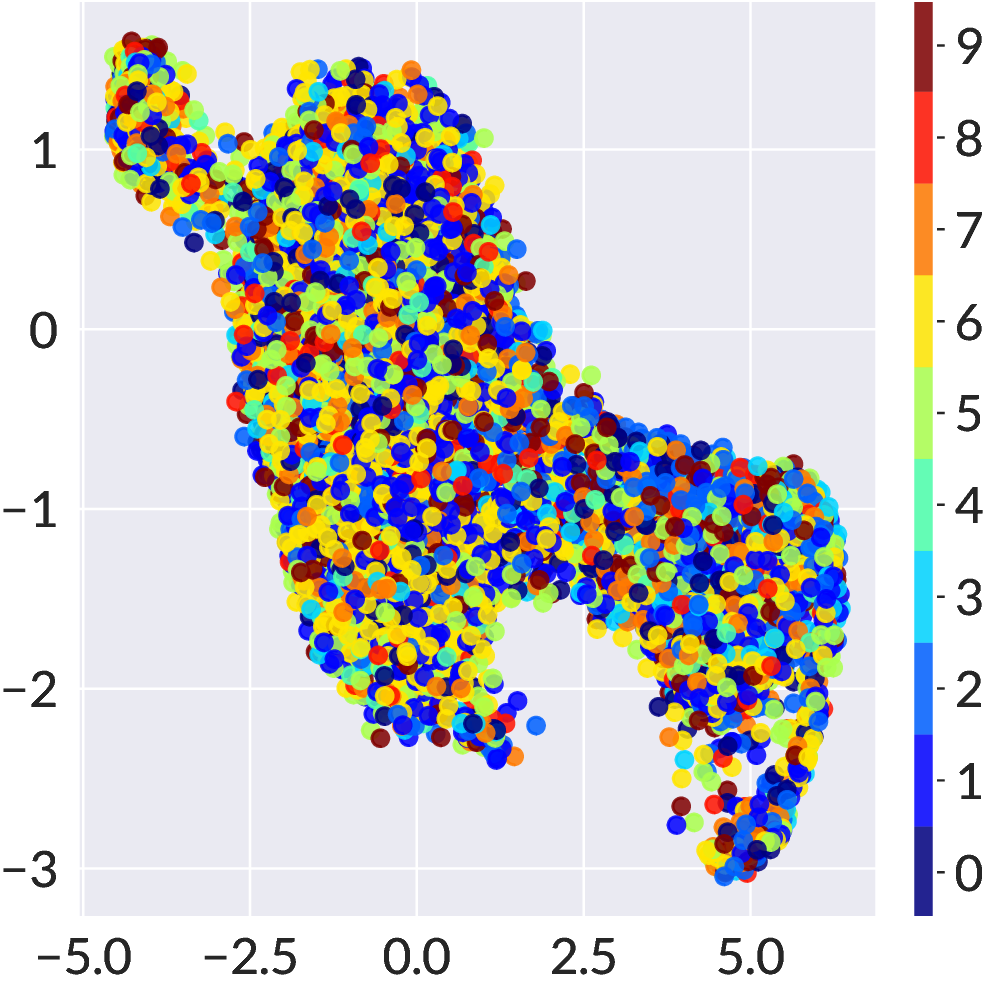}}
\subfigure[HDBSCAN]
{\includegraphics[keepaspectratio, width=0.32\columnwidth]{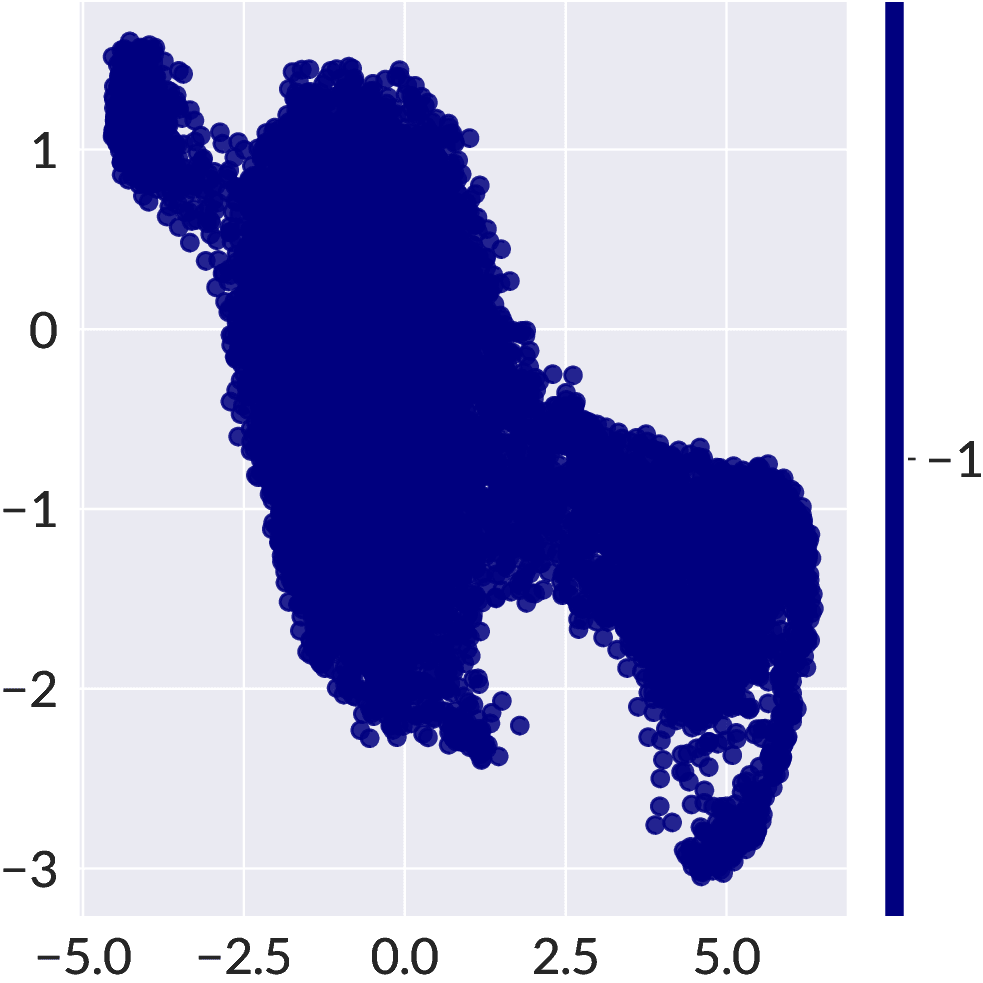}}
\caption{MNIST2 UMAP 2D Embedding representation of the clustering of the output of the U-Net.}
\label{fig:clustering_all_mnist}
\end{figure}

\begin{figure}[!h]
\centering 
\subfigure[True Class]
{\includegraphics[keepaspectratio, width=0.32\columnwidth]{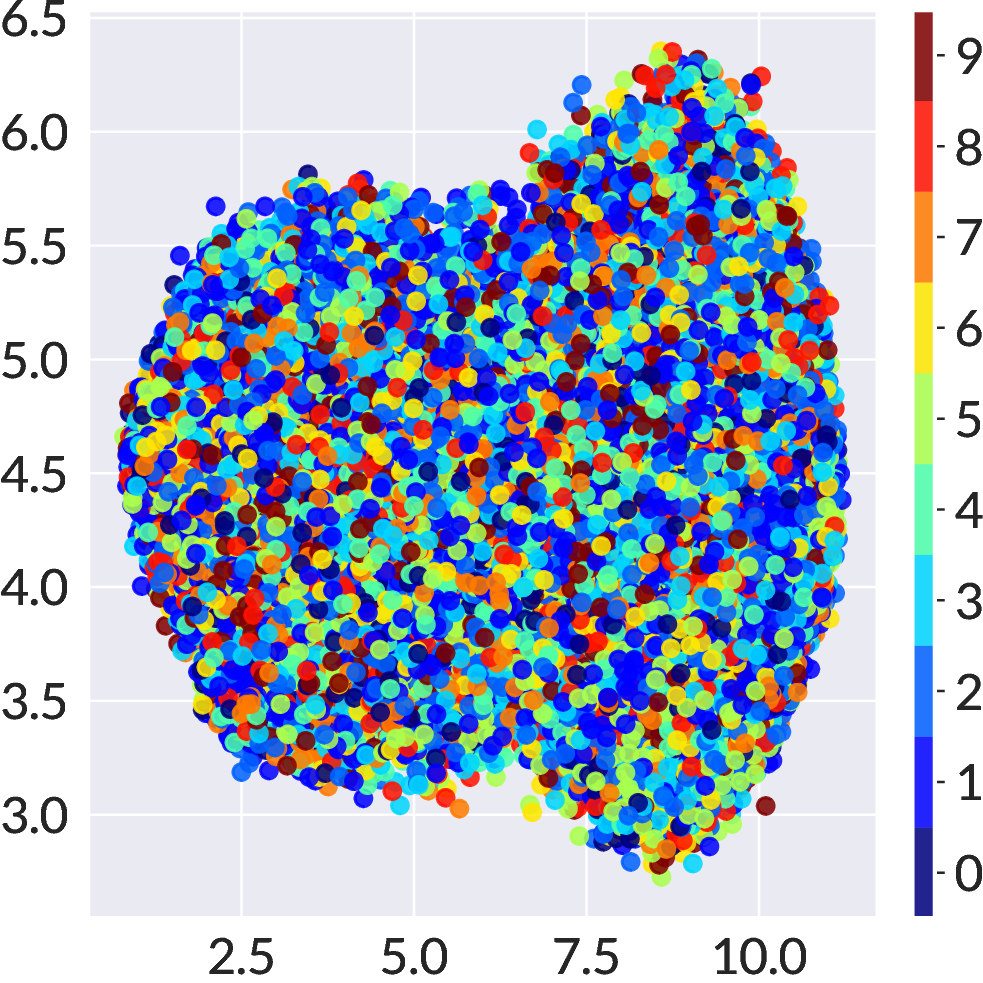}}
\subfigure[Kmeans]
{\includegraphics[keepaspectratio, width=0.32\columnwidth]{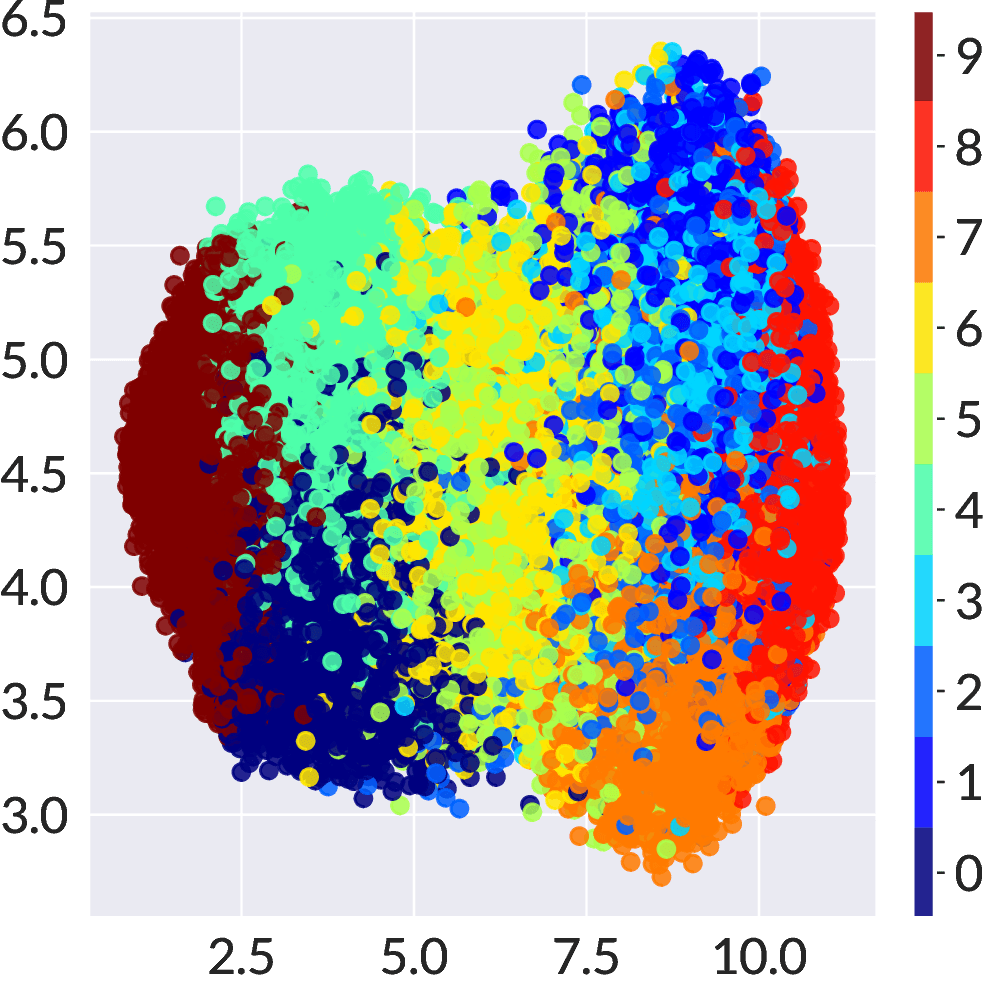}}
\subfigure[Spectral]
{\includegraphics[keepaspectratio, width=0.32\columnwidth]{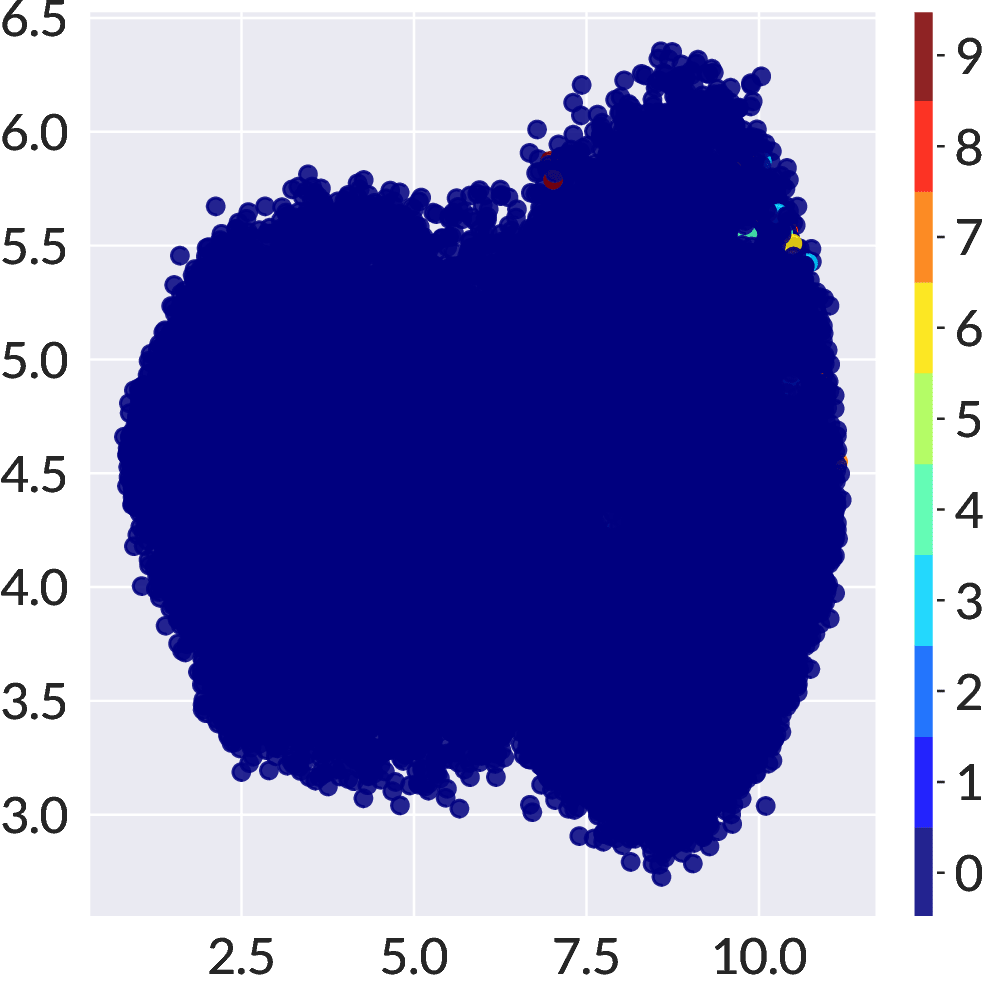}}
\subfigure[GMM]
{\includegraphics[keepaspectratio, width=0.32\columnwidth]{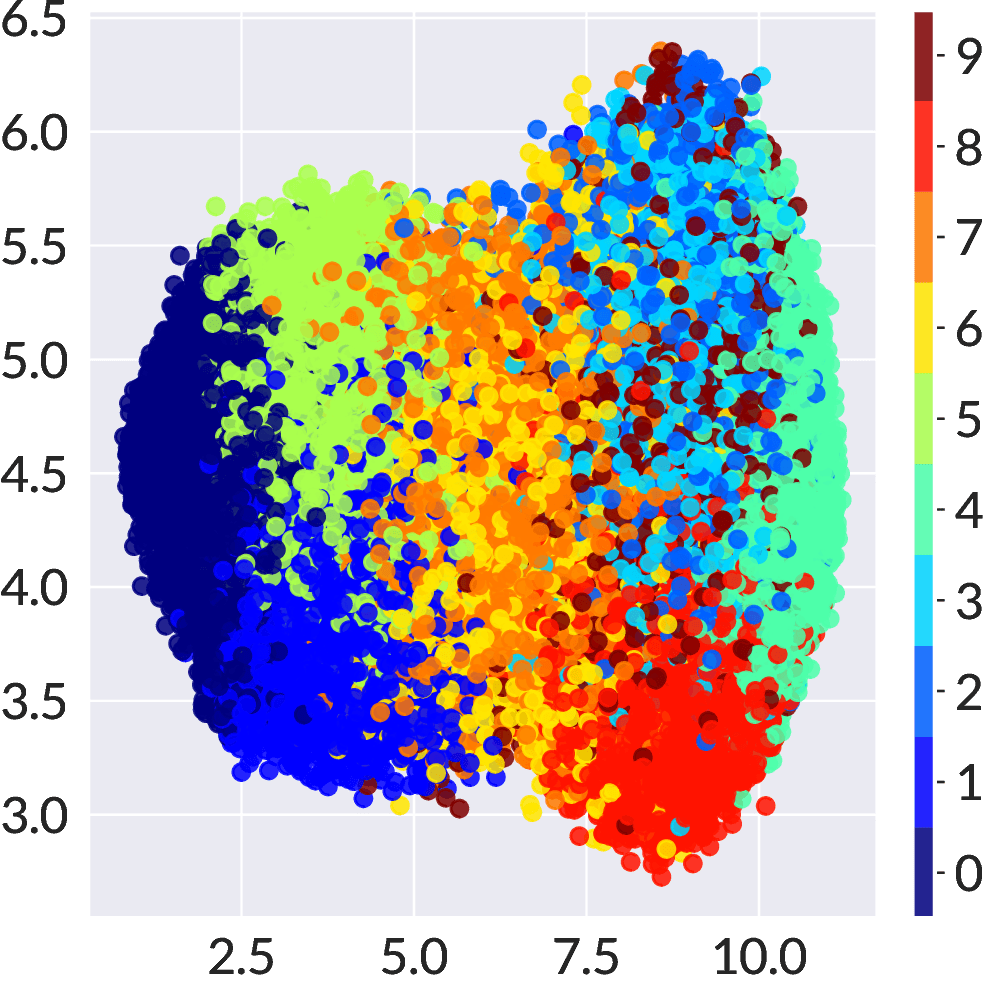}}
\subfigure[Birch]
{\includegraphics[keepaspectratio, width=0.32\columnwidth]{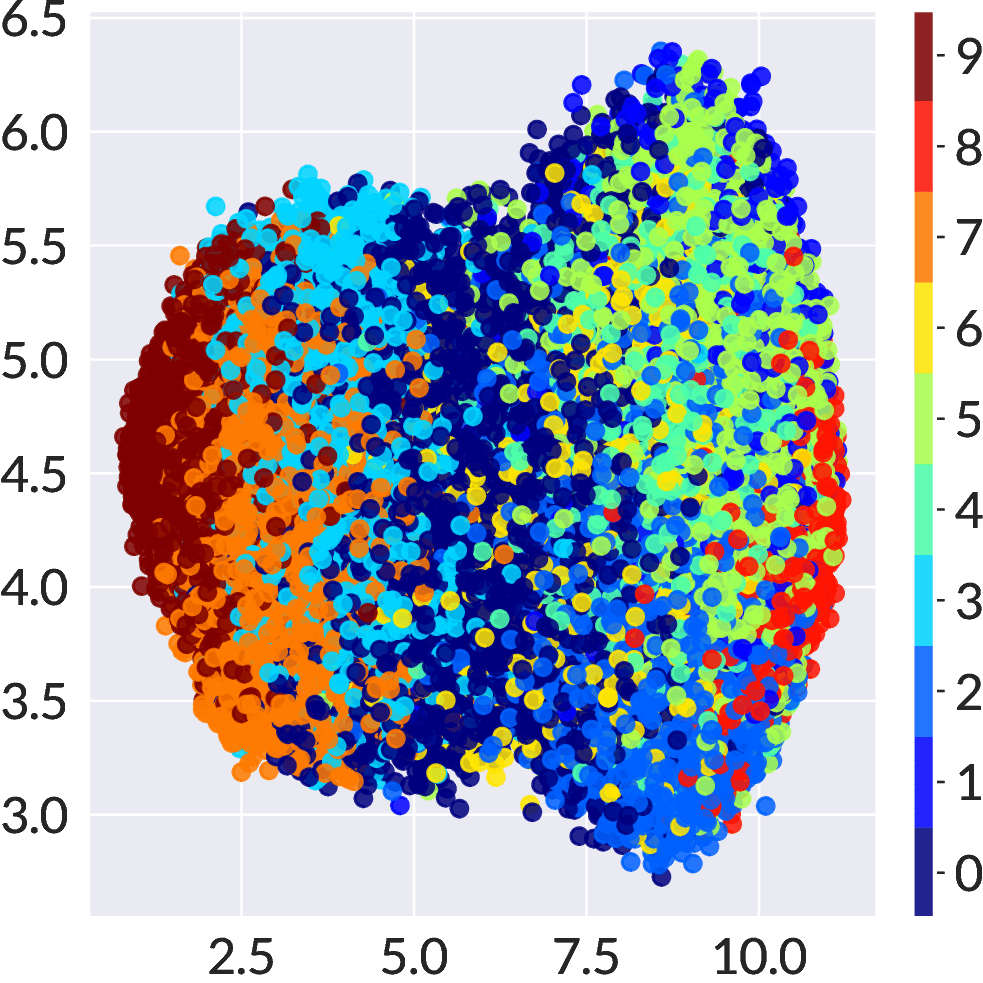}}
\subfigure[HDBSCAN]
{\includegraphics[keepaspectratio, width=0.32\columnwidth]{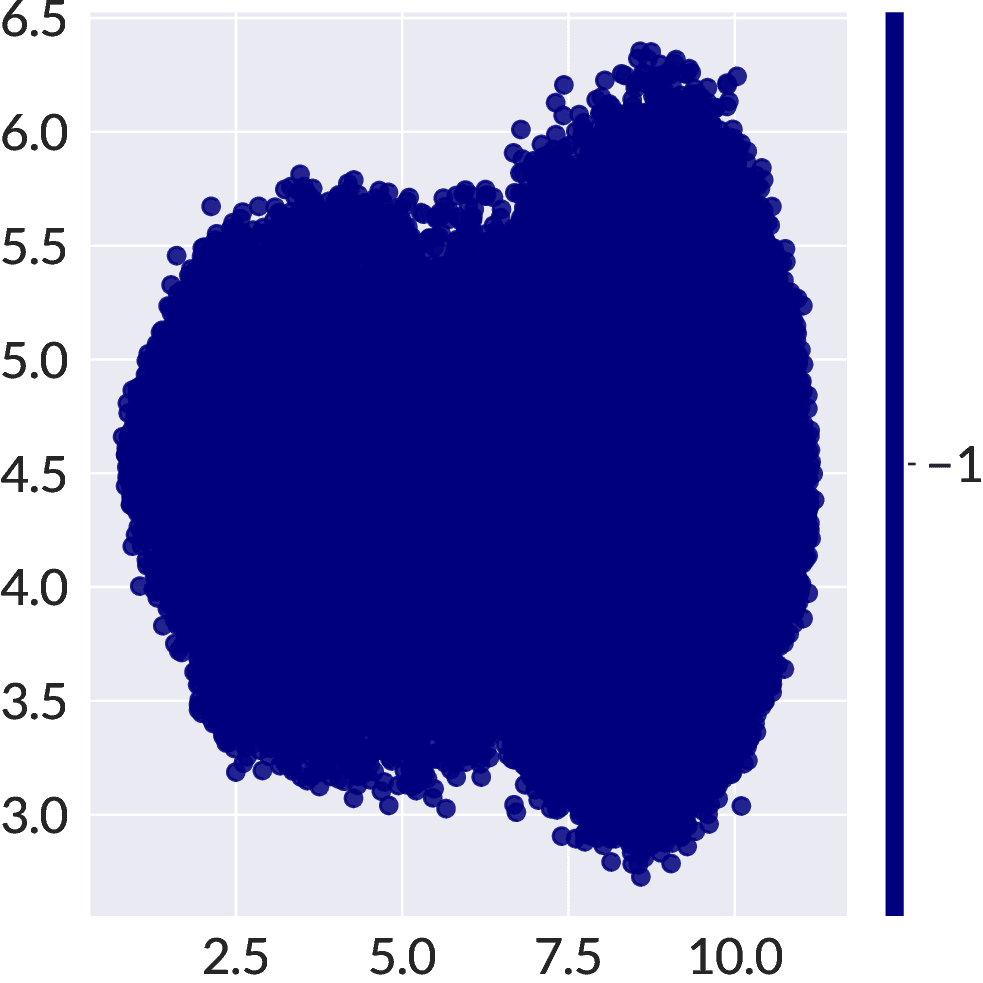}}
\caption{MNIST2 UMAP 2D Embedding representation of the clustering of the output of the U-Net.}
\label{fig:clustering_all_svhn}
\end{figure}

\begin{figure}[!h]
\centering 
\subfigure[True Class]
{\includegraphics[keepaspectratio, width=0.32\columnwidth]{figure/cifar-10/1_true.png}}
\subfigure[Kmeans]
{\includegraphics[keepaspectratio, width=0.32\columnwidth]{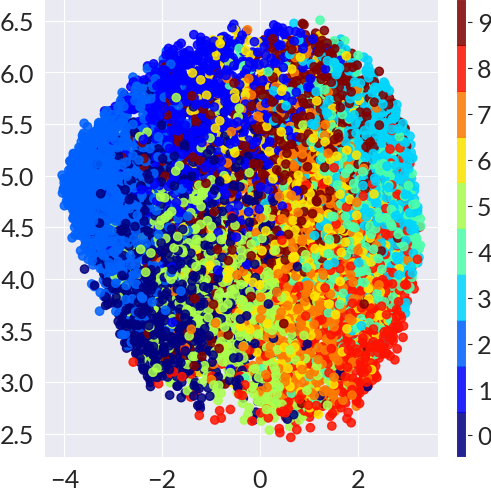}}
\subfigure[Spectral]
{\includegraphics[keepaspectratio, width=0.32\columnwidth]{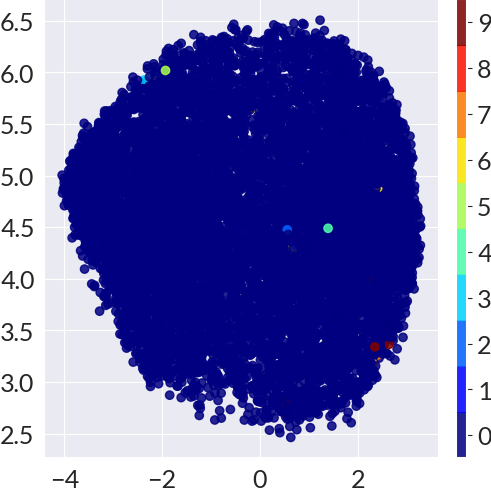}}
\subfigure[GMM]
{\includegraphics[keepaspectratio, width=0.32\columnwidth]{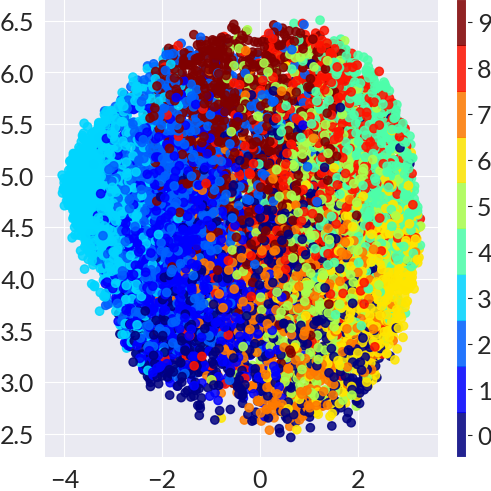}}
\subfigure[Birch]
{\includegraphics[keepaspectratio, width=0.32\columnwidth]{figure/cifar-10/5_birch.png}}
\subfigure[HDBSCAN]
{\includegraphics[keepaspectratio, width=0.32\columnwidth]{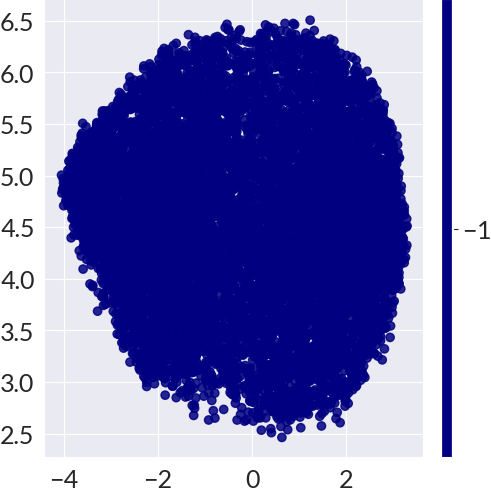}}
\caption{CIFAR-10 UMAP 2D Embedding representation of the clustering of the output of the U-Net.}
\label{fig:clustering_all_cifar10}
\end{figure}

\begin{figure}[!h]
\centering 
\subfigure[True Class]
{\includegraphics[keepaspectratio, width=0.32\columnwidth]{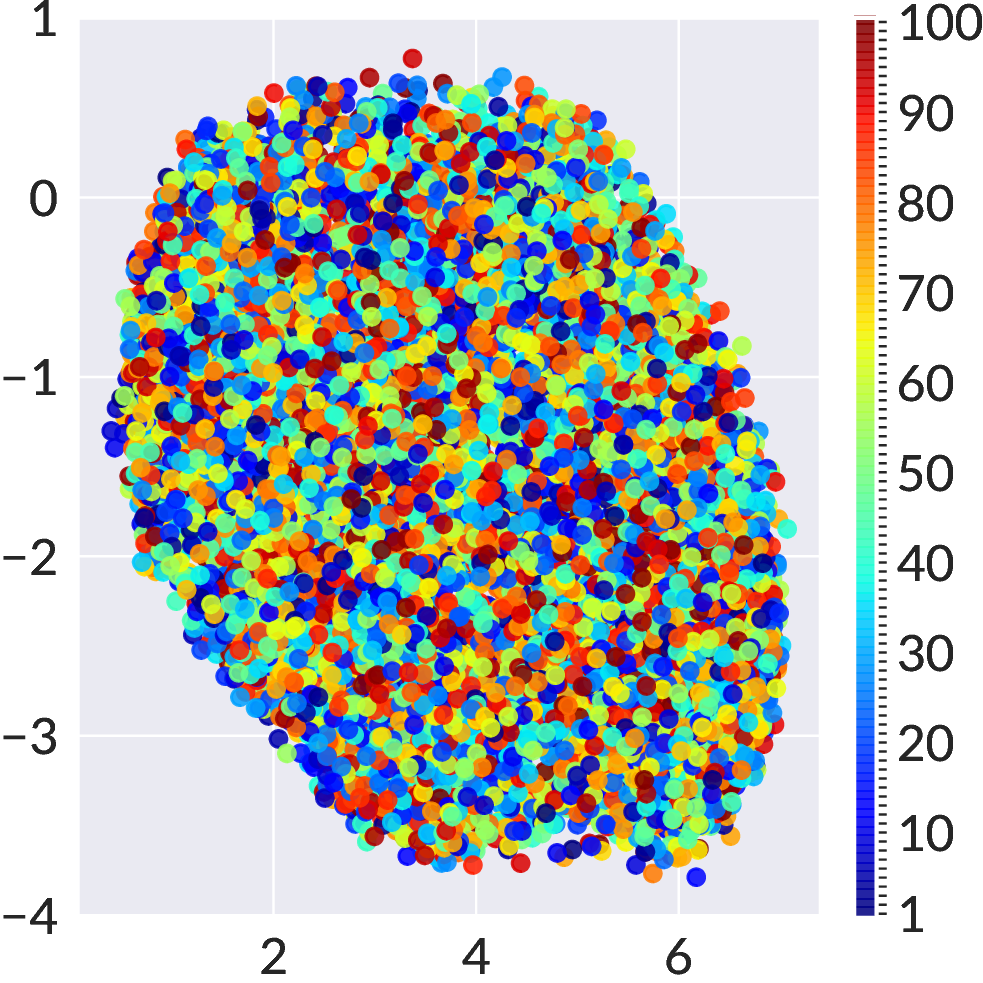}}
\subfigure[Kmeans]
{\includegraphics[keepaspectratio, width=0.32\columnwidth]{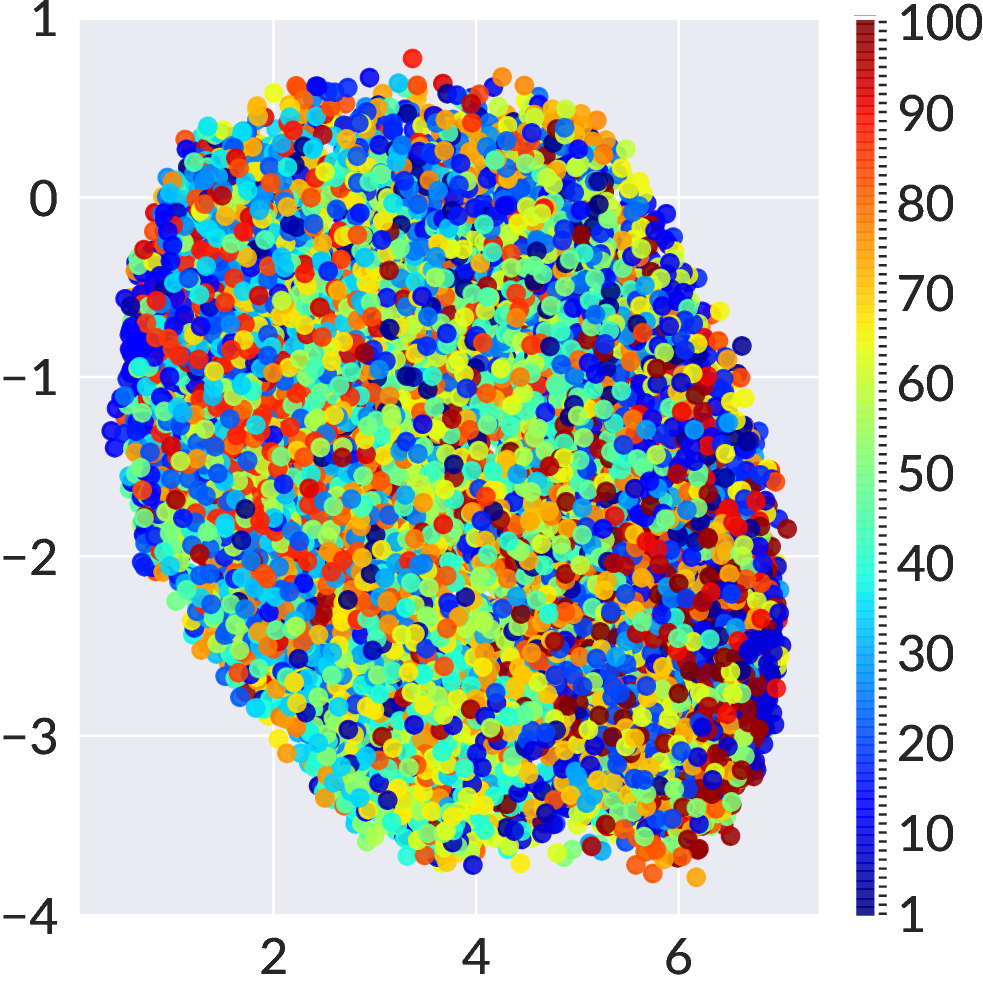}}
\subfigure[Spectral]
{\includegraphics[keepaspectratio, width=0.32\columnwidth]{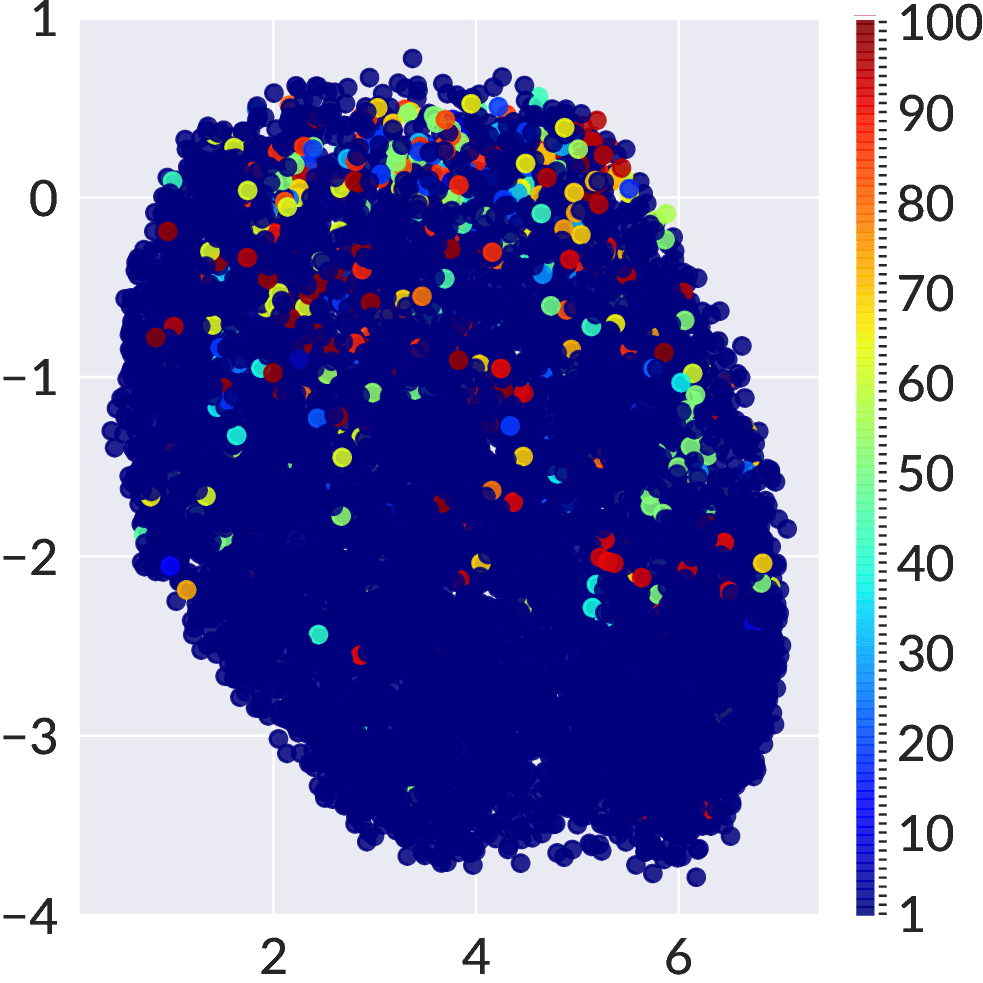}}
\subfigure[GMM]
{\includegraphics[keepaspectratio, width=0.32\columnwidth]{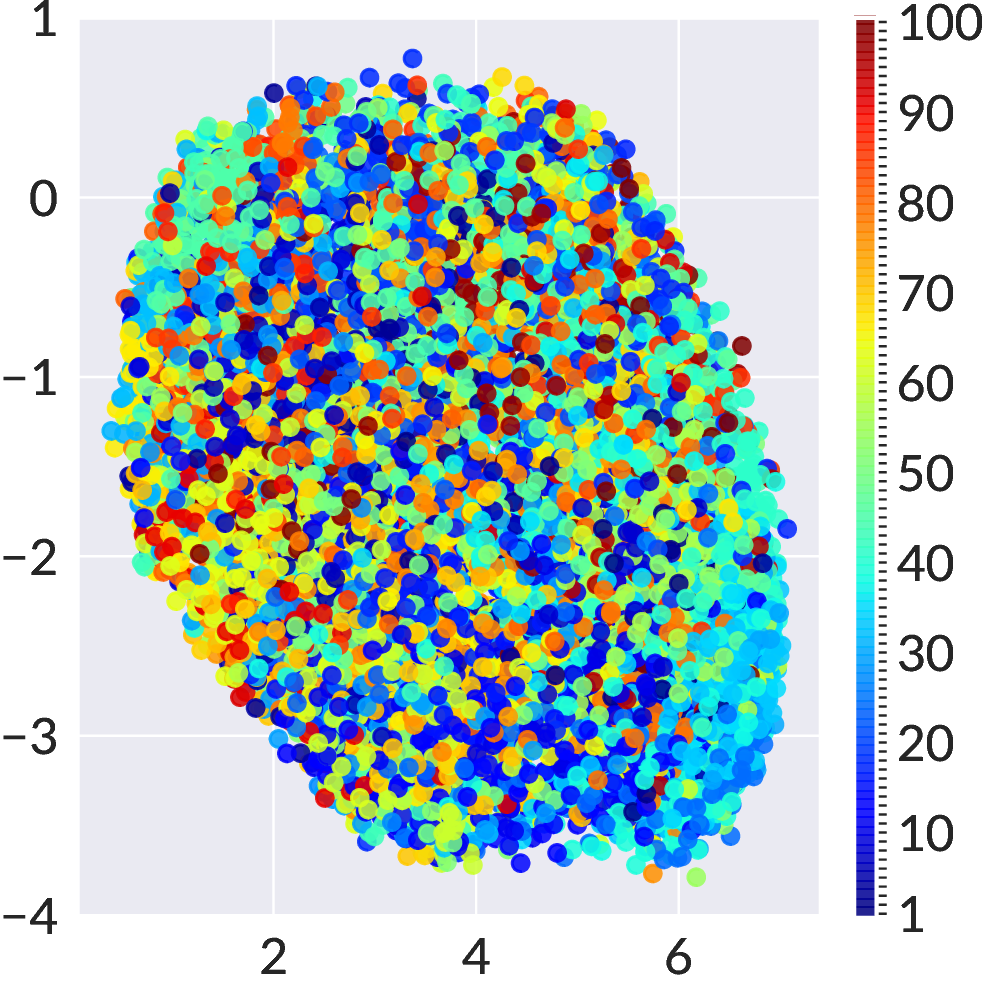}}
\subfigure[Birch]
{\includegraphics[keepaspectratio, width=0.32\columnwidth]{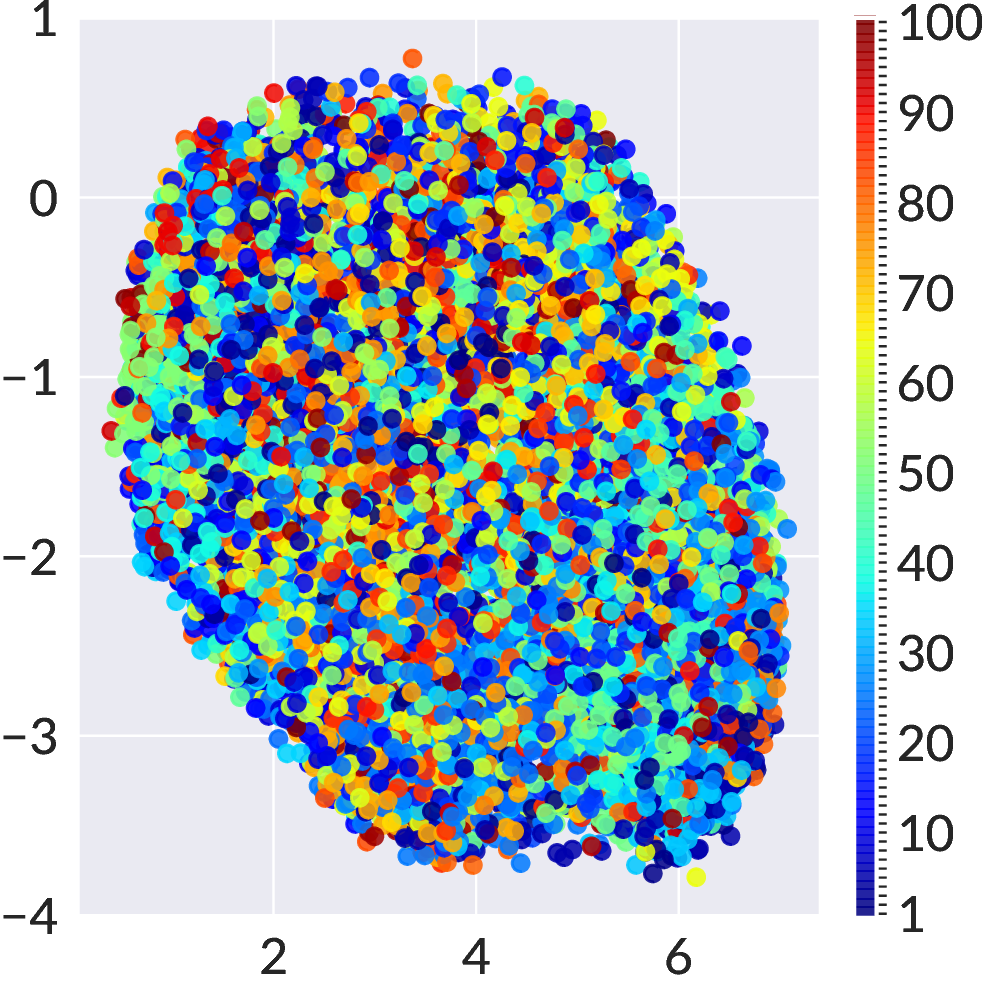}}
\subfigure[HDBSCAN]
{\includegraphics[keepaspectratio, width=0.32\columnwidth]{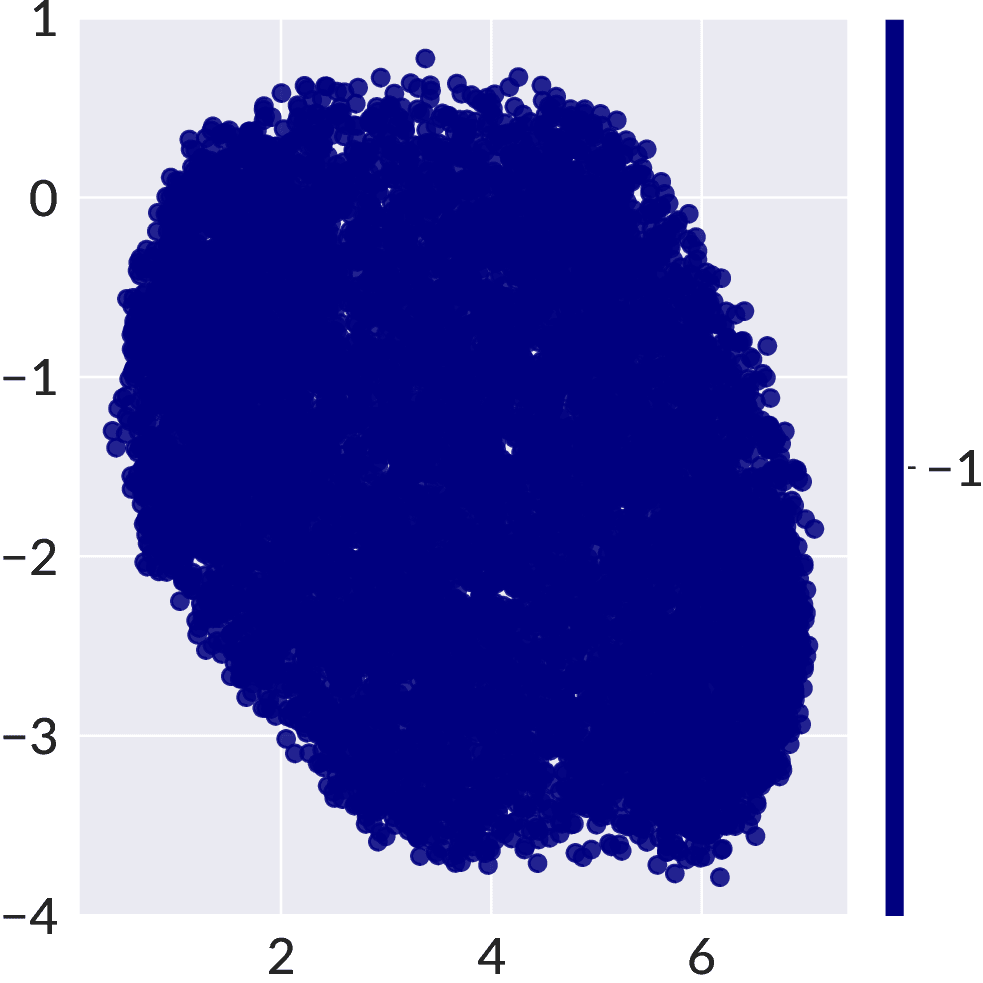}}
\caption{CIFAR-100 UMAP 2D Embedding representation of the clustering of the output of the U-Net.}
\label{fig:clustering_all_cifar100}
\end{figure}

\begin{figure}[!h]
\centering 
\subfigure[True Class]
{\includegraphics[keepaspectratio, width=0.32\columnwidth]{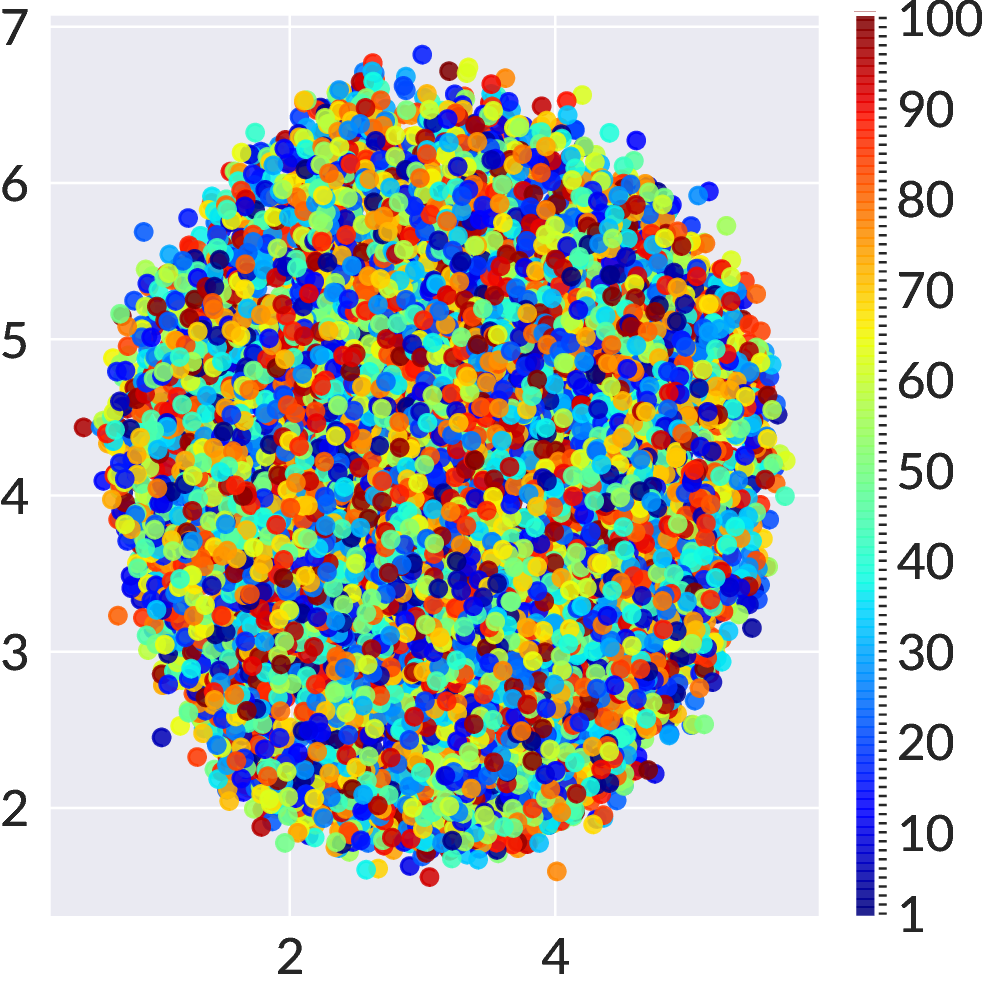}}
\subfigure[Kmeans]
{\includegraphics[keepaspectratio, width=0.32\columnwidth]{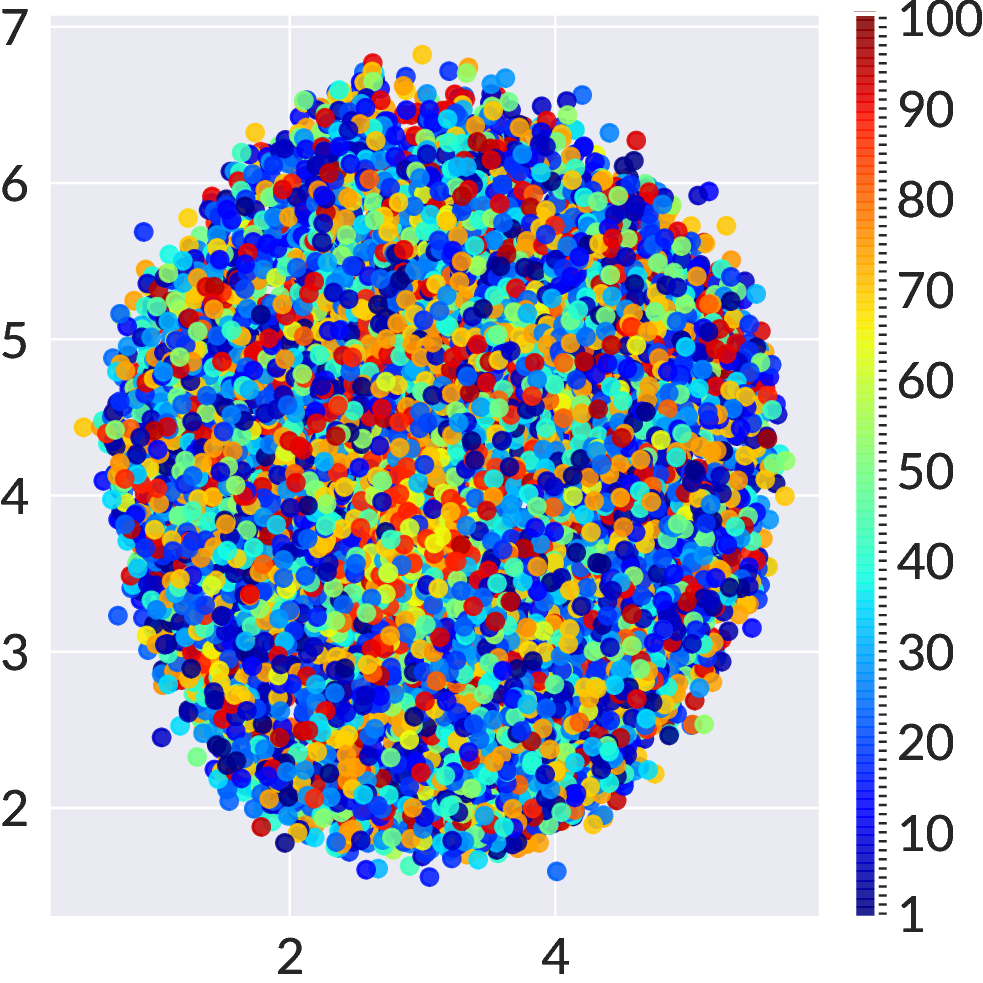}}
\subfigure[Spectral]
{\includegraphics[keepaspectratio, width=0.32\columnwidth]{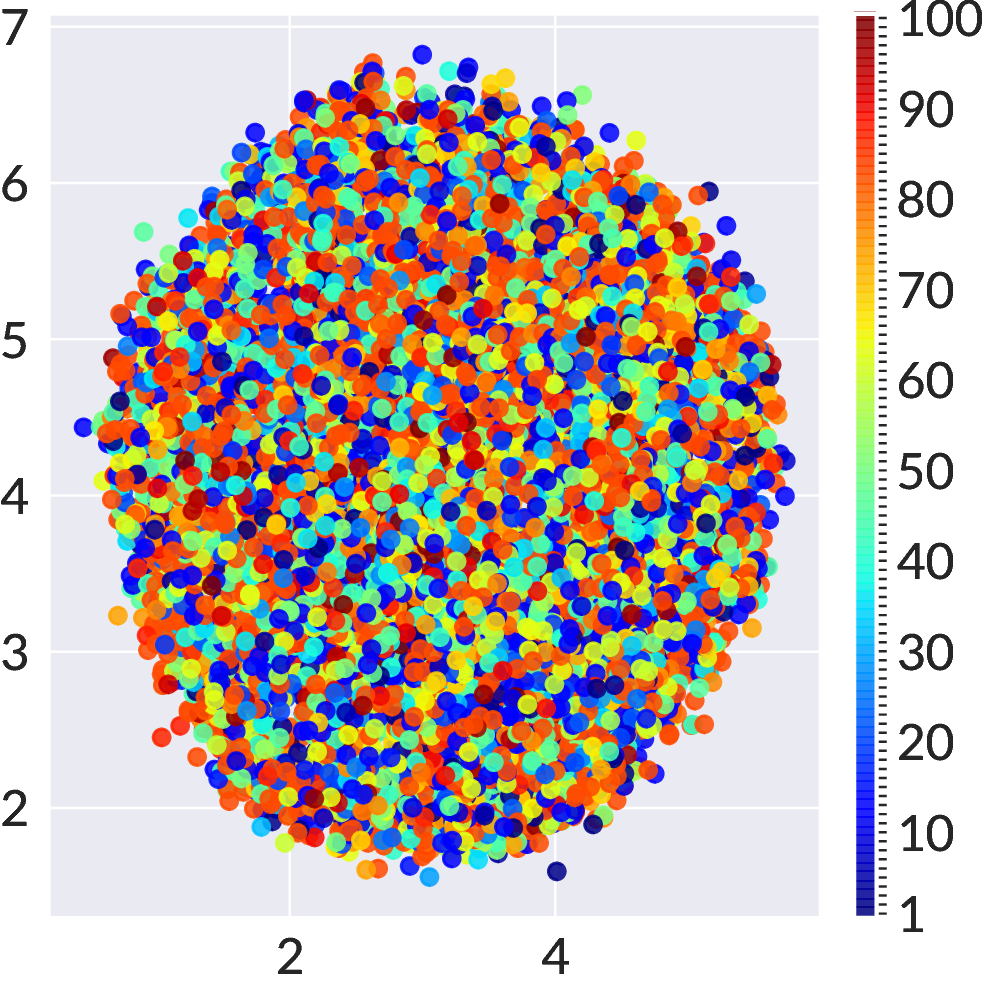}}
\subfigure[GMM]
{\includegraphics[keepaspectratio, width=0.32\columnwidth]{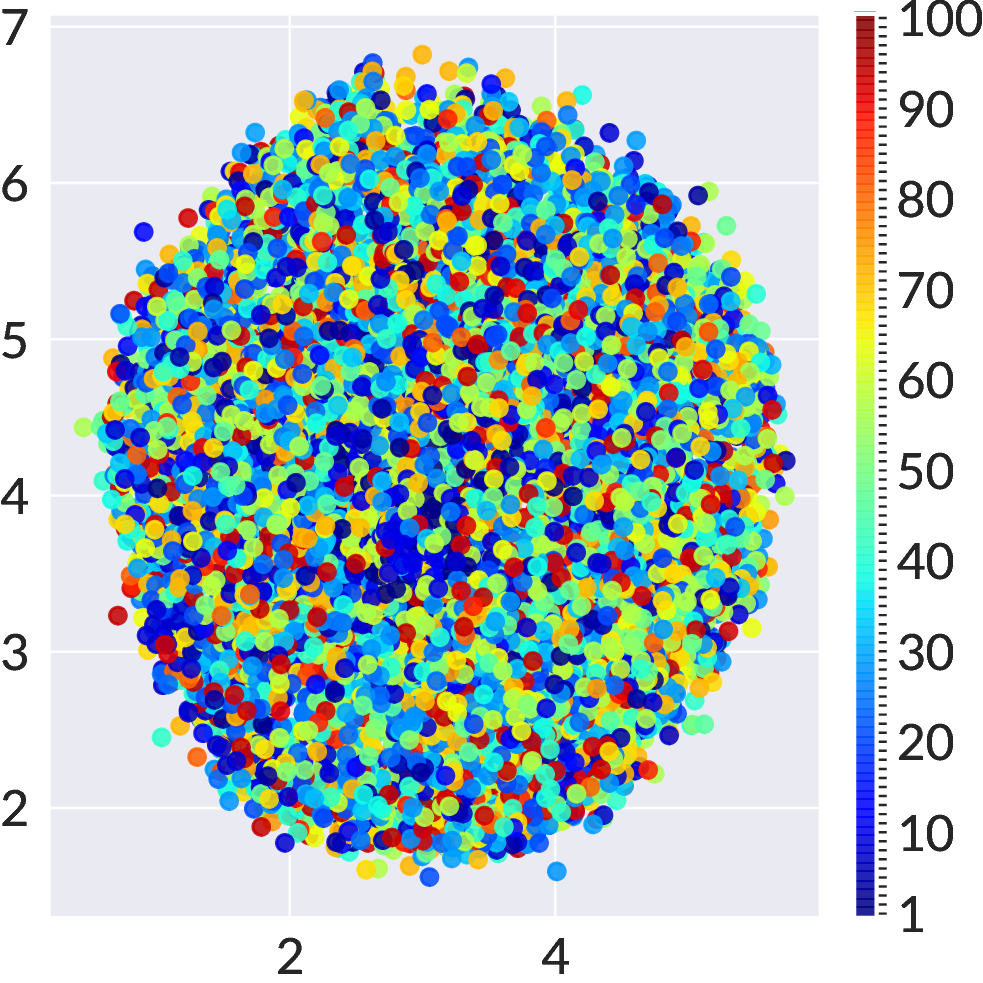}}
\subfigure[Birch]
{\includegraphics[keepaspectratio, width=0.32\columnwidth]{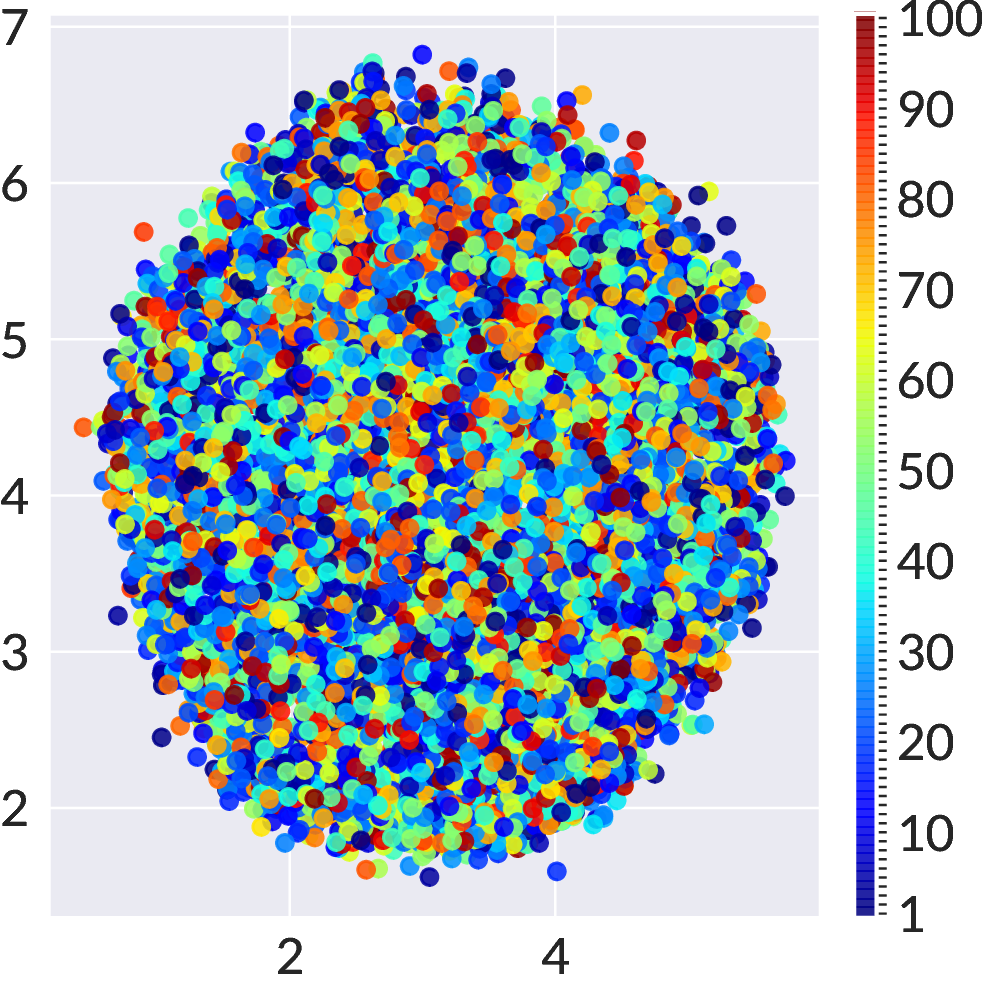}}
\subfigure[HDBSCAN]
{\includegraphics[keepaspectratio, width=0.32\columnwidth]{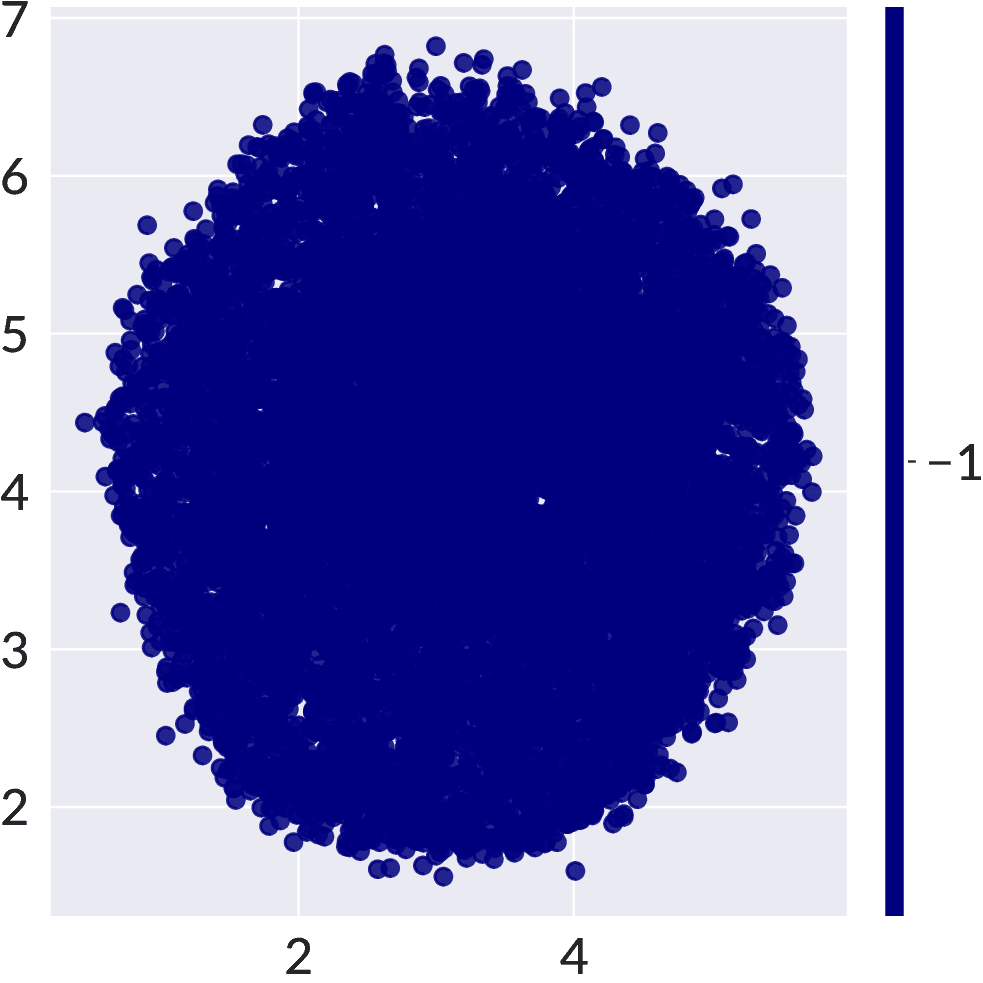}}
\caption{Mini-ImageNet UMAP 2D Embedding representation of the clustering of the output of the U-Net.}
\label{fig:clustering_all_miniimagenet}
\end{figure}

\begin{figure}[!t]
\centering 
\subfigure[MNIST]
{\includegraphics[keepaspectratio, width=0.32\columnwidth]{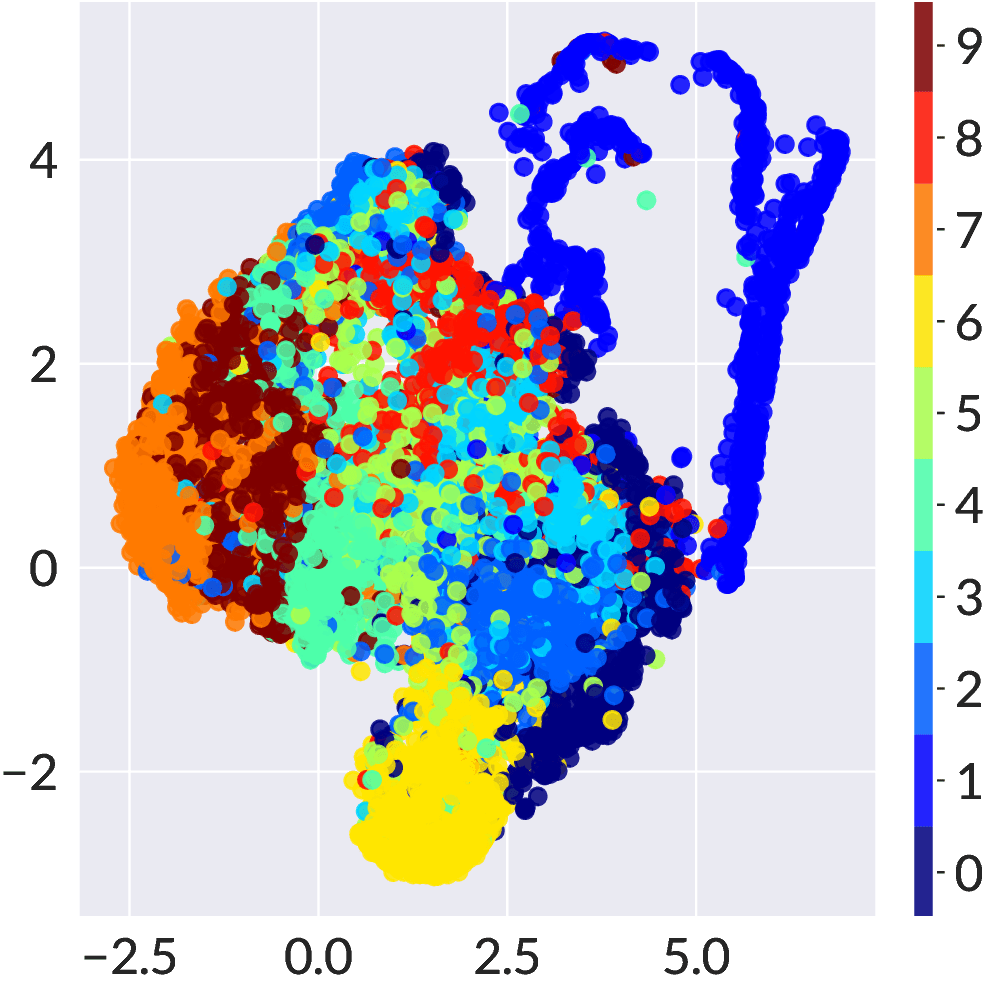}}
\subfigure[SVHN]
{\includegraphics[keepaspectratio, width=0.32\columnwidth]{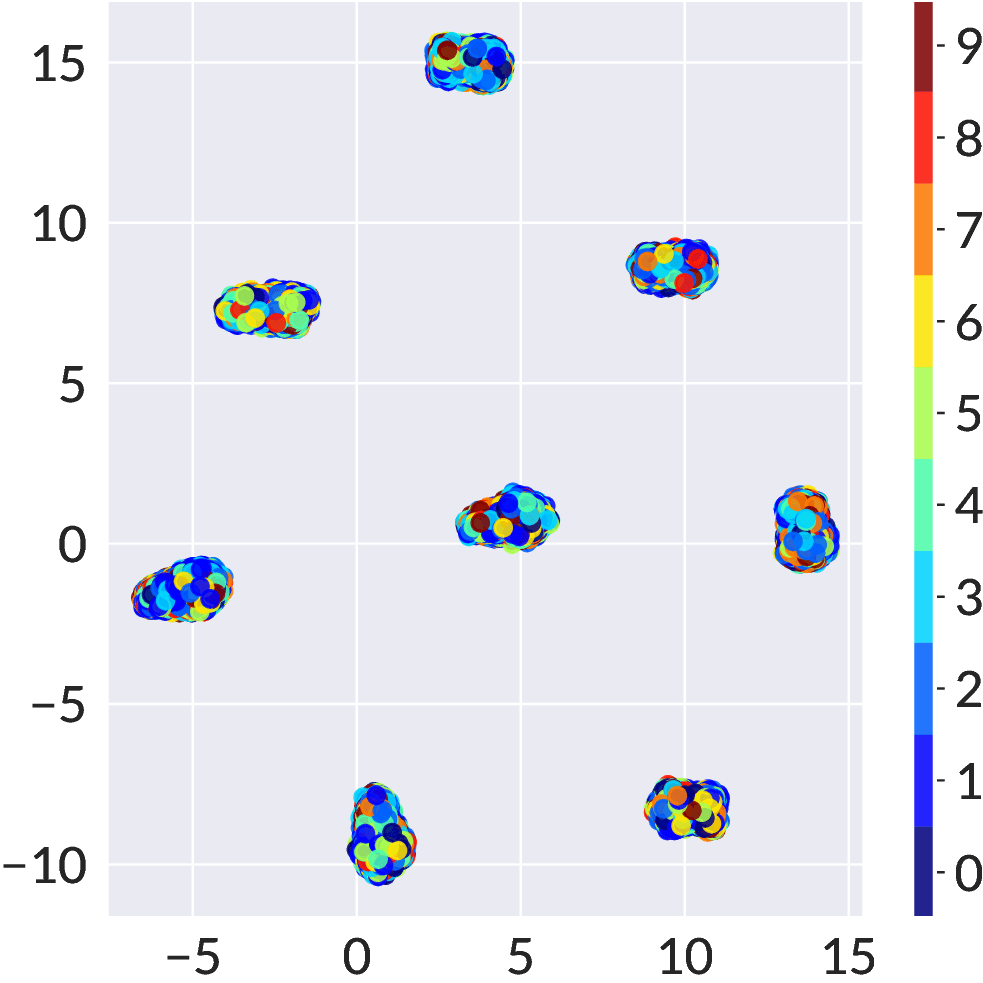}}
\subfigure[CIFAR-10]
{\includegraphics[keepaspectratio, width=0.32\columnwidth]{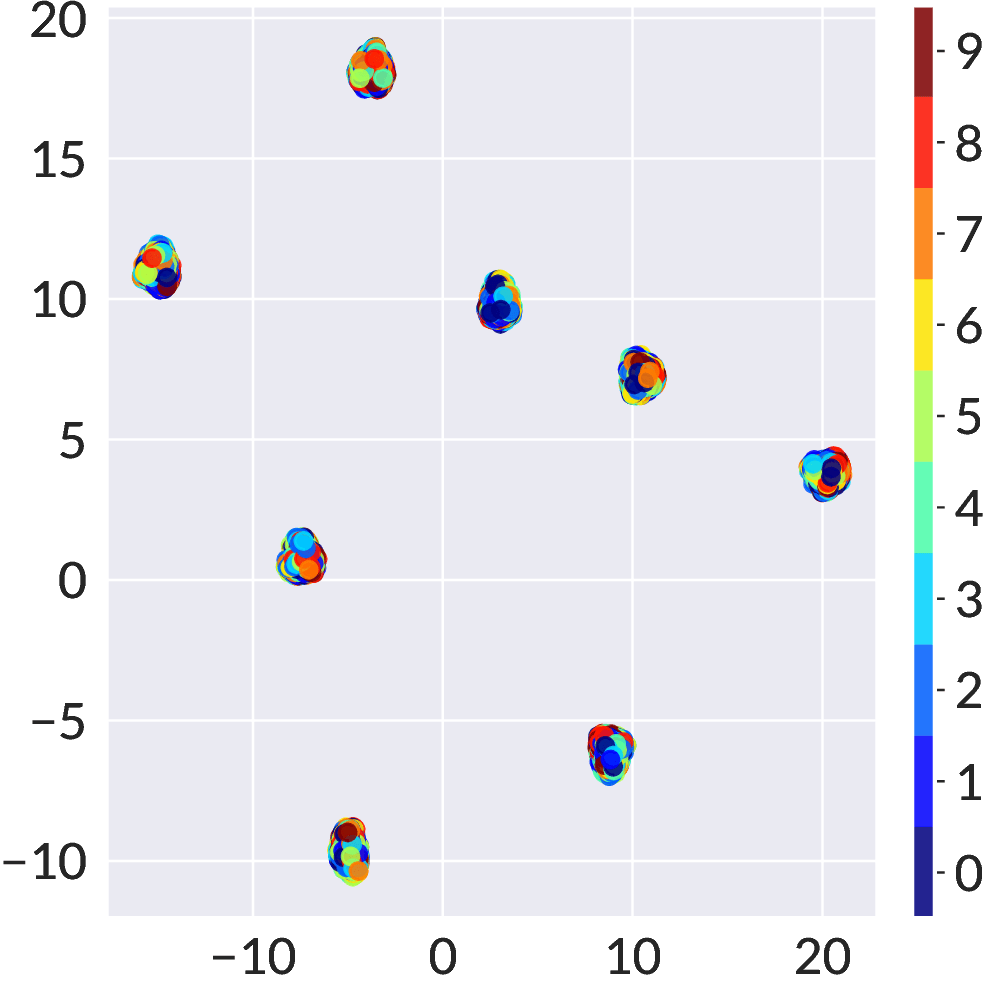}}
\subfigure[CIFAR-100]
{\includegraphics[keepaspectratio, width=0.32\columnwidth]{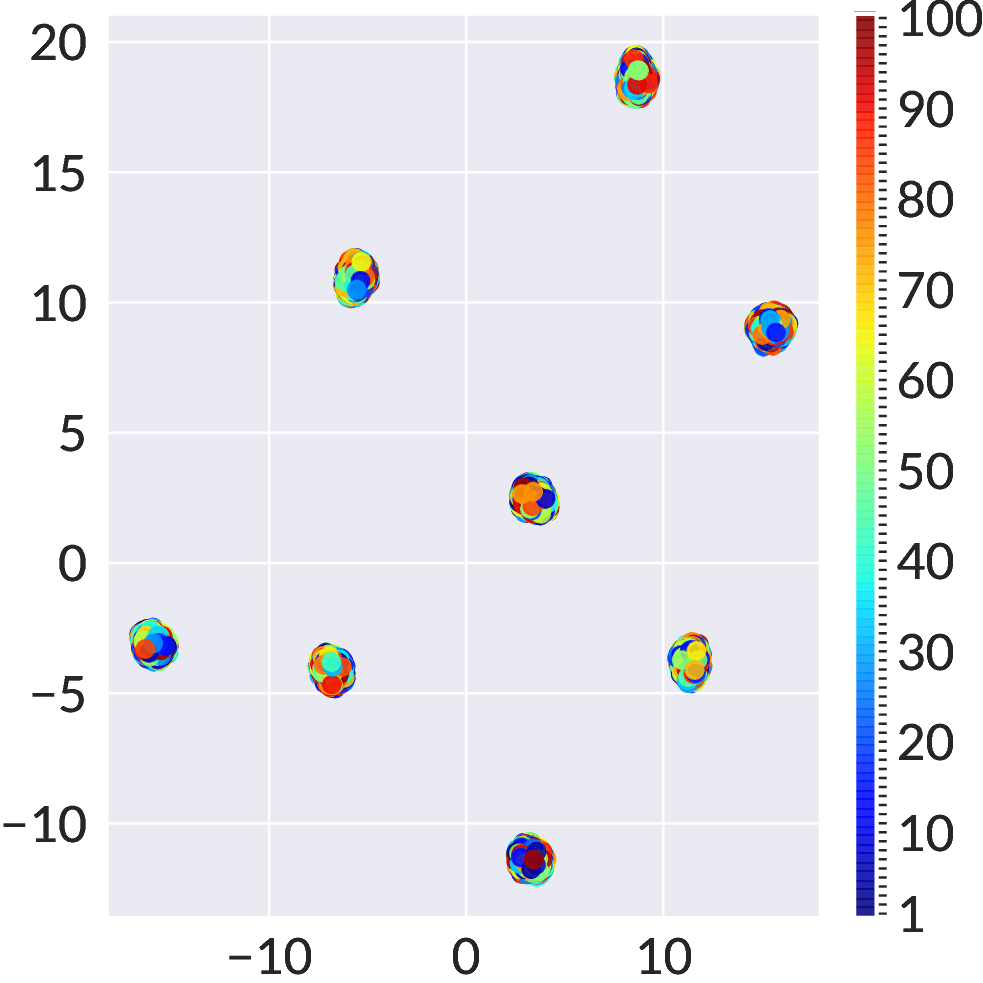}}
\subfigure[Mini-ImageNet]
{\includegraphics[keepaspectratio, width=0.32\columnwidth]{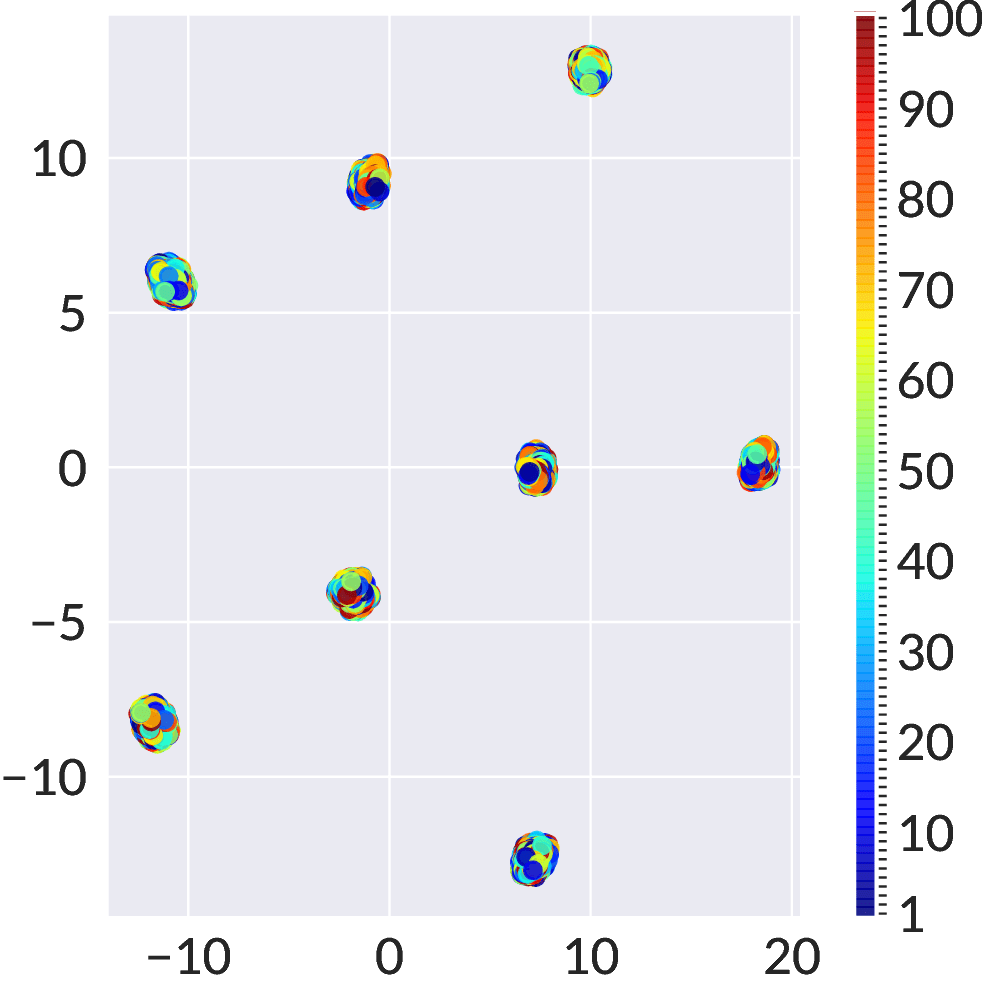}}
\caption{UMAP embeddings of output $r \bind s^{\dagger}$, showing what the space looks like once the secret is applied to extract the output. We can see on MNIST that the classes are clustered well but grouped together. This appears to be a large part of what happens to the sub-cluster that occurs on other datasets, where the secret is consistently placing the populations into one of eight groups. The consistency of the eight populations is not yet known but does show how the binding/unbinding dramatically changes the characteristics of the space.}
\label{clustering-unbound}
\end{figure}

\FloatBarrier
\subsection{Ablation Study} \label{sec:ablation}
We have shown our \shortNameTxt is 290-5000$\times$ faster than alternative options, empirically robust to subversion against unreasonable powerful adversaries, and maintains high predictive accuracy. We now perform several ablations that demonstrate all the components to our design are needed to obtain our strong results. We will perform ablation tests using ResNet-50~\citep{b7} and ReZero~\citep{b5} architectures instead of U-Net, where the secret $\boldsymbol{s}$ is projected down to match the smaller output dimension, we will show the importance of maintaining input/output size of $f_W(\cdot)$. Alternative VSA options including standard HRR, vector-derived transformation binding (VTB)~\citep{b6}, and our own improved VTB (iVTB) demonstrate the importance of our 2D HRR approach to maintaining the symbolic properties in a manner that can be extracted after the CNN. We also test alternative encoding of the 2D structure by using a space-filling Hilbert curve with 1D HRR as an approximation of the total 2D structure. We briefly summarize how each ablation type is performed. Visual examples of how HRR, VTB, iVTB, and Hilbert impact the encoding/decoding are in \autoref{sec:ablation_visual_results}.

\textbf{ReZero/ResNet:} A ReZero/ResNet architecture with 50 layers is used as $f_W(\cdot)$, resulting in an output size smaller than the input. A two layer fully-connected network is used to project $\boldsymbol{s}$ down to the output shape of the network.
\par 
\textbf{HRR:} The standard HRR with a 1D FFT is used and initialized using \citep{hrrxml}.
\par 
\textbf{VTB:} The VTB is used instead of HRR, as described by \citep{b6}. VTB replaces the FFT operation with a tiled sparse block-diagonal matrix multiplication that has similar properties as HRR. 
\par 
\textbf{iVTB:} VTB, but improved by choosing the secrets to be block orthogonal to force properties of the VTB to be exactly true, rather than true in expectation. First, a tensor is randomly sampled from a normal distribution and then QR decomposition is applied in the tensor. The orthogonal part of the QR decomposition is used as the secret $\boldsymbol{s}$, which improves binding retrieval. 
\par 
\textbf{Hilbert Space-Filling Curve:} To keep the spatial locality of the image and restore the structural conformity Hilbert Space-Filling Curve is employed. Images are encoded and decoded before and after binding and unbinding.

\begin{table}[!htbp]
\centering
\caption{Ablation Study on CIFAR-10. }
\label{tab:ablation}
\renewcommand{\arraystretch}{1.2}
\adjustbox{max width=\columnwidth}{%
\begin{tabular}{@{}cc@{}}
\toprule
Method & Accuracy \\ \midrule
U-net + HRR 2D  (\shortNameTxt)& 78.21 \\  \cmidrule(lr){1-2}
ResNet-$50$ + HRR           & 53.52 \\  
ReZero-$50$ + HRR                & 57.80 \\ 
ReZero-$50$ + VTB                & 52.21 \\  
ReZero-$50$ + iVTB               & 54.19 \\  
ReZero-$50$ + HRR + Hilbert      & 53.81 \\ 
ResNet-$50$ + HRR + Hilbert & 50.87 \\  
ReZero-$50$ + HRR 2D             & 49.32 \\
\bottomrule
\end{tabular}%
}
\end{table}

Table~\ref{tab:ablation} shows the accuracy of the different combinations of the network architectures and VSAs. Among them, HRR with ReZero blocks in the main network has the best accuracy of $57.80\%$, significantly below \shortNameTxt. This shows the U-Net design, with 2D HRR for encoding, is critical in combination to obtain our results.

We further show for each of the proposed alternative strategies  visualizations of the binding/unbinding process to provide intuition as to how they work or why their results are less effective.

\begin{figure}[!htbp]
\centering 
\subfigure[Original Image]
{\includegraphics[scale=.224]{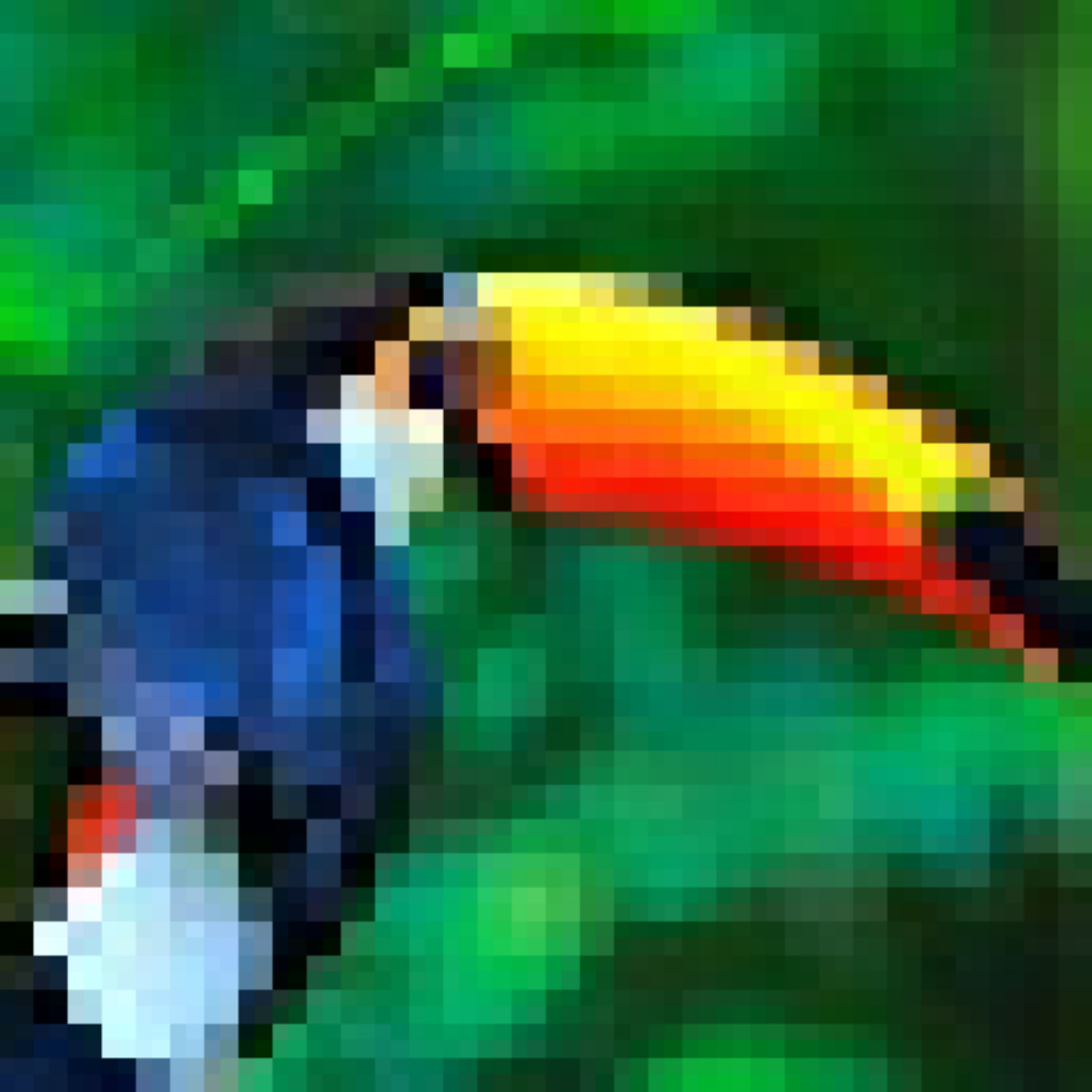}}
\subfigure[Bound Image]
{\includegraphics[scale=.224]{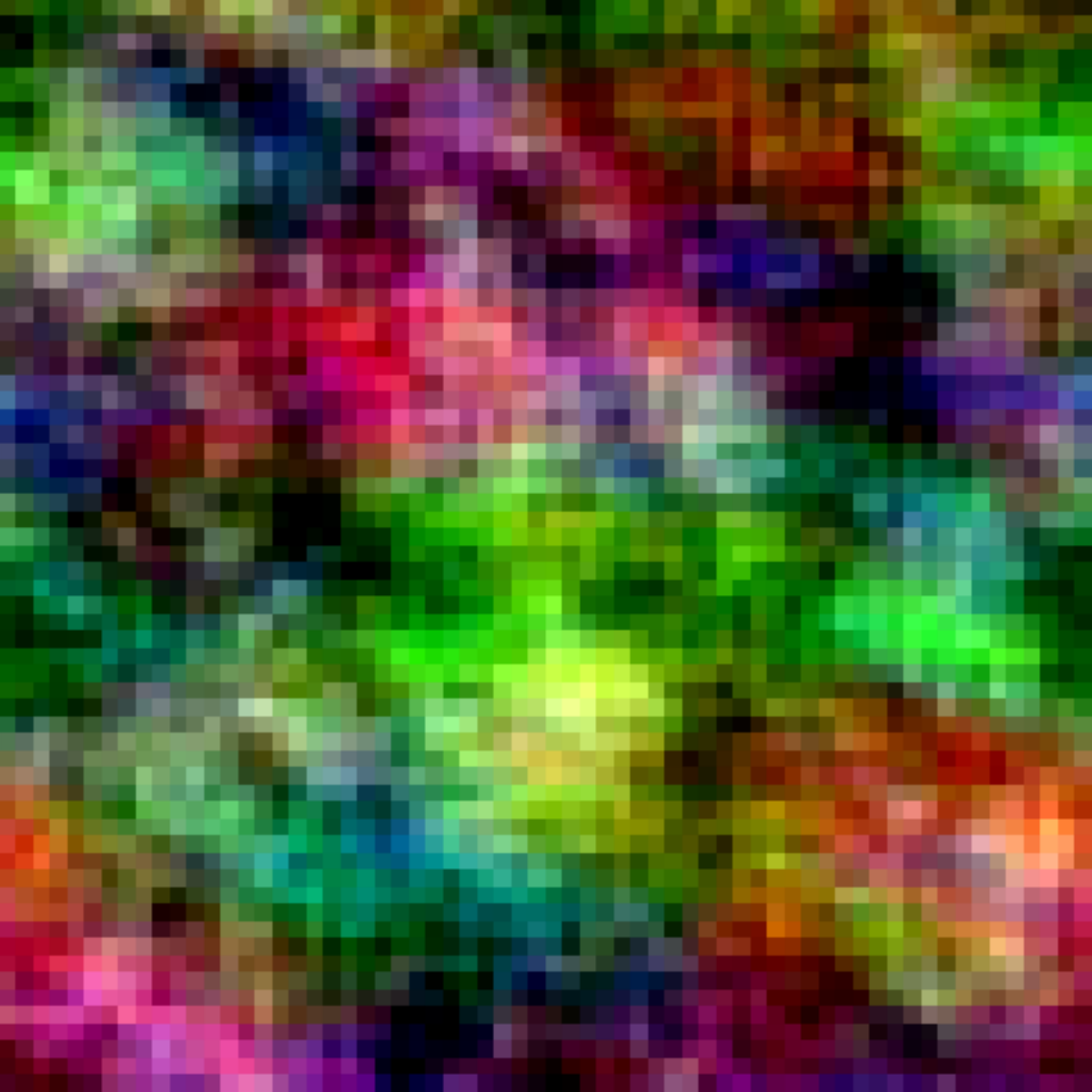}}
\subfigure[Retrieved Image]
{\includegraphics[scale=.224]{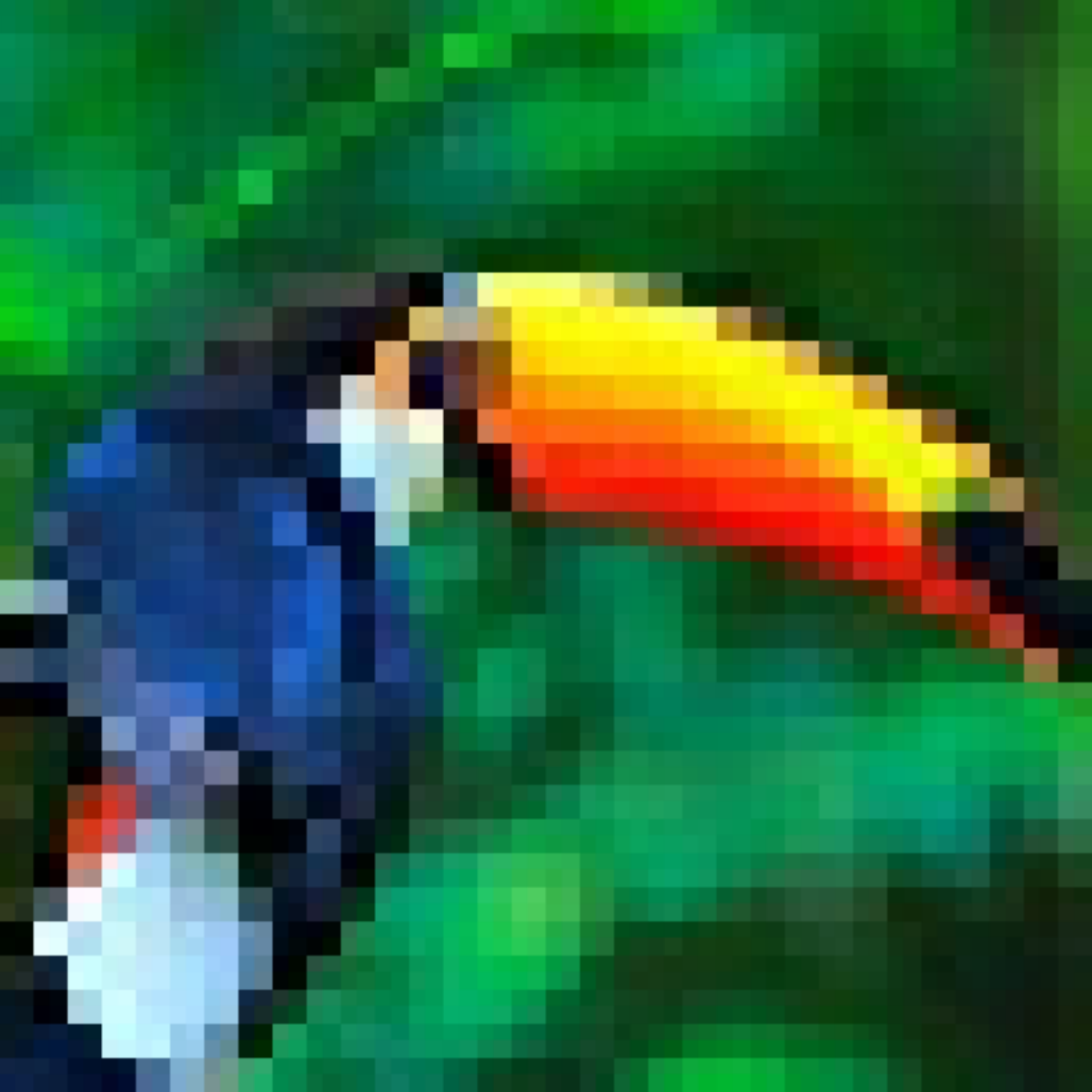}}
\caption{A standard HRR applied with the improved initialization of \citet{hrrxml}, which does not recognize the 2D structure of the data but is implicitly a 1D convolution over the linearization of the pixels. The bound image thus looks random in a different way, but the output is retrievable.}
\label{hrr}
\end{figure}

\begin{figure}[!htbp]
\centering 
\subfigure[Original Image]
{\includegraphics[scale=.224]{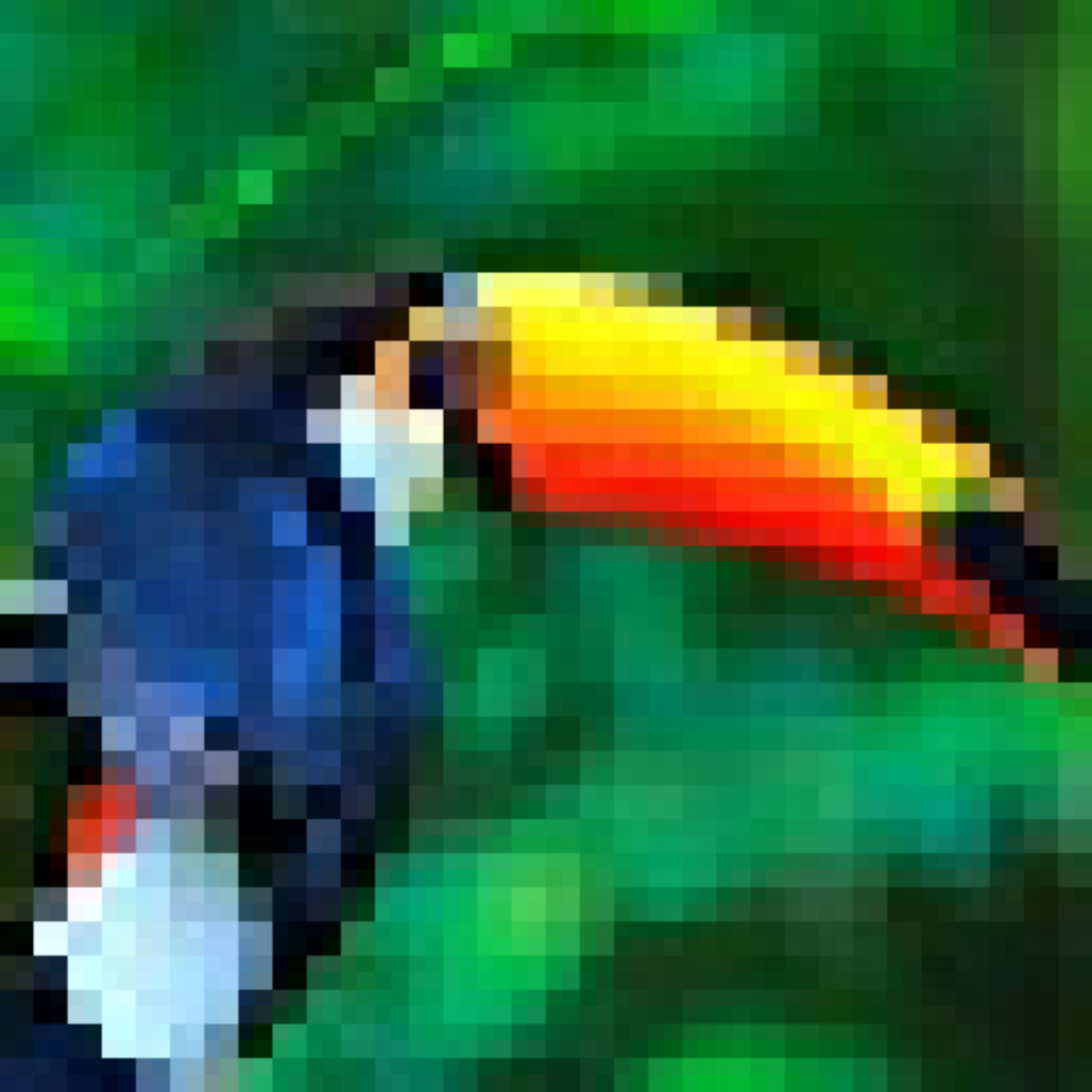}}
\subfigure[Bound Image]
{\includegraphics[scale=.224]{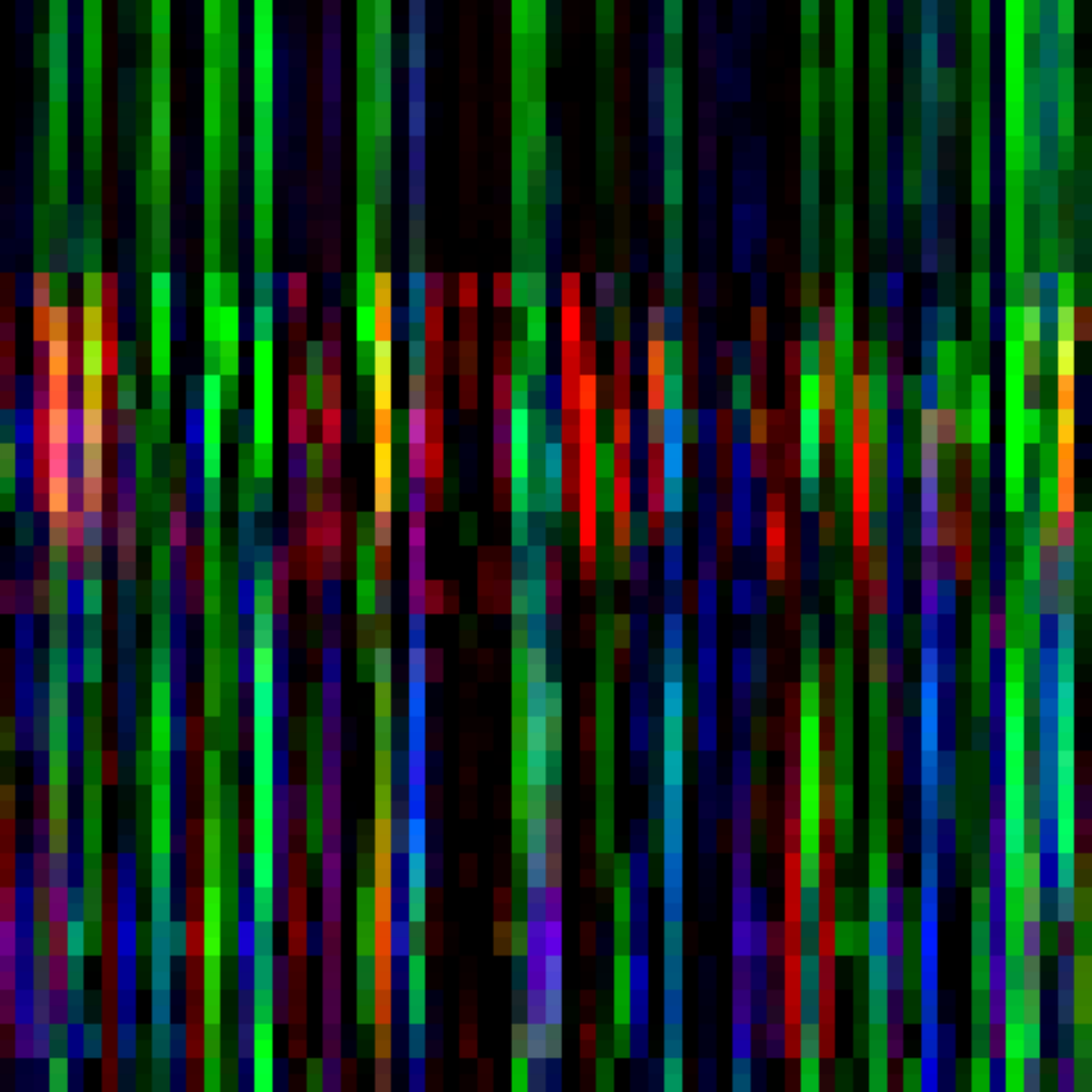}}
\subfigure[Retrieved Image]
{\includegraphics[scale=.224]{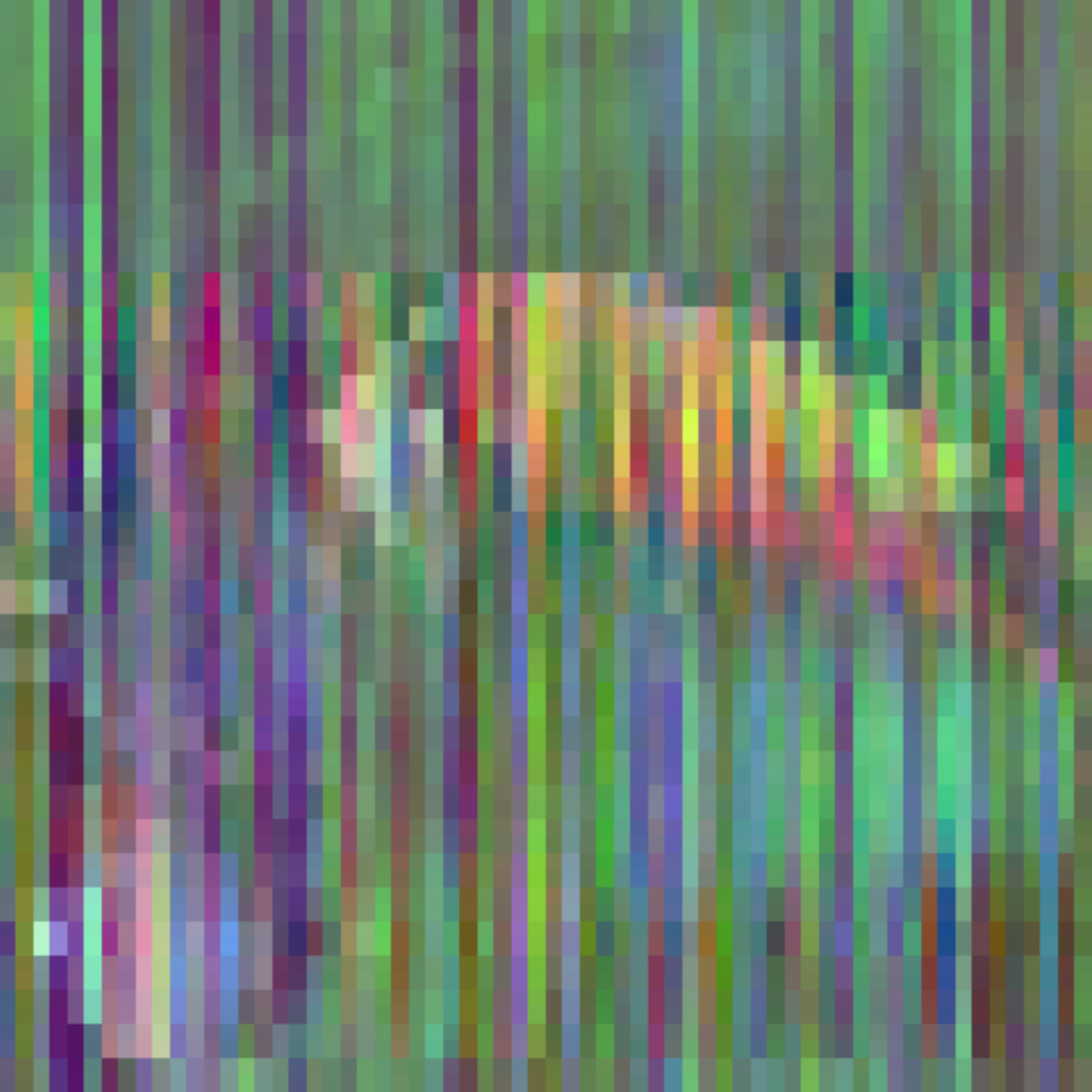}}
\caption{The VTB approach proposed by \citet{b6} is applied to the flattened/linearized version of the pixels. While the input is still obfuscated in a different visual pattern, the retrieved image is very noisy. Such noise is inevitable when multiple items are bound together, but given that we have only one item we would desire a higher quality retrieval.}
\label{vtb}
\end{figure}

\begin{figure}[!htbp]
\centering 
\subfigure[Original Image]
{\includegraphics[scale=.224]{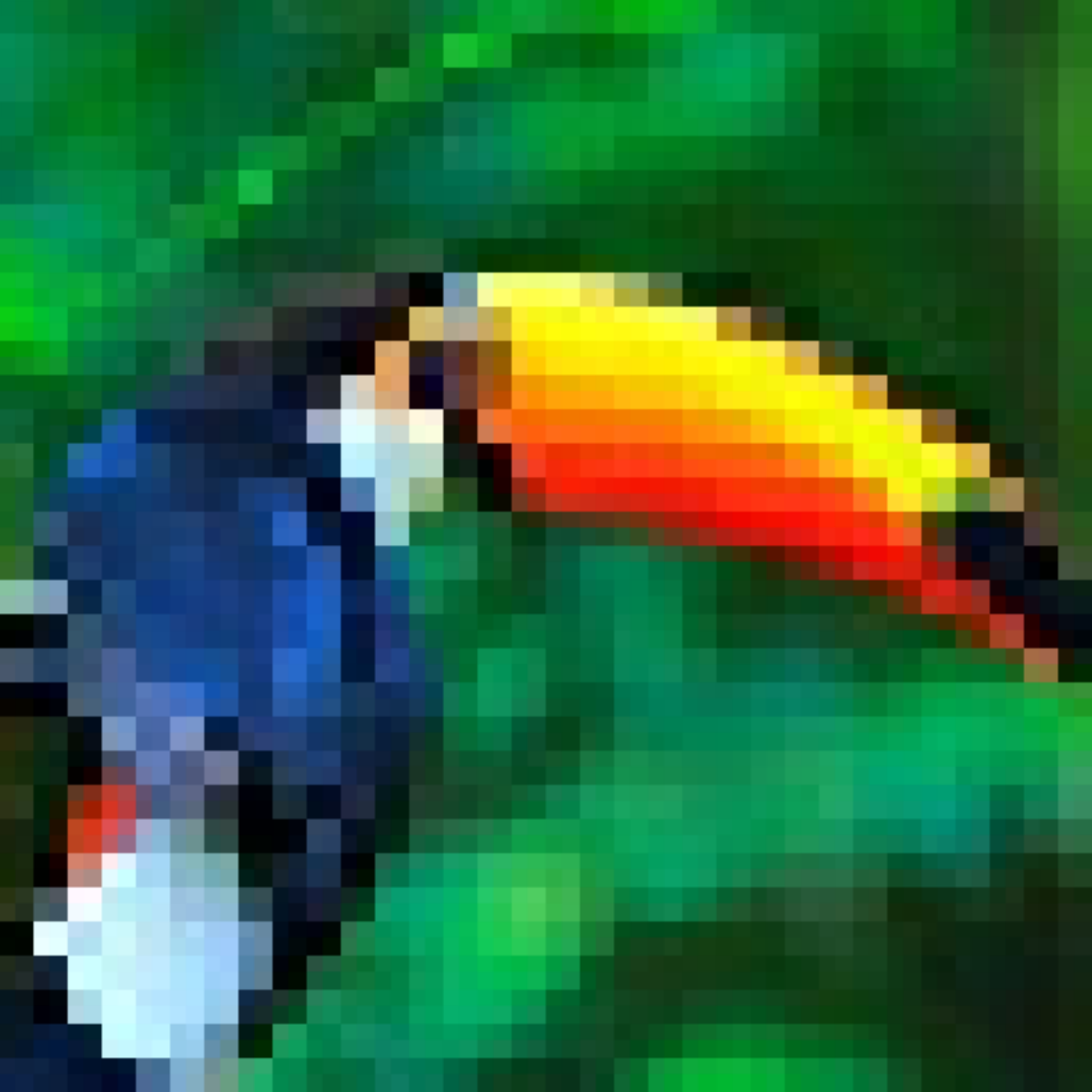}}
\subfigure[Bound Image]
{\includegraphics[scale=.224]{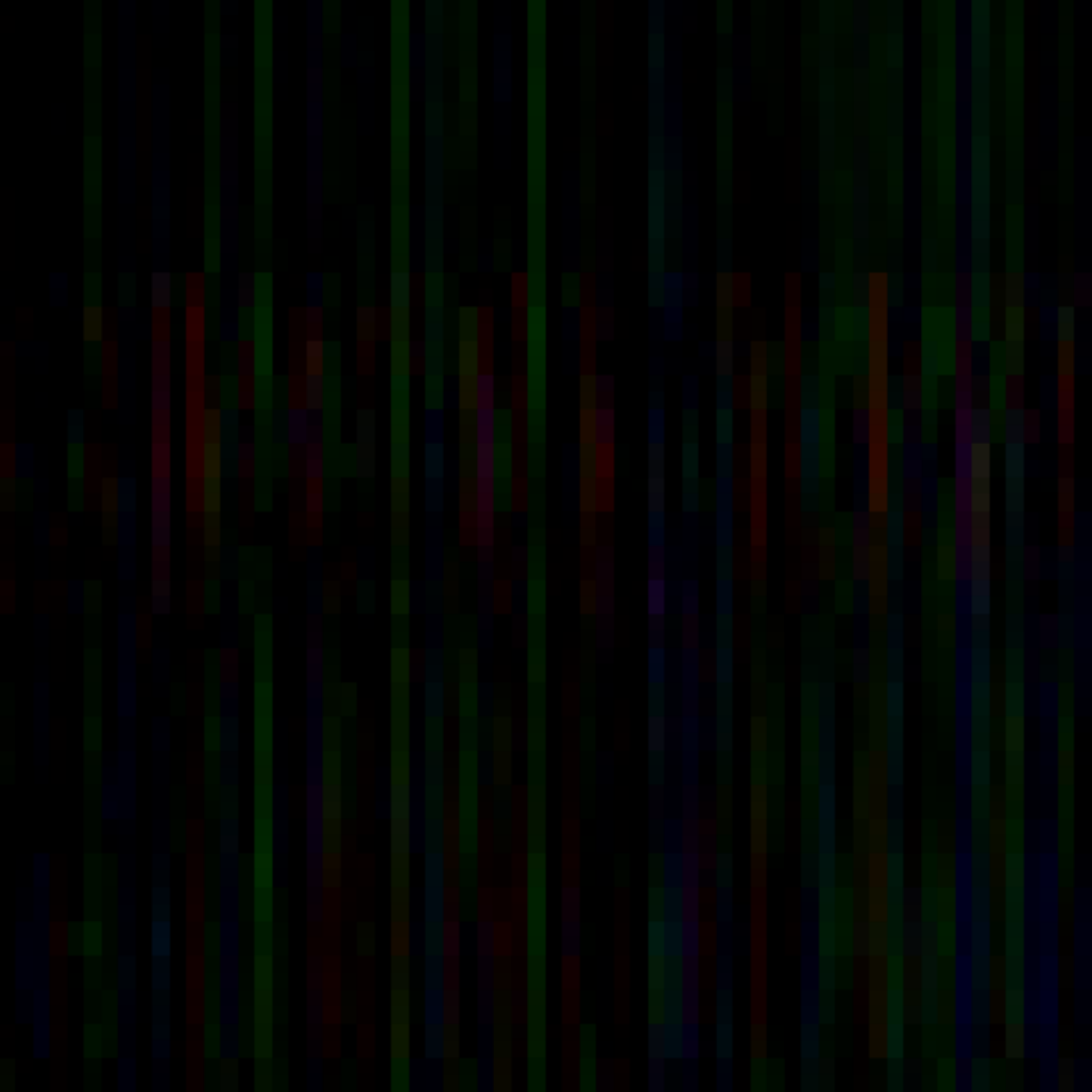}}
\subfigure[Retrieved Image]
{\includegraphics[scale=.224]{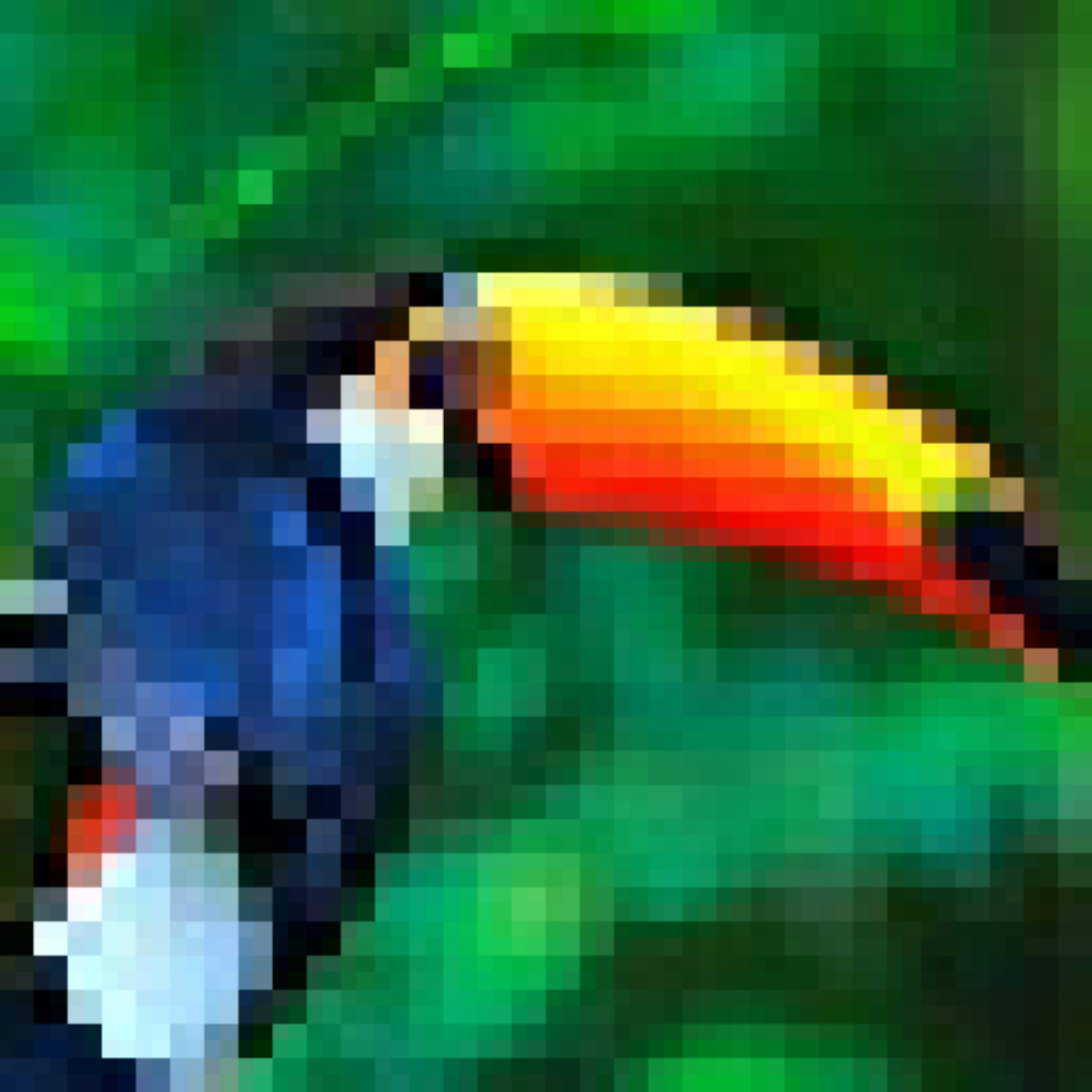}}
\caption{Our improved VTB forces the sub-structure of the vector to be orthogonal, allowing for exact retrieval of the input if there is only one bound item.}
\label{ivtb}
\end{figure}

\begin{figure}[!htbp]
\centering 
\subfigure[Original Image]
{\includegraphics[scale=.224]{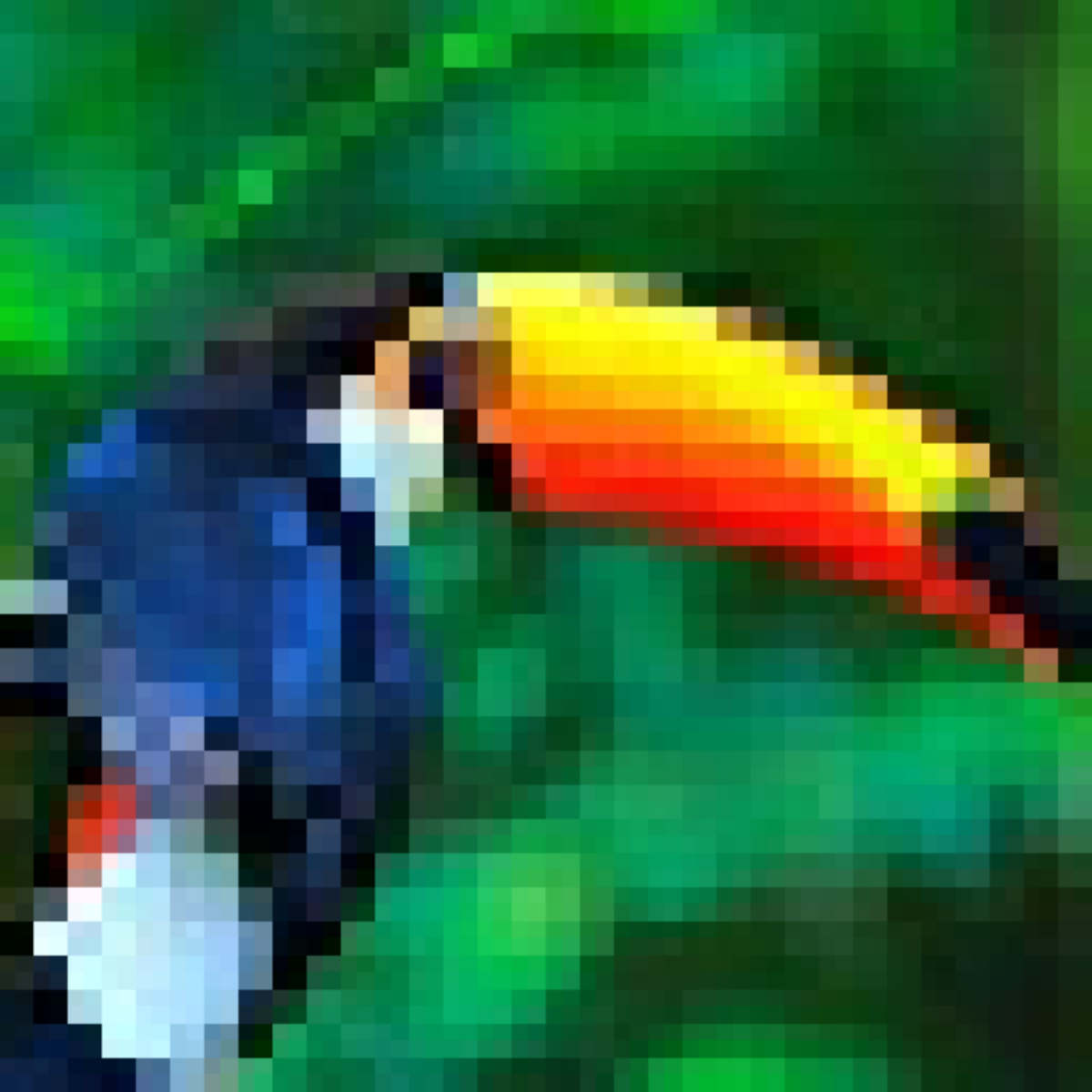}}
\subfigure[Hilbert Encode]
{\includegraphics[scale=.224]{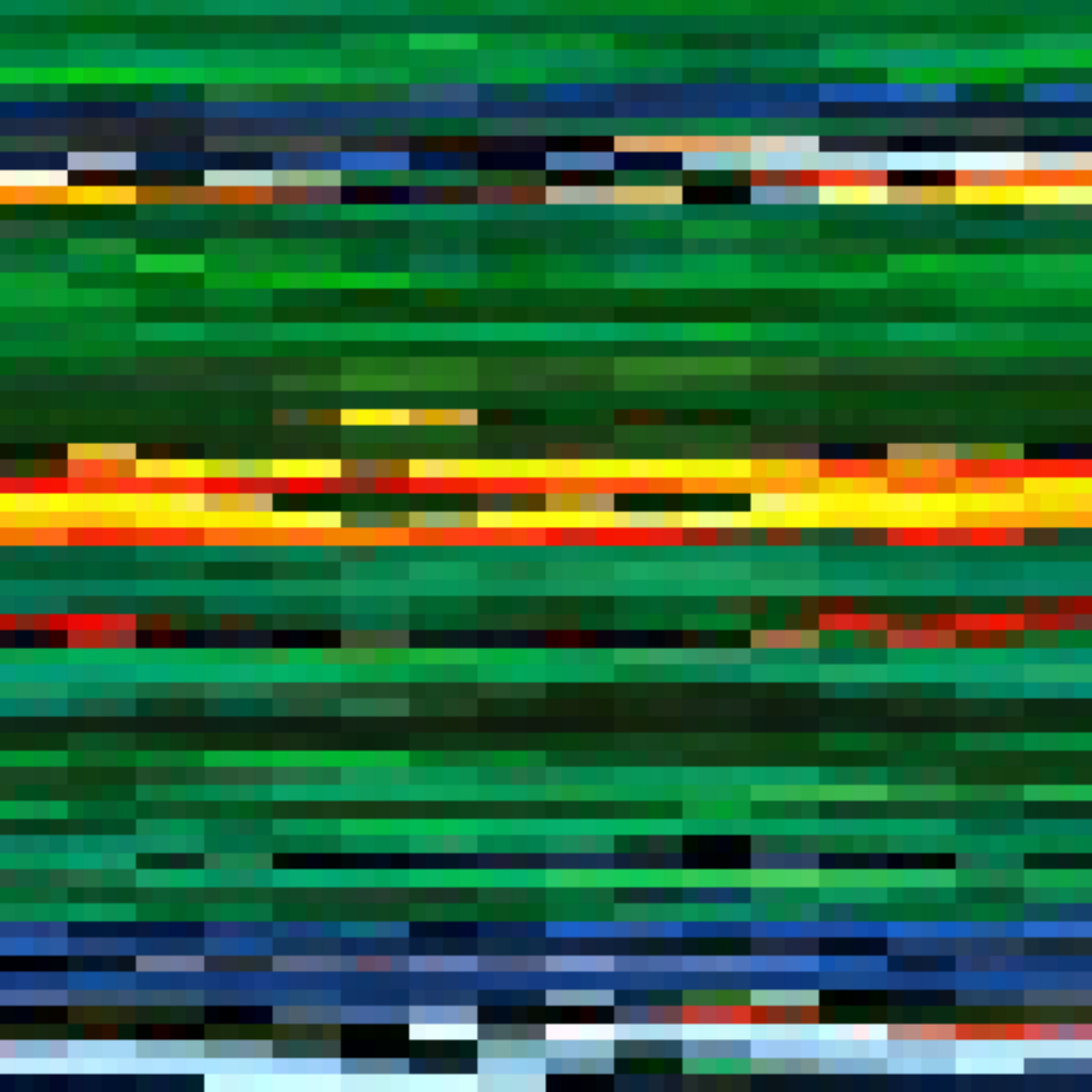}}
\subfigure[Bound Image]
{\includegraphics[scale=.224]{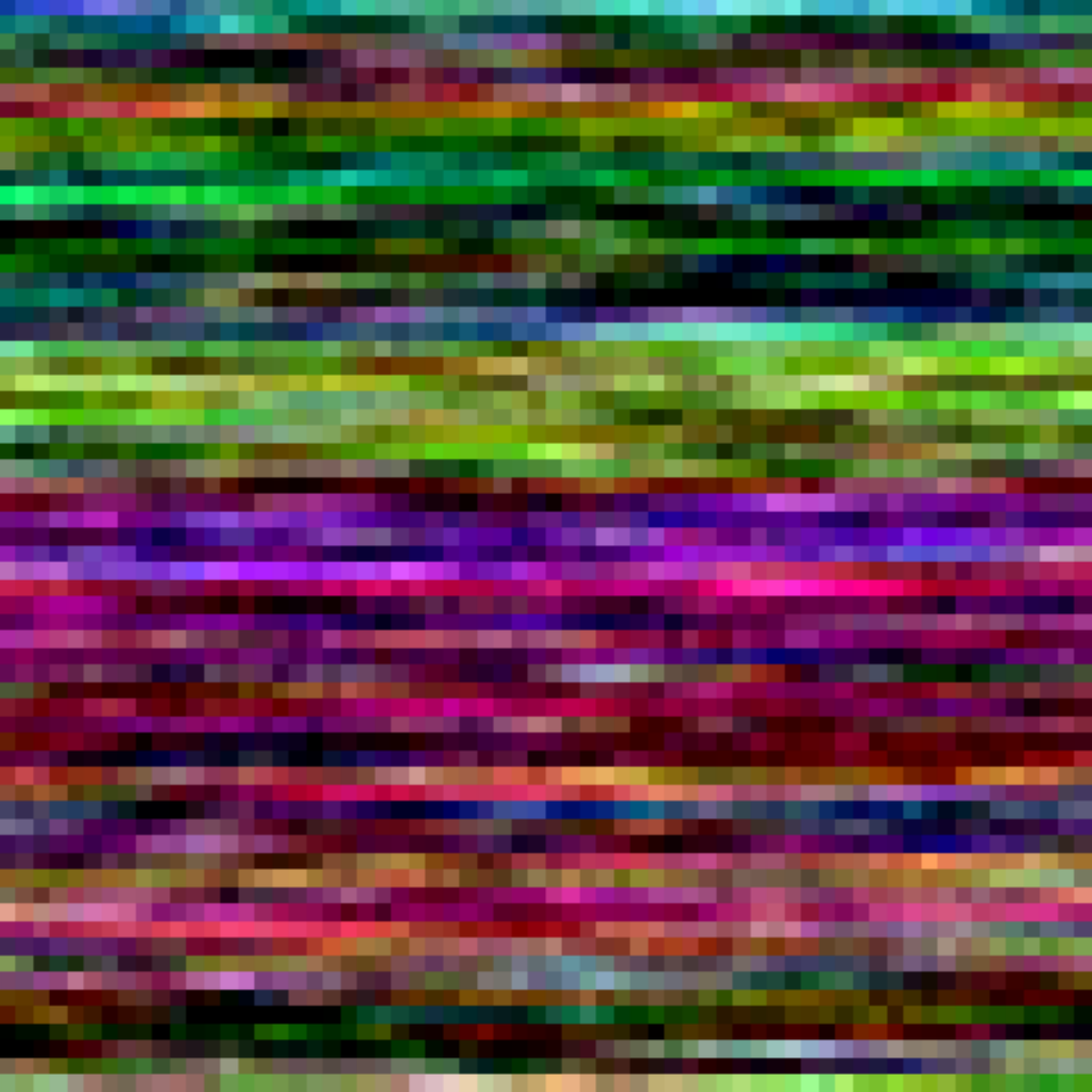}}
\subfigure[Hilbert Decode]
{\includegraphics[scale=.224]{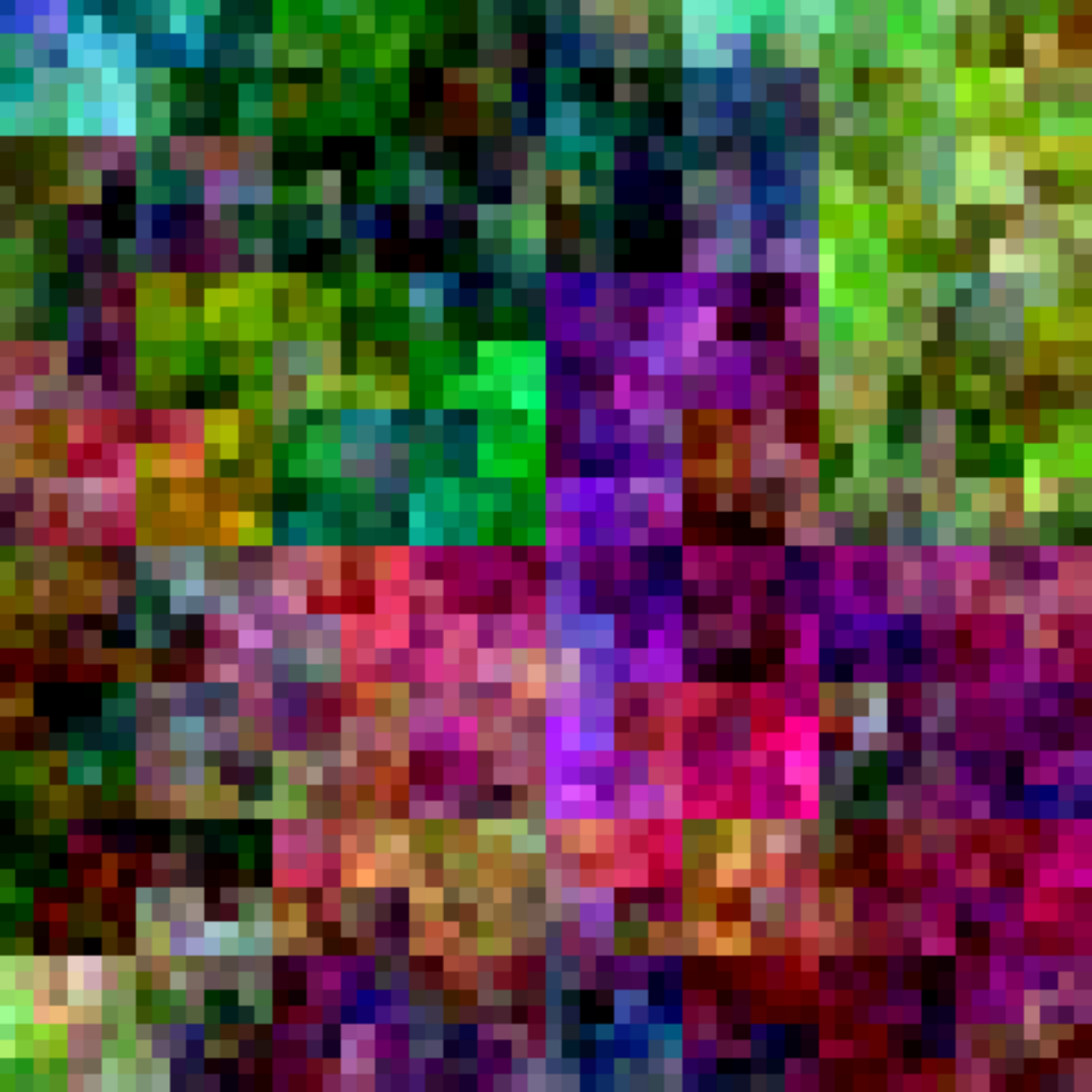}}
\subfigure[Hilbert Encode]
{\includegraphics[scale=.224]{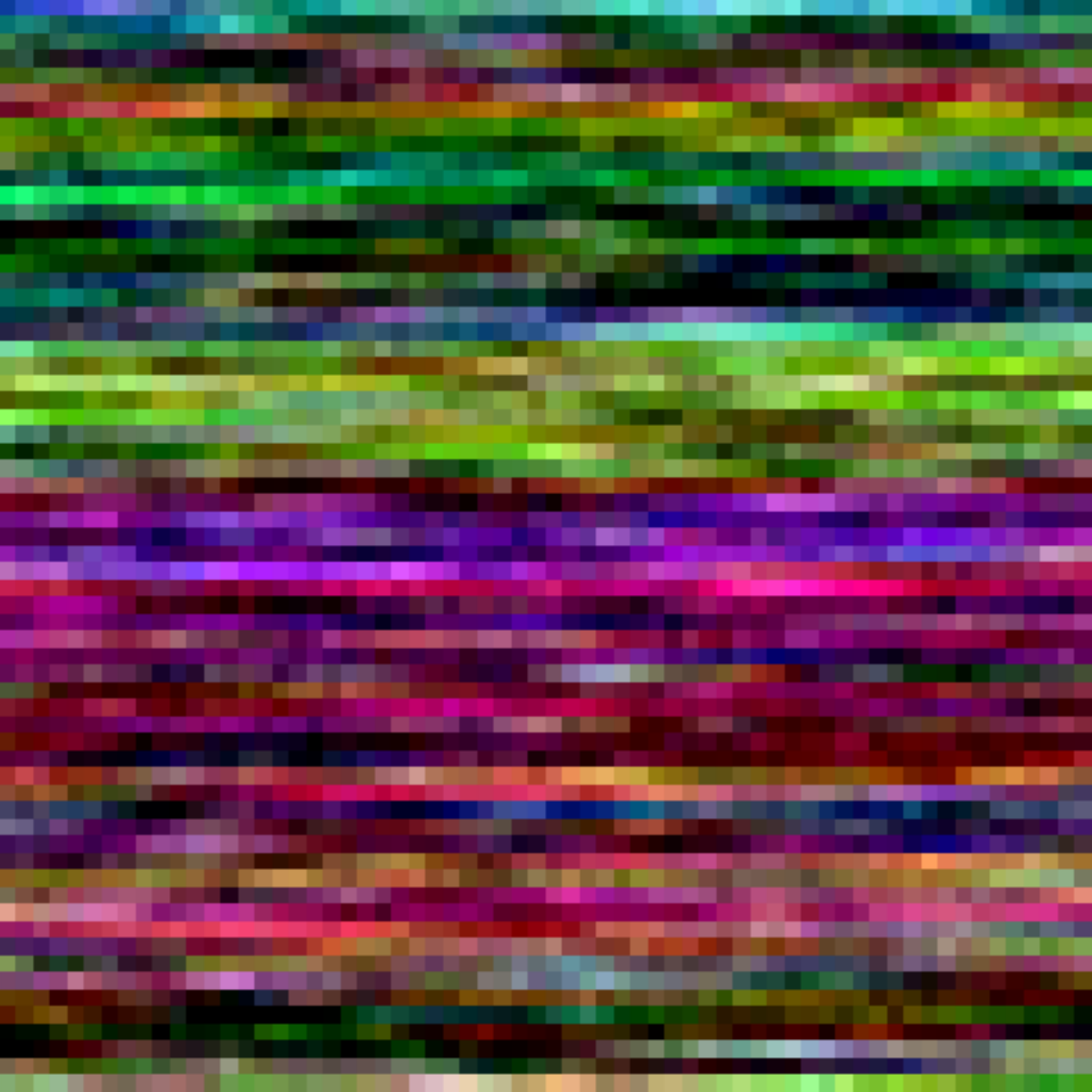}}
\subfigure[Unbound Image]
{\includegraphics[scale=.224]{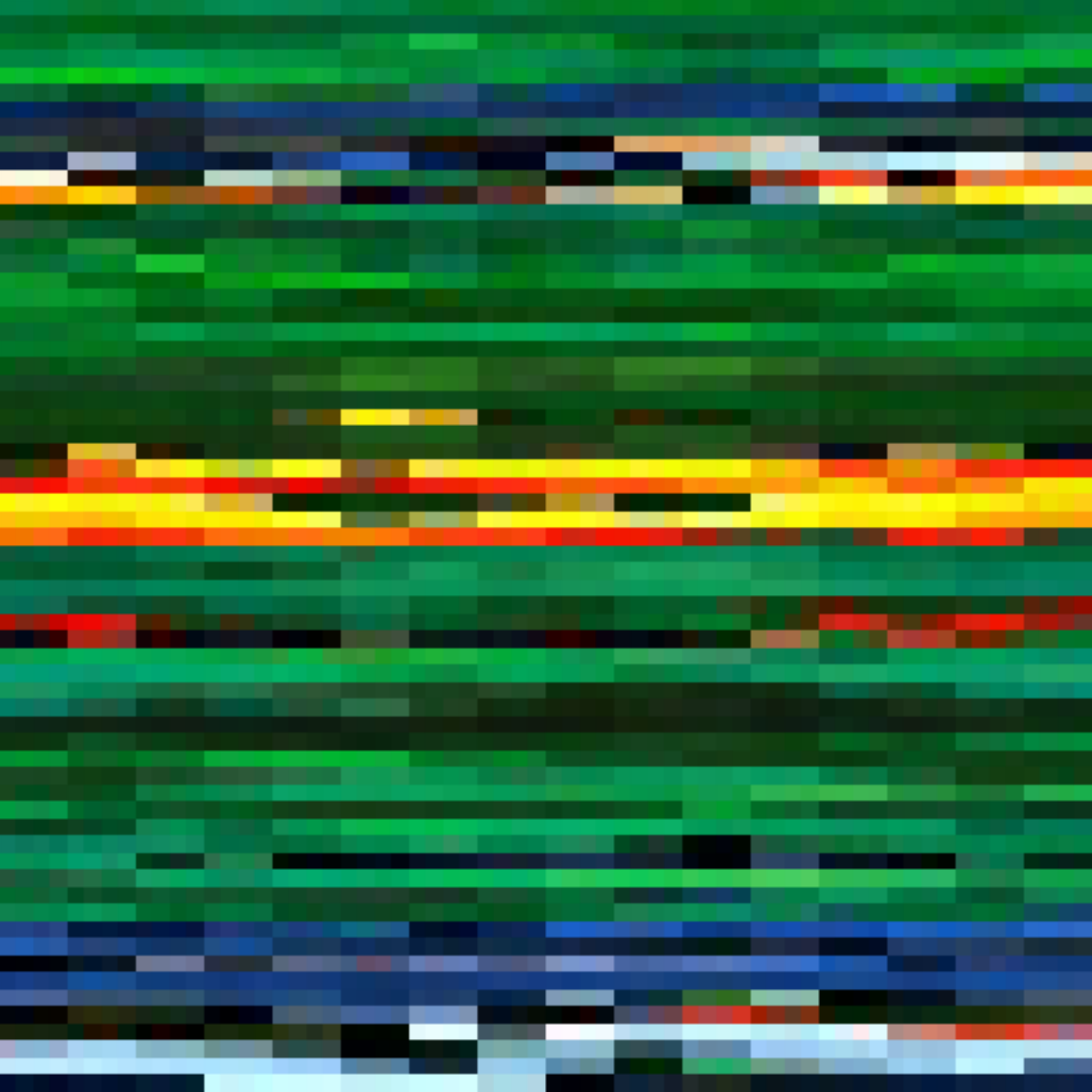}}
\subfigure[Hilbert Decode]
{\includegraphics[scale=.224]{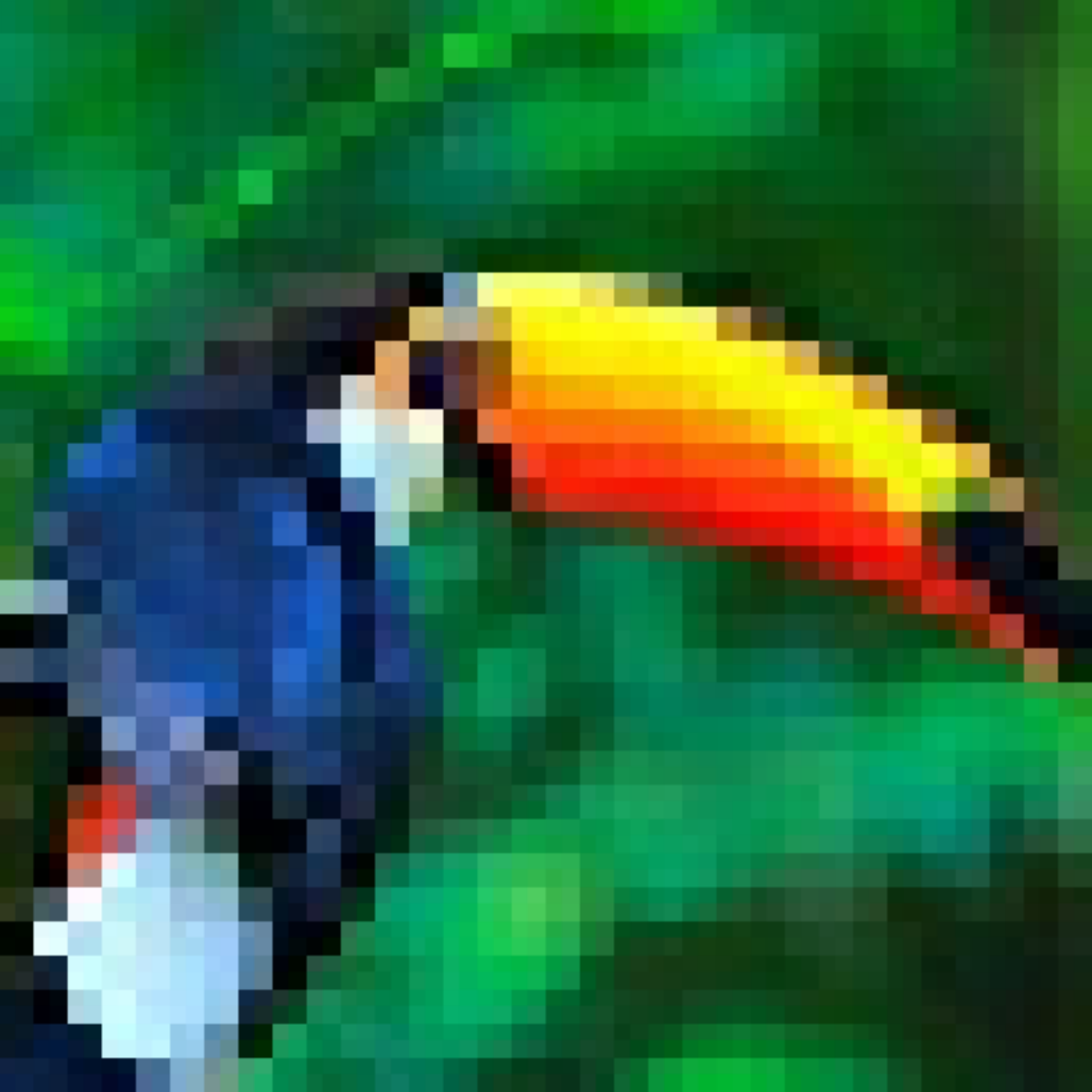}}
\caption{Hilbert Transform combined with the standard HRR. The intuition is that the standard HRR is a 1D convolution, and so if we Hilbert encode the input we linearize the pixels in a manner that is retaining much of the 2D spatial locality in a 1D space. This allows successful obfuscation and extraction but did not perform as well.}
\label{hilbert}
\end{figure}

\end{document}